\documentclass{article}
\PassOptionsToPackage{numbers,sort&compress}{natbib}
\usepackage{PRIMEarxiv}
\usepackage[utf8]{inputenc}
\usepackage[T1]{fontenc}
\usepackage{natbib}
\usepackage{nicefrac}
\usepackage{microtype}

\usepackage{amsmath,amssymb,amsthm,mathtools,bm}
\usepackage{booktabs}
\usepackage{multirow}
\usepackage{array}
\usepackage{tabularx}
\usepackage{xcolor}
\usepackage{colortbl}
\usepackage{xspace}
\usepackage{etoolbox}
\usepackage{graphicx}
\usepackage{enumitem}
\usepackage{algorithm}
\usepackage{algorithmic}
\usepackage{tikz}
\usepackage{pgfplots}
\usepgfplotslibrary{groupplots}
\pgfplotsset{compat=1.18}
\usetikzlibrary{calc,arrows.meta}
\graphicspath{{figures/}}

\AtEndPreamble{%
\newtheorem{theorem}{Theorem}%
\newtheorem{proposition}[theorem]{Proposition}%
\newtheorem{lemma}[theorem]{Lemma}%
\newtheorem{corollary}[theorem]{Corollary}%
\theoremstyle{definition}%
\newtheorem{assumption}[theorem]{Assumption}%
\theoremstyle{remark}%
\theoremstyle{plain}%
\crefname{assumption}{Assumption}{Assumptions}%
\Crefname{assumption}{Assumption}{Assumptions}%
}

\newcommand{\R}{\mathbb{R}}
\newcommand{\E}{\mathbb{E}}

\DeclareMathOperator{\softmax}{softmax}
\DeclareMathOperator{\Diag}{Diag}
\DeclareMathOperator{\lse}{LSE}

\DeclareMathOperator{\trace}{tr}

\providecommand{\norm}[1]{\ensuremath{\left\lVert #1\right\rVert}}
\newcommand{\St}{\mathrm{St}}
\newcommand{\cR}{\mathcal{R}}
\newcommand{\onev}{\bm{1}}

\newcommand{\method}{\textsc{OVAL}\xspace}

\definecolor{gold}{RGB}{243,156,18}
\definecolor{silver}{RGB}{110,128,138}
\definecolor{ribbonblue}{RGB}{76,149,214}
\newcommand{\tikzmedal}[2]{\raisebox{-0.25ex}{\tikz[x=1ex,y=1ex]{%
  \fill[ribbonblue] (-0.45,0.9) -- (-0.05,0.9) -- (0.25,0.1) -- (-0.15,0.1) -- cycle;
  \fill[ribbonblue!70!black] (0.45,0.9) -- (0.05,0.9) -- (-0.25,0.1) -- (0.15,0.1) -- cycle;
  \fill[#1] (0,-0.25) circle (0.55);
  \node[font=\fontsize{4}{4}\selectfont\bfseries,text=white] at (0,-0.25) {#2};}}}
\IfFileExists{twemojis.sty}{%
  \usepackage{twemojis}
  \newcommand{\goldmedal}{\raisebox{-0.25ex}{\twemoji[height=1.45\fontcharht\font`X]{1st place medal}}}
  \newcommand{\silvermedal}{\raisebox{-0.25ex}{\twemoji[height=1.45\fontcharht\font`X]{2nd place medal}}}
}{%
  \newcommand{\goldmedal}{\tikzmedal{gold}{1}}
  \newcommand{\silvermedal}{\tikzmedal{silver}{2}}
}
\newcommand{\best}[1]{\textbf{#1}\,\goldmedal}
\newcommand{\second}[1]{#1\,\silvermedal}

\definecolor{vblue}{HTML}{2A78D6}
\newcommand{\topone}[1]{\textbf{#1}}
\newcommand{\toptwo}[1]{\cellcolor{vblue!30}#1}
\newcommand{\topthree}[1]{\cellcolor{vblue!12}#1}

\newcommand{\eg}{e.g.\@\xspace}
\newcommand{\ie}{i.e.\@\xspace}

\def\showpendingids{0}
\newcommand{\pending}[1]{{\color{red}\ifnum\showpendingids=1 \texttt{\scriptsize[#1]}\else --\fi}}

\usepackage{caption}
\usepackage{placeins}
\usepackage{url}

\definecolor{linkblue}{rgb}{0.21,0.49,0.74}
\usepackage[breaklinks,colorlinks,allcolors=linkblue]{hyperref}
\usepackage[capitalize,nameinlink]{cleveref}

\title{OVAL: Output-Aware Local Page Bases for KV Cache Compression}

\author{
  Ashkan Shahbazi, Chayne Thrash, Soheil Kolouri \vspace{0.1in}\\
  College of Connected Computing \\
  Vanderbilt University \\
  Nashville, TN, USA \\
  \texttt{\{ashkan.shahbazi,chayne.thrash,soheil.kolouri\}@vanderbilt.edu}
}

\makeatletter
\providecommand\@LN[2]{}
\providecommand\@LN@col[1]{}
\makeatother

\begin{document}
\maketitle

\begin{abstract}
Long-context decoding, whether over long reasoning traces, long videos, or robot episodes, is increasingly limited by the memory of the key-value (KV) cache. Page-retrieval sparse attention makes each decoding step cheaper by attending only to the top-scoring KV pages. However, it still stores every page exactly, and adds a small digest per page for scoring, so the cache ends up slightly \emph{larger} than the full cache.

We propose \method, which stores each page as a pair of output-aware low-rank bases that act as both the retrieval index and the storage. A key basis scores the page and rebuilds its keys when the page is selected, and a PCA basis rebuilds its values. The exact pages can then be dropped, and nothing extra is added to the cache.

The key idea is that the key basis should be output-aware. A low-rank error in the keys perturbs the attention output at first order, but perturbs the page mass that retrieval ranks only at second order in expectation. \method therefore builds the key basis from a mix of key geometry (key PCA) and the attention-output error, with a per-head mixing weight set in closed form.

At the same attended budget, \method is best or tied-best in eight of nine long-reasoning settings. It matches FullKV on MATH500 with 66\% of its memory, is 1.25 points above the best baseline on AIME24 with 56\%, and is on par with FullKV on LongBench~v2. On \mbox{Video-MME} and LVBench, it is within 0.32 points of exact-page retrieval at the same read, while storing 53\% of the full cache instead of 103\%. Its savings also stack with quantization and token pruning: with 4-bit bases, it is within 0.25 points of FullKV at 41\% of its memory.
\end{abstract}

\begin{center}
\small Code: \url{https://github.com/Ashkan13776/oval-kv}
\end{center}

\begin{figure}[t]
\begin{center}
\centering
\captionsetup{type=figure}
\begingroup
\definecolor{paper}{HTML}{FFFFFF}
\definecolor{lane}{HTML}{F1EDE7}
\definecolor{hair}{HTML}{C9C2B8}
\definecolor{ink}{HTML}{1E1A17}
\definecolor{ink2}{HTML}{3F3730}
\definecolor{ink3}{HTML}{574D45}
\definecolor{psC}{HTML}{4D5560}   
\definecolor{ovC}{HTML}{D55E00}   
\definecolor{ovD}{HTML}{A84800}   
\definecolor{keyO}{HTML}{3F82B8}\definecolor{keyI}{HTML}{9CC2E0}
\definecolor{valO}{HTML}{23906E}\definecolor{valI}{HTML}{8ED0B6}
\definecolor{summ}{HTML}{C46FA0}
\definecolor{burg}{HTML}{9B2226}

\tikzset{
  pill/.style={text=white,rounded corners=6pt,inner xsep=6pt,inner ysep=3pt,font=\small\bfseries},
  badge/.style={circle,fill=ink2,text=white,inner sep=1.5pt,minimum size=15pt,font=\small\bfseries},
  sel/.style={dashed,line width=1.2pt,rounded corners=3pt},
  lbl/.style={font=\scriptsize\bfseries,text=ink3},
}
\newcommand{\tagt}[2]{\textcolor{#1}{\fontsize{6}{7}\selectfont\bfseries\,#2}}

\newcommand{\kvblock}[6]{%
  \pgfmathsetmacro{\bh}{0.1+#3*#4+(#3-1)*0.03}
  \fill[#5,rounded corners=1.5pt] (#1,#2) rectangle ++(1.0,\bh);
  \foreach \r in {1,...,#3}{\foreach \c in {0,...,3}{
    \pgfmathsetmacro{\cx}{#1+0.05+\c*0.2325}
    \pgfmathsetmacro{\cy}{#2+0.05+(\r-1)*(#4+0.03)}
    \fill[#6,rounded corners=0.6pt] (\cx,\cy) rectangle ++(0.2025,#4);}}}

\newcommand{\scorebars}[5]{%
  \draw[hair,line width=0.6pt] (2.15,#1) -- (9.15,#1);
  \foreach \i in {0,...,5}{
    \pgfmathsetmacro{\v}{#3[\i]}
    \fill[hair!70,rounded corners=1pt] ({2.65+1.2*\i-0.28},#1) rectangle ++(0.56,{#2*\v});}
  \foreach \i in {#4}{
    \pgfmathsetmacro{\v}{#3[\i]}
    \fill[#5,rounded corners=1pt] ({2.65+1.2*\i-0.28},#1) rectangle ++(0.56,{#2*\v});}}

\newcommand{\fullpage}[3]{\begin{scope}[opacity=#3]
  \kvblock{#1}{#2}{5}{0.13}{valO}{valI}
  \kvblock{#1}{#2+0.92}{5}{0.13}{keyO}{keyI}
  \fill[summ,rounded corners=1pt] (#1,#2+1.84) rectangle ++(1.0,0.13);\end{scope}}
\newcommand{\kvshade}[4]{%
  \pgfmathsetmacro{\bh}{0.1+#3*#4+(#3-1)*0.03}
  \shade[left color=keyO,right color=valO,rounded corners=1.5pt] (#1,#2) rectangle ++(1.0,\bh);
  \foreach \r in {1,...,#3}{\foreach \c/\m in {0/100,1/66,2/33,3/0}{
    \pgfmathsetmacro{\cx}{#1+0.05+\c*0.2325}
    \pgfmathsetmacro{\cy}{#2+0.05+(\r-1)*(#4+0.03)}
    \fill[keyI!\m!valI,rounded corners=0.6pt] (\cx,\cy) rectangle ++(0.2025,#4);}}}
\newcommand{\ovalpage}[3]{\begin{scope}[opacity=#3]
  \kvblock{#1}{#2}{4}{0.09}{valO}{valI}
  \kvshade{#1}{#2+0.60}{3}{0.09}\end{scope}}

\resizebox{\linewidth}{!}{%
\begin{tikzpicture}[font=\small,text=ink]

\def\lg{{0.45,0.88,0.80,0.62}}
\def\imp{{0.40,0.84,0.38,0.70}}
\newcommand{\spage}[3]{\begin{scope}[shift={(#1,#2)},scale=0.6]\fullpage{0}{0}{#3}\end{scope}}
\newcommand{\opage}[3]{\begin{scope}[shift={(#1,#2)},xscale=0.6,yscale=0.72]\ovalpage{0}{0}{#3}\end{scope}}
\newcommand{\rpage}[2]{\begin{scope}[shift={(#1,#2)},scale=0.6]
  \kvblock{0}{0}{5}{0.13}{valO}{valI}\kvblock{0}{0.92}{5}{0.13}{keyO}{keyI}\end{scope}
  \draw[ovC,dashed,line width=0.9pt,rounded corners=1.5pt] (#1-0.03,#2-0.04) rectangle ++(0.66,1.12);}
\newcommand{\lanebars}[5]{%
  \draw[hair,line width=0.6pt] (5.02,#1) -- (6.72,#1);
  \foreach \i in {0,...,3}{
    \pgfmathsetmacro{\v}{#2[\i]}
    \pgfmathsetmacro{\bx}{5.10+0.42*\i}
    \fill[hair!70,rounded corners=1pt] (\bx,#1) rectangle ++(0.28,{1.05*\v});
    \pgfmathtruncatemacro{\n}{\i+1}
    \node[font=\tiny,text=ink3,inner sep=1pt] at (\bx+0.14,#1+1.05*\v+0.1) {$P_{\n}$};}
  \foreach \i in {#4,#5}{
    \pgfmathsetmacro{\v}{#2[\i]}
    \pgfmathsetmacro{\bx}{5.10+0.42*\i}
    \fill[#3,rounded corners=1pt] (\bx,#1) rectangle ++(0.28,{1.05*\v});}}
\tikzset{stage/.style={font=\scriptsize\itshape\bfseries,text=ink2},
         flow/.style={-latex,line width=0.9pt,draw=ink3},
         attn/.style={draw=ink2,fill=paper,rounded corners=3pt,line width=0.7pt,font=\scriptsize\bfseries,minimum height=0.8cm,minimum width=0.68cm,inner sep=1pt}}

\node[badge] at (1.45,7.3) {a};
\node[pill,fill=psC,anchor=west] at (1.9,7.30) {Page retrieval};
\node[font=\scriptsize\bfseries,text=ink3,anchor=west] at (4.05,7.30) {exact pages + key summary};
\fill[lane,rounded corners=4pt] (1.9,4.85) rectangle (9.4,6.95);
\node[stage] at (3.37,6.73) {store};
\node[stage] at (5.87,6.73) {score, top-$k$};
\node[stage] at (8.25,6.73) {load, attend};
\foreach \i in {0,...,3}{\spage{2.08+0.66*\i}{5.25}{1}}
\draw[flow] (4.72,5.80) -- (4.98,5.80);
\lanebars{5.25}{\lg}{psC}{1}{2}
\node[font=\tiny\bfseries,text=ink3] at (5.87,5.08) {bound from key summary};
\draw[flow] (6.80,5.80) -- (7.06,5.80);
\spage{7.10}{5.25}{1}\spage{7.80}{5.25}{1}
\draw[flow] (8.48,5.80) -- (8.66,5.80);
\node[attn] at (9.02,5.80) {attn};
\node[pill,fill=ovD,anchor=west] at (1.9,4.35) {OVAL};
\node[font=\scriptsize\bfseries,text=ovD,anchor=west] at (3.2,4.35) {output-aware key basis};
\fill[lane,rounded corners=4pt] (1.9,1.70) rectangle (9.4,4.0);
\node[stage] at (3.37,3.78) {store};
\node[stage] at (5.87,3.78) {score, top-$k$};
\node[stage,text=ovD] at (8.25,3.78) {rebuild, attend};
\foreach \i in {0,...,3}{\opage{2.08+0.66*\i}{2.38}{1}}
\draw[flow] (4.72,2.75) -- (4.98,2.75);
\lanebars{2.20}{\imp}{ovC}{1}{3}
\node[font=\tiny\bfseries,text=ink3] at (5.87,2.03) {mass from key basis};
\draw[flow] (6.80,2.75) -- (7.06,2.75);
\rpage{7.10}{2.20}\rpage{7.80}{2.20}
\draw[flow] (8.48,2.75) -- (8.66,2.75);
\node[attn] at (9.02,2.75) {attn};
\begin{scope}[yshift=1.32cm,font=\scriptsize\bfseries,text=ink3]
  \fill[keyO] (1.95,-0.1) rectangle ++(0.2,0.2); \node[anchor=west] at (2.2,0) {keys};
  \fill[valO] (3.05,-0.1) rectangle ++(0.2,0.2); \node[anchor=west] at (3.3,0) {values};
  \fill[summ] (4.35,-0.1) rectangle ++(0.2,0.2); \node[anchor=west] at (4.6,0) {key summary};
  \fill[ink3] (6.45,-0.1) rectangle ++(0.14,0.2); \node[anchor=west] at (6.65,0) {top-$k$};
  \draw[ovC,dashed,line width=0.9pt] (7.75,-0.1) rectangle ++(0.2,0.2); \node[anchor=west] at (8.0,0) {rebuilt};
  \shade[left color=keyO,right color=valO] (1.95,-0.6) rectangle ++(0.45,0.2); \node[anchor=west] at (2.45,-0.5) {output-aware key basis (built from keys and values)};
\end{scope}

\draw[hair,dashed] (9.85,0.0) -- (9.85,7.6);
\node[badge] at (10.3,7.3) {b};
\node[pill,fill=ink2,anchor=west,font=\scriptsize\bfseries,inner ysep=2pt] at (10.75,7.3) {\method on a cached robot episode};
\node[font=\scriptsize\itshape,text=ink3,anchor=east] at (19.3,7.3) {Qwen2.5-VL};
\node[draw=ink,fill=paper,rounded corners=4pt,line width=0.6pt,inner xsep=4pt,inner ysep=2pt,font=\scriptsize,anchor=west] at (10.75,6.82) {\textbf{Q:} \emph{Repeat the task the robot performed first.}};
\node[font=\fontsize{6}{6.6}\selectfont\bfseries,text=ovD,anchor=east,align=right,inner sep=0pt] at (19.3,6.88) {stored: 53\% of FullKV\\read per step: 14.5\%};
\node[anchor=south west,inner sep=0] at (10.750,4.860) {\includegraphics[width=1.342cm,height=1.342cm]{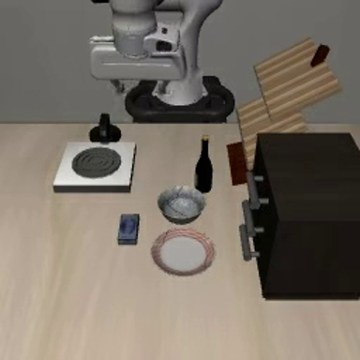}};
\draw[ovC,line width=1.2pt,rounded corners=1.5pt] (10.720,4.830) rectangle ++(1.402,1.402);
\node[fill=paper,fill opacity=0.85,text opacity=1,inner sep=0.8pt,font=\tiny\bfseries,text=ink,anchor=south east] at (12.062,4.890) {1s};
\fill[hair!45] (10.750,4.660) rectangle ++(1.342,0.10);
\fill[ovC] (10.750,4.660) rectangle ++(1.342,0.10);
\node[anchor=south west,inner sep=0] at (12.192,4.860) {\includegraphics[width=1.342cm,height=1.342cm]{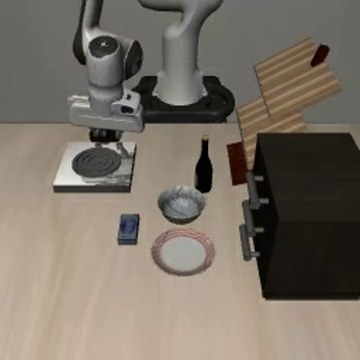}};
\fill[paper,opacity=0.47] (12.192,4.860) rectangle ++(1.342,1.342);
\draw[ovC,line width=1.2pt,rounded corners=1.5pt] (12.162,4.830) rectangle ++(1.402,1.402);
\node[fill=paper,fill opacity=0.85,text opacity=1,inner sep=0.8pt,font=\tiny\bfseries,text=ink,anchor=south east] at (13.503,4.890) {4s};
\fill[hair!45] (12.192,4.660) rectangle ++(1.342,0.10);
\fill[ovC] (12.192,4.660) rectangle ++(0.537,0.10);
\node[anchor=south west,inner sep=0] at (13.633,4.860) {\includegraphics[width=1.342cm,height=1.342cm]{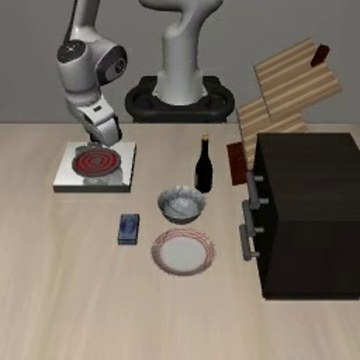}};
\fill[paper,opacity=0.49] (13.633,4.860) rectangle ++(1.342,1.342);
\draw[ovC,line width=1.2pt,rounded corners=1.5pt] (13.603,4.830) rectangle ++(1.402,1.402);
\node[fill=paper,fill opacity=0.85,text opacity=1,inner sep=0.8pt,font=\tiny\bfseries,text=ink,anchor=south east] at (14.945,4.890) {7s};
\fill[hair!45] (13.633,4.660) rectangle ++(1.342,0.10);
\fill[ovC] (13.633,4.660) rectangle ++(0.502,0.10);
\node[anchor=south west,inner sep=0] at (15.075,4.860) {\includegraphics[width=1.342cm,height=1.342cm]{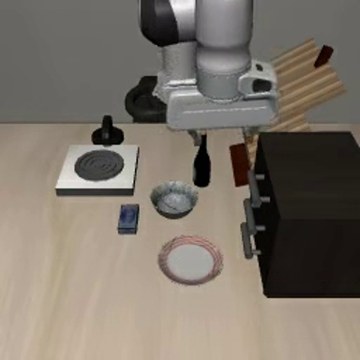}};
\fill[paper,opacity=0.56] (15.075,4.860) rectangle ++(1.342,1.342);
\node[fill=paper,fill opacity=0.85,text opacity=1,inner sep=0.8pt,font=\tiny\bfseries,text=ink,anchor=south east] at (16.387,4.890) {10s};
\fill[hair!45] (15.075,4.660) rectangle ++(1.342,0.10);
\fill[ink3!60] (15.075,4.660) rectangle ++(0.370,0.10);
\node[anchor=south west,inner sep=0] at (16.517,4.860) {\includegraphics[width=1.342cm,height=1.342cm]{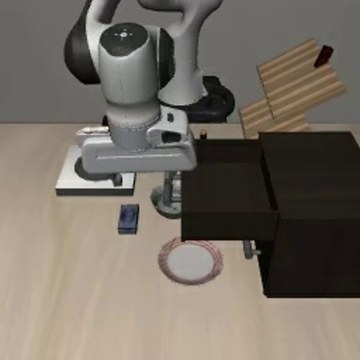}};
\fill[paper,opacity=0.58] (16.517,4.860) rectangle ++(1.342,1.342);
\node[fill=paper,fill opacity=0.85,text opacity=1,inner sep=0.8pt,font=\tiny\bfseries,text=ink,anchor=south east] at (17.828,4.890) {14s};
\fill[hair!45] (16.517,4.660) rectangle ++(1.342,0.10);
\fill[ink3!60] (16.517,4.660) rectangle ++(0.350,0.10);
\node[anchor=south west,inner sep=0] at (17.958,4.860) {\includegraphics[width=1.342cm,height=1.342cm]{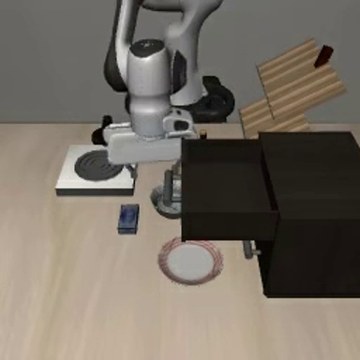}};
\fill[paper,opacity=0.61] (17.958,4.860) rectangle ++(1.342,1.342);
\node[fill=paper,fill opacity=0.85,text opacity=1,inner sep=0.8pt,font=\tiny\bfseries,text=ink,anchor=south east] at (19.270,4.890) {18s};
\fill[hair!45] (17.958,4.660) rectangle ++(1.342,0.10);
\fill[ink3!60] (17.958,4.660) rectangle ++(0.284,0.10);
\node[anchor=south west,inner sep=0] at (10.750,2.860) {\includegraphics[width=1.342cm,height=1.342cm]{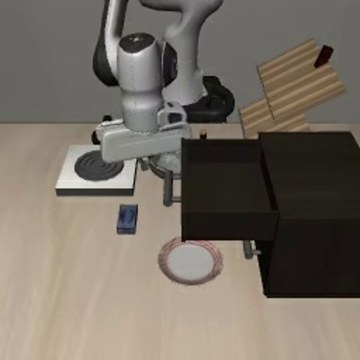}};
\fill[paper,opacity=0.60] (10.750,2.860) rectangle ++(1.342,1.342);
\node[fill=paper,fill opacity=0.85,text opacity=1,inner sep=0.8pt,font=\tiny\bfseries,text=ink,anchor=south east] at (12.062,2.890) {22s};
\fill[hair!45] (10.750,2.660) rectangle ++(1.342,0.10);
\fill[ink3!60] (10.750,2.660) rectangle ++(0.308,0.10);
\node[anchor=south west,inner sep=0] at (12.192,2.860) {\includegraphics[width=1.342cm,height=1.342cm]{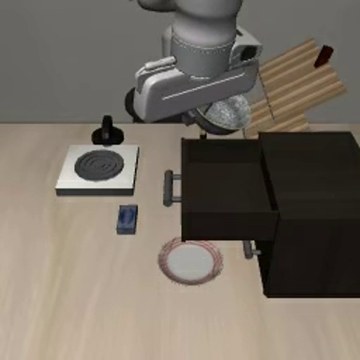}};
\fill[paper,opacity=0.62] (12.192,2.860) rectangle ++(1.342,1.342);
\node[fill=paper,fill opacity=0.85,text opacity=1,inner sep=0.8pt,font=\tiny\bfseries,text=ink,anchor=south east] at (13.503,2.890) {25s};
\fill[hair!45] (12.192,2.660) rectangle ++(1.342,0.10);
\fill[ink3!60] (12.192,2.660) rectangle ++(0.274,0.10);
\node[anchor=south west,inner sep=0] at (13.633,2.860) {\includegraphics[width=1.342cm,height=1.342cm]{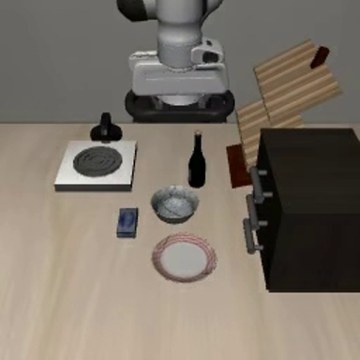}};
\fill[paper,opacity=0.57] (13.633,2.860) rectangle ++(1.342,1.342);
\node[fill=paper,fill opacity=0.85,text opacity=1,inner sep=0.8pt,font=\tiny\bfseries,text=ink,anchor=south east] at (14.945,2.890) {29s};
\fill[hair!45] (13.633,2.660) rectangle ++(1.342,0.10);
\fill[ink3!60] (13.633,2.660) rectangle ++(0.368,0.10);
\node[anchor=south west,inner sep=0] at (15.075,2.860) {\includegraphics[width=1.342cm,height=1.342cm]{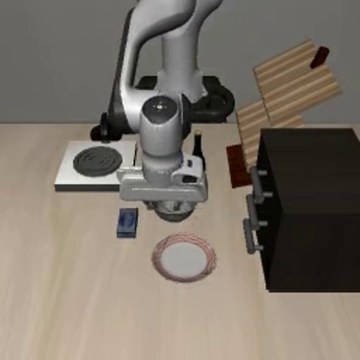}};
\fill[paper,opacity=0.57] (15.075,2.860) rectangle ++(1.342,1.342);
\node[fill=paper,fill opacity=0.85,text opacity=1,inner sep=0.8pt,font=\tiny\bfseries,text=ink,anchor=south east] at (16.387,2.890) {32s};
\fill[hair!45] (15.075,2.660) rectangle ++(1.342,0.10);
\fill[ink3!60] (15.075,2.660) rectangle ++(0.368,0.10);
\node[anchor=south west,inner sep=0] at (16.517,2.860) {\includegraphics[width=1.342cm,height=1.342cm]{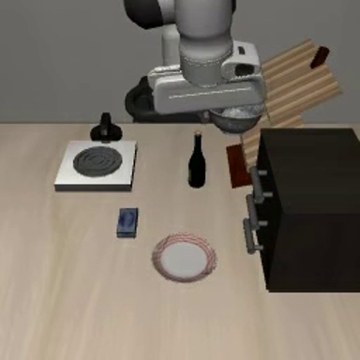}};
\fill[paper,opacity=0.52] (16.517,2.860) rectangle ++(1.342,1.342);
\node[fill=paper,fill opacity=0.85,text opacity=1,inner sep=0.8pt,font=\tiny\bfseries,text=ink,anchor=south east] at (17.828,2.890) {35s};
\fill[hair!45] (16.517,2.660) rectangle ++(1.342,0.10);
\fill[ink3!60] (16.517,2.660) rectangle ++(0.449,0.10);
\node[anchor=south west,inner sep=0] at (17.958,2.860) {\includegraphics[width=1.342cm,height=1.342cm]{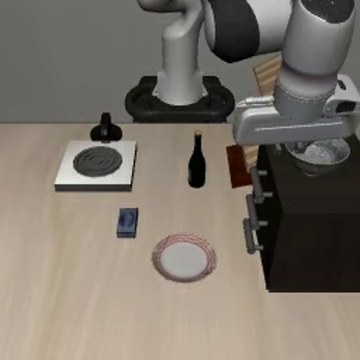}};
\fill[paper,opacity=0.37] (17.958,2.860) rectangle ++(1.342,1.342);
\node[fill=paper,fill opacity=0.85,text opacity=1,inner sep=0.8pt,font=\tiny\bfseries,text=ink,anchor=south east] at (19.270,2.890) {38s};
\fill[hair!45] (17.958,2.660) rectangle ++(1.342,0.10);
\fill[ink3!60] (17.958,2.660) rectangle ++(0.710,0.10);
\draw[ovC,line width=0.6pt] (10.770,6.272) -- ++(0,0.08) -- (14.955,6.352) -- ++(0,-0.08);
\node[font=\scriptsize\bfseries,text=ovD,anchor=south,inner sep=1pt] at (12.863,6.352) {task 1: turn on the stove (23\% selected)};
\draw[ink3,line width=0.6pt] (15.095,6.272) -- ++(0,0.08) -- (19.280,6.352) -- ++(0,-0.08);
\node[font=\scriptsize\bfseries,text=ink3,anchor=south,inner sep=1pt] at (17.188,6.352) {task 2: bowl into drawer (9\% selected)};
\draw[ink3,line width=0.6pt] (10.770,4.272) -- ++(0,0.08) -- (13.513,4.352) -- ++(0,-0.08);
\node[font=\scriptsize\bfseries,text=ink3,anchor=south,inner sep=1pt] at (12.142,4.352) {task 2};
\draw[ink3,line width=0.6pt] (13.653,4.272) -- ++(0,0.08) -- (19.280,4.352) -- ++(0,-0.08);
\node[font=\scriptsize\bfseries,text=ink3,anchor=south,inner sep=1pt] at (16.467,4.352) {task 3: bowl on cabinet (13\% selected)};
\node[draw=valO,fill=valO!8,rounded corners=4pt,line width=0.6pt,inner xsep=4pt,inner ysep=2pt,font=\scriptsize,anchor=west] at (10.75,2.38) {\textbf{A:} \emph{turn on the stove}\ \textcolor{valO}{$\checkmark$}\ \ executed by OpenVLA, task succeeds};
\node[font=\scriptsize\bfseries,text=ink2,anchor=south west,inner sep=0pt] at (10.75,1.83) {Share of each frame's visual tokens \method selects};
\node[anchor=south west,inner sep=0] at (11.300,1.260) {\includegraphics[width=0.470cm,height=0.470cm]{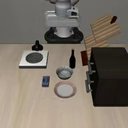}};
\node[anchor=south west,inner sep=0] at (11.800,1.260) {\includegraphics[width=0.470cm,height=0.470cm]{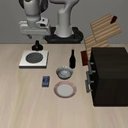}};
\node[anchor=south west,inner sep=0] at (12.300,1.260) {\includegraphics[width=0.470cm,height=0.470cm]{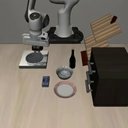}};
\node[anchor=south west,inner sep=0] at (12.800,1.260) {\includegraphics[width=0.470cm,height=0.470cm]{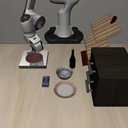}};
\node[anchor=south west,inner sep=0] at (13.300,1.260) {\includegraphics[width=0.470cm,height=0.470cm]{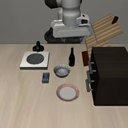}};
\node[anchor=south west,inner sep=0] at (13.800,1.260) {\includegraphics[width=0.470cm,height=0.470cm]{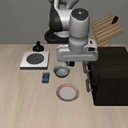}};
\node[anchor=south west,inner sep=0] at (14.300,1.260) {\includegraphics[width=0.470cm,height=0.470cm]{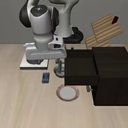}};
\node[anchor=south west,inner sep=0] at (14.800,1.260) {\includegraphics[width=0.470cm,height=0.470cm]{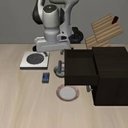}};
\node[anchor=south west,inner sep=0] at (15.300,1.260) {\includegraphics[width=0.470cm,height=0.470cm]{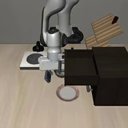}};
\node[anchor=south west,inner sep=0] at (15.800,1.260) {\includegraphics[width=0.470cm,height=0.470cm]{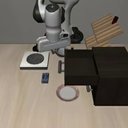}};
\node[anchor=south west,inner sep=0] at (16.300,1.260) {\includegraphics[width=0.470cm,height=0.470cm]{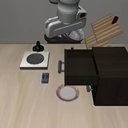}};
\node[anchor=south west,inner sep=0] at (16.800,1.260) {\includegraphics[width=0.470cm,height=0.470cm]{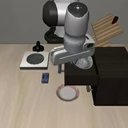}};
\node[anchor=south west,inner sep=0] at (17.300,1.260) {\includegraphics[width=0.470cm,height=0.470cm]{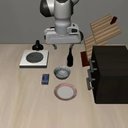}};
\node[anchor=south west,inner sep=0] at (17.800,1.260) {\includegraphics[width=0.470cm,height=0.470cm]{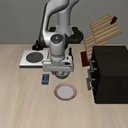}};
\node[anchor=south west,inner sep=0] at (18.300,1.260) {\includegraphics[width=0.470cm,height=0.470cm]{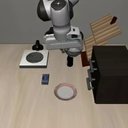}};
\node[anchor=south west,inner sep=0] at (18.800,1.260) {\includegraphics[width=0.470cm,height=0.470cm]{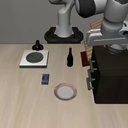}};
\fill[ovC] (11.3200,0.500) rectangle ++(0.210,0.6600);
\fill[ovC] (11.5700,0.500) rectangle ++(0.210,0.3320);
\fill[ovC] (11.8200,0.500) rectangle ++(0.210,0.2103);
\fill[ovC] (12.0700,0.500) rectangle ++(0.210,0.1774);
\fill[ovC] (12.3200,0.500) rectangle ++(0.210,0.1564);
\fill[ovC] (12.5700,0.500) rectangle ++(0.210,0.1579);
\fill[ovC] (12.8200,0.500) rectangle ++(0.210,0.1919);
\fill[ink3!45] (13.0700,0.500) rectangle ++(0.210,0.1206);
\fill[ink3!45] (13.3200,0.500) rectangle ++(0.210,0.1379);
\fill[ink3!45] (13.5700,0.500) rectangle ++(0.210,0.1043);
\fill[ink3!45] (13.8200,0.500) rectangle ++(0.210,0.1017);
\fill[ink3!45] (14.0700,0.500) rectangle ++(0.210,0.1112);
\fill[ink3!45] (14.3200,0.500) rectangle ++(0.210,0.1238);
\fill[ink3!45] (14.5700,0.500) rectangle ++(0.210,0.0896);
\fill[ink3!45] (14.8200,0.500) rectangle ++(0.210,0.0950);
\fill[ink3!45] (15.0700,0.500) rectangle ++(0.210,0.0926);
\fill[ink3!45] (15.3200,0.500) rectangle ++(0.210,0.0819);
\fill[ink3!45] (15.5700,0.500) rectangle ++(0.210,0.1065);
\fill[ink3!45] (15.8200,0.500) rectangle ++(0.210,0.1040);
\fill[ink3!45] (16.0700,0.500) rectangle ++(0.210,0.1347);
\fill[ink3!45] (16.3200,0.500) rectangle ++(0.210,0.0797);
\fill[ink3!45] (16.5700,0.500) rectangle ++(0.210,0.0767);
\fill[ink3!45] (16.8200,0.500) rectangle ++(0.210,0.1056);
\fill[ink3!45] (17.0700,0.500) rectangle ++(0.210,0.1253);
\fill[ink3!45] (17.3200,0.500) rectangle ++(0.210,0.1230);
\fill[ink3!45] (17.5700,0.500) rectangle ++(0.210,0.0895);
\fill[ink3!45] (17.8200,0.500) rectangle ++(0.210,0.1567);
\fill[ink3!45] (18.0700,0.500) rectangle ++(0.210,0.1093);
\fill[ink3!45] (18.3200,0.500) rectangle ++(0.210,0.1190);
\fill[ink3!45] (18.5700,0.500) rectangle ++(0.210,0.1934);
\fill[ink3!45] (18.8200,0.500) rectangle ++(0.210,0.1850);
\fill[ink3!45] (19.0700,0.500) rectangle ++(0.210,0.2793);
\draw[hair,line width=0.6pt] (11.30,0.500) -- (19.30,0.500);
\draw[ink2,densely dashed,line width=0.6pt] (11.30,0.654) -- (19.30,0.654);
\node[font=\scriptsize,text=ink2,anchor=south,inner sep=1.5pt] at (15.05,0.654) {mean 13\% selected};
\node[font=\scriptsize,text=ink2,anchor=east,inner sep=1pt] at (11.25,1.160) {57\%};
\node[font=\scriptsize,text=ink2,anchor=east,inner sep=1pt] at (11.25,0.500) {0};
\draw[ovC,line width=0.6pt] (11.310,0.430) -- ++(0,-0.07) -- (13.040,0.360) -- ++(0,0.07);
\node[font=\scriptsize\bfseries,text=ovD,anchor=north,inner sep=1.5pt] at (12.175,0.360) {task 1 (asked)};
\draw[ink3,line width=0.6pt] (13.060,0.430) -- ++(0,-0.07) -- (17.040,0.360) -- ++(0,0.07);
\node[font=\scriptsize\bfseries,text=ink2,anchor=north,inner sep=1.5pt] at (15.050,0.360) {task 2};
\draw[ink3,line width=0.6pt] (17.060,0.430) -- ++(0,-0.07) -- (19.290,0.360) -- ++(0,0.07);
\node[font=\scriptsize\bfseries,text=ink2,anchor=north,inner sep=1.5pt] at (18.175,0.360) {task 3};

\end{tikzpicture}%
}
\endgroup
\caption{\textbf{(a)}~Page retrieval (\eg, Quest~\citep{pmlr-v235-tang24l}, FreeKV~\citep{liu2026freekvboostingkvcache}) stores every exact page plus a key summary and loads the top-scoring pages. \method stores each page as two low-rank bases, half the memory at our ranks: a key basis built from keys and values (blue to green), also the retrieval index, and a PCA value basis. The key basis scores pages by approximate attention mass and rebuilds the selected pages with keys that preserve the attention output. \textbf{(b)}~Robot-episode recall executed by a VLA policy: a LIBERO Watch-and-Do episode cached before the question. Each frame is shown as strongly as \method selects it (bars: share of its visual tokens selected, relative to the largest). Bottom: the same per frame pair, as Qwen2.5-VL encodes two frames together. Selection is highest at the start and rises again at the end, as in the average over all long \mbox{Video-MME} questions (\cref{supp:fig:videomme-frames,supp:fig:filmstrips}). OpenVLA executes \method's command and succeeds.}
\label{fig:teaser}
\end{center}
\end{figure}

\section{Introduction}
\label{sec:intro}
Reasoning traces, long documents, and videos all make the KV cache the main memory and bandwidth bottleneck of decoding~\citep{fu2024challengesdeployinglongcontexttransformers}. A 64-frame video alone is about 13K visual tokens.

Page-retrieval methods for sparse attention build on paged KV memory~\citep{kwon2023efficientmemorymanagementlarge}. They keep a compact key summary, or \emph{digest}, for each cache page, and at each step attend only to the highest-scoring pages~\citep{pmlr-v235-tang24l,chen2024arkvale,hwang2026lockspagelocalcompactkey,liu2026freekvboostingkvcache}. This reduces how much is \emph{read} per step, but not how much is \emph{stored}. Every page keeps its exact keys and values, and the digests come on top, so these methods store 101--103\% of FullKV (\cref{fig:teaser}a, \cref{fig:pareto}a). Since the cache stays resident for the whole generation, it is the stored footprint, not the per-step read, that limits the context length and batch size a device can serve.

This leads to a simple question: \emph{can the index itself replace the pages}, so that retrieval shrinks the cache instead of growing it?

\paragraph{Can the index be the storage?}
A digest only has to rank pages. To replace the keys, it must also preserve what the layer outputs. Existing digests and low-rank key approximations, however, are designed to preserve the logits~\citep{pmlr-v235-tang24l,hwang2026lockspagelocalcompactkey,singhania2024loki,pmlr-v267-sun25b}. A logit perturbation changes the output $o(q)=\sum_t\alpha_t(q)\,v_t$ only when it moves attention between tokens with different values. When attention within a page is close to uniform, two things happen:
\begin{itemize}[leftmargin=*,itemsep=1pt,topsep=2pt]
\item the output error is \emph{first order} in the discarded logit directions, weighted by the spread of the values;
\item the page-mass error that drives retrieval is \emph{second order} in expectation over isotropic queries (first order per query; \cref{sec:basis}).
\end{itemize}
So a basis that only looks at keys can rank pages well and still rebuild them poorly.

\method addresses this by building a page-local key basis that mixes key geometry and output sensitivity through a single coefficient $\eta\in[0,1]$. Values are stored as rank-$r_v$ PCA bases, and the exact pages are dropped. During decoding, the key basis scores the pages and only the selected ones are rebuilt (\cref{fig:teaser}a). Because one basis both indexes and stores each page, \method keeps the attention budget, and therefore the per-step read, of the retrieval methods it replaces. Quantization (bits per element) and visual-token pruning (which tokens enter the cache) act on other axes, and \method composes with both.

\paragraph{Contributions.}
\begin{itemize}[leftmargin=*,itemsep=2pt,topsep=2pt]
\item \textbf{A retrieval index that is also the storage.} In \method, each page is kept only as a low-rank key basis, which scores the page and rebuilds its keys on selection, plus a PCA value basis. The exact pages are dropped and no separate index is added, so \method stores 53--66\% of the full cache (\cref{sec:method,fig:pareto}).
\item \textbf{An output-aware basis with a closed-form mixing rule.} We analyze what the stored key basis must preserve (\cref{sec:basis}), which leads us to mix the attention-output error into the key-PCA objective. A closed-form rule sets this mix per head from unlabeled calibration traces, with no sweep. Replacing it with key PCA lowers accuracy in all but one long-reasoning setting (by up to 8.5 points; sign test $p{=}0.008$; \cref{sec:ablations}).
\item \textbf{Strong accuracy at a fraction of the memory.} At a matched attended budget, \method is best or tied-best in avg@8 in eight of the nine long-reasoning settings and Pareto-optimal in accuracy versus memory. On AIME24 with Qwen-14B, it is within 0.83 points of FullKV at 56\% of its memory (\cref{tab:reasoning,fig:pareto}). It is on par with FullKV on LongBench~v2 on all three models (\cref{tab:longbench}).
\item \textbf{Video and robot episodes.} On \mbox{Video-MME}~\citep{fu2025video} and LVBench~\citep{wang2024lvbench}, \method is the only method we evaluate that lowers both storage (53\%) and per-step read (about 15\%) without discarding tokens. It is on average 0.25 points below Quest and FreeKV, which store 103\% (\cref{sec:exp-video}). We also test recall of cached LIBERO~\citep{liu2023libero} robot episodes, acted on by a VLA policy.
\item \textbf{Composes with quantization and pruning.} With 4-bit KIVI-style quantization, \method is within 0.25 points of FullKV at 41\% of its memory on MATH500, and every quantized \method is at or above full-precision FreeKV on both tasks at 21--40\% of its memory (\cref{tab:abl-quant}). On the cache that VisionZip keeps, it stores about 27\% of FullKV (\cref{supp:tab:compose}).
\end{itemize}

Our code is available at \url{https://github.com/Ashkan13776/oval-kv}.
\section{Related Work}
\label{sec:related}

\paragraph{Page and token retrieval.}
\method belongs to the family of retrieval methods. These methods score cache entries cheaply at each step and attend only to the top-scoring ones. They differ mainly in how they summarize a page:
\begin{itemize}[leftmargin=*,itemsep=1pt,topsep=2pt]
\item Quest bounds page logits with per-channel key extremes~\citep{pmlr-v235-tang24l};
\item ArkVale uses bounding geometry and recalls evicted pages~\citep{chen2024arkvale};
\item UNIQUE uses key means and dispersion~\citep{deng2026uniqueuniversaltopksparse};
\item FreeKV selects pages speculatively with the previous query and corrects for drift~\citep{liu2026freekvboostingkvcache};
\item LOCKS scores pages with a page-local key PCA summary~\citep{hwang2026lockspagelocalcompactkey}, which is exactly \method's key basis at $\eta=0$.
\end{itemize}
All of these digests only \emph{rank} pages: they estimate or preserve the page logits, so every exact page is still kept next to its digest. Several retrieval methods offload the exact cache to host memory and fetch only the selected entries~\citep{298683,pmlr-v267-sun25b,chen2024arkvale,liu2026freekvboostingkvcache}. This moves the footprint off the GPU, but does not shrink it. In \method, by contrast, the index \emph{replaces} the page.

At the token level, methods score keys in reduced spaces~\citep{pmlr-v235-ribar24a,singhania2024loki,298683}, use quantized search, hashing, or clustering~\citep{10.1145/3725338,li2026rabitqcache,qi2026pariskv,chen2025magicpig,liu2025clusterkvmanipulatingllmkv}, learn the sparsity pattern~\citep{akhauri2026tokenbutlertokenimportancepredictable,yang2025attentionpredictortemporalpatternsmatter,gao2025seerattention,yuan-etal-2025-native,NEURIPS2025_19eae75b}, or adapt the budget~\citep{lin2025twilightadaptiveattentionsparsity}. Again, all of them keep the original entries.

\paragraph{Orthogonal axes.}
Other approaches shrink the cache along different axes, and can be combined with \method.
Grouped-query attention~\citep{ainslie-etal-2023-gqa} and quantization~\citep{hooper2024kvquant,yun2026outputawarerotationint2kvcache} make each entry smaller.
Eviction keeps a subset of entries and drops the rest for good. StreamingLLM keeps attention sinks and a recent window~\citep{xiao2024efficient}; H2O keeps heavy hitters by accumulated attention~\citep{NEURIPS2023_6ceefa7b}; SnapKV selects entries with an observation window at the end of the prompt~\citep{NEURIPS2024_28ab4182}; Ada-KV adapts the budget per head~\citep{feng2025adakvoptimizingkvcache}; and RaaS keeps the pages that recent reasoning steps attended~\citep{hu2025raasreasoningawareattentionsparsity}. RazorAttention and DuoAttention keep the full cache only in retrieval heads~\citep{tang2025razorattention,xiao2024duoattentionefficientlongcontextllm}.

Output-aware criteria have also been used for eviction and quantization. CriticalKV~\citep{feng2026criticalkv} and OBCache~\citep{pmlr-v306-gu26h} rank entries for eviction by the output perturbation their removal causes, and OptR~\citep{yun2026outputawarerotationint2kvcache} rotates the cache before 2-bit quantization to protect the output.
Finally, shared low-rank projections apply one basis to all tokens of a head. The basis may come from key PCA~\citep{singhania2024loki,pmlr-v267-sun25b,bhatnagar2026starkv,NEURIPS2025_00a0ebca}, the query--key interaction~\citep{pmlr-v300-lesens26a}, or the layer output error~\citep{benfenati2026dontstieflearningkv}; ShadowKV pairs it with retrieval. \method instead fits a separate output-aware basis to each page.

\paragraph{KV cache for video LLMs and robot episodes.}
In streaming and multi-turn use, a video is cached before the question, and the same cache serves every later question~\citep{ICLR2025_67a9b444,NEURIPS2025_caef5f5e,li2025scbench}. ReKV retrieves query-relevant video KV entries and offloads the rest~\citep{ICLR2025_67a9b444}. InfiniPot-V~\citep{NEURIPS2025_caef5f5e} and KVzip~\citep{NEURIPS2025_f4eaa4b8} compress the cache without knowing the query. Other methods quantize~\citep{vidkv2025}, reload chunks~\citep{videox2l2025}, or merge and prune visual tokens~\citep{Tao_2025_CVPR,fastv2024,visionzip2025,longvu2025}. Recalling robot episodes raises the same problem; we study it on LIBERO with the OpenVLA policy~\citep{kim2024openvla,liu2023libero}.
\section{Method}
\label{sec:method}

\method is a page-retrieval method whose index also stores the page. Each page keeps two things:
\begin{itemize}[leftmargin=*,itemsep=1pt,topsep=2pt]
\item a low-rank \emph{key basis}, built from key geometry and the attention-output error, which scores the page and replaces its exact keys;
\item a PCA \emph{value basis}, which replaces its exact values.
\end{itemize}
Nothing else is added to the cache.

\paragraph{Notation.} We describe a single KV head $h$; keys and queries include RoPE. $\mathcal G(h)$ is the set of query heads that share this KV head, and $W_{O,i}$ is the output-projection block of query head $i$. The cache is split into pages of $m$ tokens, with keys and values $K_P,V_P\in\R^{m\times d}$ (rows $k_t^\top$, $v_t^\top$). At each decoding step, the head scores all pages and attends over the top $n_p$.

\paragraph{Roadmap.} \Cref{sec:basis} derives the key basis from the output error it causes. \Cref{sec:keycomp,sec:valcomp} show how keys and values are rebuilt, and \cref{sec:etarule} explains how the mixing coefficient is set for each head.

\subsection{Output-aware page basis}
\label{sec:basis}

The key basis has two jobs. First, it \emph{scores} the page for retrieval, which depends on the page's attention mass. Second, it \emph{rebuilds} the page's keys for attention, which affects the page's output. We look at both through the error that a low-rank key basis introduces, in expectation over isotropic queries, $\E_q[qq^\top]=I_d$. This is a simplifying model, but the per-head rule built on it (\cref{sec:etarule}) lands close to the best $\eta$ in practice (\cref{supp:fig:eta-regret}). A whitened version for anisotropic queries is derived in \cref{supp:ext-whitened}.
Let $\mu_P=\frac1m K_P^\top\onev$ and $\nu_P=\frac1m V_P^\top\onev$ be the page
means and $X_P=K_P-\onev\mu_P^\top$, $Y_P=V_P-\onev\nu_P^\top$ the centered
keys and values; we drop the index $P$ when it is clear. For a query $q$, the
page's log attention mass is
\begin{equation}
\ell_P(q)=\log\sum_{t=1}^m e^{k_t^\top q/\sqrt d}
=\frac{\mu^\top q}{\sqrt d}+\lse\big(z(q)\big),
\label{eq:logmass}
\end{equation}
with $z(q)=Xq/\sqrt d$. Pages with a larger $\ell_P$ receive a larger share of the attention. A basis keeps $\mu$ exactly
and approximates the centered keys inside an $r$-dimensional subspace with
orthonormal basis $B\in\R^{d\times r}$ and projector $P_B=BB^\top$. This
perturbs the within-page logits by
$e_B(q)=X(I-P_B)q/\sqrt d$.

\paragraph{What the error does.}
The within-page attention output is $f(z)=V^\top\softmax(z)=\nu+Y^\top\softmax(z)$.
Expanding around uniform within-page attention gives
(\cref{supp:first-order,supp:mass}):
\begin{align}
f(z)-f(\widehat z_B)&=\tfrac1m\,Y^\top e_B(q)+O\big(\rho\norm{e_B(q)}\big),
\label{eq:out-first}\\
\ell_P(q)-\widehat\ell_P(q)&=\tfrac1{2m}\big(\norm{e_B(q)}^2+2\,\widehat z_B^\top e_B(q)\big)\notag\\
&\quad+O\big(\rho^3\big),
\label{eq:mass-second}
\end{align}
where $\widehat z_B=XP_Bq/\sqrt d$ and $\rho$ bounds $\norm{z}$ and
$\norm{\widehat z_B}$.

These two expansions behave quite differently. The output error is \emph{first} order in the discarded logit directions, and it is weighted by the value deviations $Y$. The mass error that drives retrieval is only \emph{second} order in expectation over isotropic queries, because the cross term vanishes (per query, it is still first order). The takeaway is that a basis serving both retrieval and attention should look at the values (\cref{supp:prop:counter}), unlike key-only summaries~\citep{pmlr-v235-tang24l,hwang2026lockspagelocalcompactkey,singhania2024loki}. Note that the value covariance $Y^\top Y$ by itself is not enough: the output depends on how keys and values \emph{pair} (\cref{supp:covariance}).

\paragraph{Hybrid objective.}
The analysis above is local. The output term is the leading error only near uniform attention (\cref{supp:prop:first-order}), while key energy bounds both errors at any logit scale (\cref{supp:prop:diameter,supp:mass}). Since neither criterion is optimal in general, we write both errors as quadratic risks in $B$ and mix them.
We measure output errors in the residual stream, through the head's output projection $G_h=\sum_{i\in\mathcal G(h)}W_{O,i}^\top W_{O,i}$. The expected squared logit
error and linearized output error are then
\begin{equation}
\begin{aligned}
\cR_{\mathrm{key}}(B)&=\tfrac1d\trace\big((I-P_B)\,X^\top X\,(I-P_B)\big),\\
\cR_{\mathrm{out}}(B)&=\tfrac1{m^2d}\trace\big((I-P_B)\,M_{\mathrm{out}}\,(I-P_B)\big),
\end{aligned}
\label{eq:risks}
\end{equation}
with $M_{\mathrm{out}}=X^\top YG_hY^\top X$. Normalizing each risk by its
total energy and mixing them with $\eta\in[0,1]$ gives one quadratic risk
with matrix
\begin{equation}
M_\eta=(1-\eta)\,\frac{X^\top X}{\trace(X^\top X)}+\eta\,\frac{M_{\mathrm{out}}}{\trace(M_{\mathrm{out}})}.
\label{eq:Meta}
\end{equation}
By the Ky Fan principle~\citep{fan1949theorem}, the optimal $B$ spans the top-$r$ eigenvectors of $M_\eta$ (\cref{supp:kyfan,supp:cor:hybrid}). The two ends of the range are easy to interpret: $\eta=0$ is page-local key PCA (the LOCKS summary, \cref{sec:related}), and $\eta=1$ keeps the directions that move the output the most. Values in between trade key energy for output sensitivity, and \cref{sec:etarule} picks $\eta$ for each head.

\paragraph{Computation.}
Since $m<d$, we can solve the eigenproblem at size $m\times m$. With
\begin{equation*}
A_\eta=\frac{1-\eta}{\trace(XX^\top)}\,I_m+\frac{\eta\;YG_hY^\top}{\trace(YG_hY^\top XX^\top)}
\end{equation*}
and $S=A_\eta^{1/2}X$, we have $M_\eta=S^\top S$; if $SS^\top U=U\Lambda$
holds the top $r$ eigenpairs, then $B_P=S^\top U\Lambda^{-1/2}$
(\cref{supp:small-gram}). The basis is
\begin{equation}
(\mu_P,\;B_P,\;C_P),\qquad C_P=X_PB_P\in\R^{m\times r}.
\label{eq:record}
\end{equation}
The basis is built once, when the page is complete, in $O(m^2d+m^3)$ time, and $G_h^{1/2}$ is computed once per layer.

\paragraph{Selecting pages from the basis.}
For a query head $i\in\mathcal G(h)$, the score of page $P$ is the log-mass of
its reconstructed keys $\widehat k_t=\mu_P+B_PC_{P,t}^\top$:
\begin{equation}
s_i(P)=\lse_t\big(q_i^\top\widehat k_t/\sqrt d\big)=\frac{\mu_P^\top q_i}{\sqrt d}+\lse\Big(\frac{C_P(B_P^\top q_i)}{\sqrt d}\Big),
\label{eq:score}
\end{equation}
This costs $O(dr+mr)$ per page instead of $O(md)$. The query heads of a group are then combined as
\begin{equation*}
S_h(P)=\tfrac{1}{|\mathcal G(h)|}\textstyle\sum_{i\in\mathcal G(h)}\softmax_P\big(s_i(P)\big),
\end{equation*}
and each KV head selects its $n_p$ pages with the largest $S_h$.

\subsection{Key reconstruction from the same basis}
\label{sec:keycomp}

Since the basis is a rank-$r$ factorization of the centered keys, it can stand in for them directly. The selected pages are rebuilt as
\begin{equation}
\widehat K_P=\onev\mu_P^\top+C_PB_P^\top=\onev\mu_P^\top+X_PP_{B_P},
\label{eq:keyrec}
\end{equation}
and attention uses $\widehat K_P$ in place of $K_P$. The logits of a rebuilt page are $\mu_P^\top q/\sqrt d+\widehat z_B(q)$, which is exactly the approximation analyzed in \cref{sec:basis}. So \cref{eq:out-first} gives the error of attending to rebuilt keys, and the output term of $M_\eta$ protects the directions this error depends on (\cref{supp:recon}). The same rank $r$ is used for scoring and reconstruction, and only the $n_p$ selected pages are rebuilt, at a cost of $O(n_pmdr)$ per step.

\subsection{Value basis}
\label{sec:valcomp}

Values are never scored, so their basis only has to rebuild them. Each page stores a rank-$r_v$ PCA basis $(\nu_P,D_P,T_P)$, where $D_P\in\R^{d\times r_v}$
holds the top-$r_v$ right singular vectors of $Y_P$ and $T_P=Y_PD_P$, and the
selected values are rebuilt as
\begin{equation}
\widehat V_P=\onev\nu_P^\top+T_PD_P^\top,
\label{eq:valrec}
\end{equation}
the best rank-$r_v$ approximation of the centered values.

\subsection{A closed-form rule of thumb for $\eta$}
\label{sec:etarule}

The best balance between key fidelity ($\eta=0$) and output sensitivity ($\eta=1$) need not be the same for every head. We could sweep $\eta$ against labels, but that costs a full decode per value. Instead, we measure each head's output error at the two endpoints on a few calibration traces, and take the minimizer of a simple interpolation in closed form. Let $o_i(q)$ be the FullKV output of query head
$i$ and $\widehat o_i^{\,\eta}(q)$ the output of \method at $\eta$. The error KV head $h$
adds to the residual stream is
\begin{equation}
\Delta_h^{\eta}(q)=\sum_{i\in\mathcal G(h)}W_{O,i}\big(o_i(q)-\widehat o_i^{\,\eta}(q)\big).
\label{eq:delta}
\end{equation}
Summing over the calibration queries,
\begin{equation}
\begin{gathered}
a_h=\textstyle\sum_q\norm{\Delta_h^{(0)}(q)}^2,\qquad b_h=\textstyle\sum_q\norm{\Delta_h^{(1)}(q)}^2,\\
c_h=\textstyle\sum_q\big\langle\Delta_h^{(0)}(q),\Delta_h^{(1)}(q)\big\rangle.
\end{gathered}
\label{eq:abc}
\end{equation}
Modeling $\Delta_h^{\eta}\approx(1-\eta)\Delta_h^{(0)}+\eta\Delta_h^{(1)}$
gives the convex quadratic
$f_h(\eta)=(1-\eta)^2a_h+\eta^2b_h+2\eta(1-\eta)c_h$, minimized by
\begin{equation}
\eta_h=\operatorname{clip}_{[0,1]}\!\left(\frac{a_h-c_h}{a_h+b_h-2c_h}\right).
\label{eq:etarule}
\end{equation}
For most heads, the measured error is close to $f_h$. On held-out MATH500 traces (Qwen-7B, Llama-8B), the median excess over the best grid $\eta$ is below 1\% (\cref{supp:eta-rule,supp:fig:eta-regret}). Because $\Delta_h$ is measured with real queries, the rule also partly makes up for what the isotropic model misses.

Each model is calibrated once per setting, offline (\cref{sec:setup,supp:calib}). Since $\eta$ only enters the basis construction, the per-head $\eta_h$ does not change the decode cost. \Cref{alg:oval} puts all the pieces of \method together.

\begin{algorithm}[t]
\caption{\method for one KV head $h$ (one layer).}
\label{alg:oval}
\small
\begin{algorithmic}[1]
\REQUIRE ranks $r$ (keys), $r_v$ (values); page budget $n_p$; $\eta_h$ from \cref{eq:etarule}, computed once offline
\STATE \textbf{when a page $P$ of $m$ tokens is complete:}
\STATE \quad $\mu_P,\nu_P\leftarrow$ means of $K_P,V_P$;\; $X_P,Y_P\leftarrow$ centered keys, values
\STATE \quad $B_P\leftarrow$ top-$r$ eigenvectors of $M_{\eta_h}$ (\cref{eq:Meta});\; $C_P\leftarrow X_PB_P$
\STATE \quad $(D_P,T_P)\leftarrow$ rank-$r_v$ PCA of $Y_P$
\STATE \quad store $(\mu_P,B_P,C_P)$ and $(\nu_P,D_P,T_P)$;\; drop $K_P,V_P$
\STATE \textbf{at each decode step, queries $q_i$, $i\in\mathcal G(h)$:}
\STATE \quad $s_i(P)\leftarrow\mu_P^\top q_i/\sqrt d+\lse\big(C_P(B_P^\top q_i)/\sqrt d\big)$ for every page (\cref{eq:score})
\STATE \quad $\mathcal S\leftarrow$ the $n_p$ pages with the largest $S_h(P)$
\STATE \quad \textbf{for} $P\in\mathcal S$: $\widehat K_P\leftarrow\onev\mu_P^\top+C_PB_P^\top$,\; $\widehat V_P\leftarrow\onev\nu_P^\top+T_PD_P^\top$
\STATE \quad attend over exact sink, $\{(\widehat K_P,\widehat V_P)\}_{P\in\mathcal S}$, exact recent
\end{algorithmic}
\end{algorithm}
\section{Experiments}
\label{sec:exp}
\subsection{Setup}
\label{sec:setup}

We evaluate \method in three regimes: long reasoning (long outputs), long-context understanding and generation (long inputs and outputs), and video and robot episodes (long visual inputs).

\paragraph{Models and benchmarks.}
\textit{Long reasoning.} DeepSeek-R1-Distill Llama-8B, Qwen-7B, and Qwen-14B~\citep{Guo2025} on
MATH500~\citep{hendrycksmath2021,lightman2024lets}, AIME24~\citep{zhang2024accessing,zhang2024llama}, and
GPQA~\citep{rein2024gpqa} (50, 30, and 50 problems; $k=8$ samples per problem, avg@8 and pass@8).
\textit{Long context.} Llama-3.1-8B, Qwen2.5-7B, and Qwen2.5-14B
Instruct~\citep{qwen2025qwen25technicalreport,grattafiori2024llama3herdmodels} on
LongBench~v2~\citep{bai-etal-2025-longbench} and LongGenBench~\citep{ICLR2025_141304a3}.
For both of these settings, we follow FreeKV's evaluation~\citep{liu2026freekvboostingkvcache} exactly: same problems, seeds, prompts, decoding, budgets, and metrics.
\textit{Video and robot episodes.} LLaVA-Video-7B-Qwen2~\citep{zhang2024videoinstructiontuningsynthetic}
and Qwen2.5-VL-7B-Instruct~\citep{bai2025qwen25vl} on \mbox{Video-MME}~\citep{fu2025video} and
LVBench~\citep{wang2024lvbench}. We also build LIBERO Watch-and-Do, a robot-episode recall task based on LIBERO~\citep{liu2023libero}, whose answers are executed by the OpenVLA policy~\citep{kim2024openvla} (\cref{sec:exp-video}).

\paragraph{Baselines.}
All retrieval and eviction methods attend to the same per-step budget, including the exact sink and recent tokens
(\cref{supp:tab:protocols,supp:protocols}). For long reasoning and long context, we compare with the retrieval methods Quest~\citep{pmlr-v235-tang24l},
ArkVale~\citep{chen2024arkvale}, ShadowKV~\citep{pmlr-v267-sun25b}, InfiniGen~\citep{298683}, and FreeKV, and the
eviction methods RazorAttention~\citep{tang2025razorattention} and RaaS~\citep{hu2025raasreasoningawareattentionsparsity}.
For video, we compare with Quest, FreeKV, the retriever ReKV~\citep{ICLR2025_67a9b444}, and the token reducer InfiniPot-V~\citep{NEURIPS2025_caef5f5e}. SnapKV~\citep{NEURIPS2024_28ab4182} ($\dagger$) and VisionZip~\citep{visionzip2025} (keeping 50\% and 25\% of the visual tokens) are included for context.
FullKV is the reference throughout; \cref{supp:read} lists which numbers are taken from prior work and which we rerun.

\paragraph{\method configuration.}
Every setting uses the same configuration: 32-token pages, rank-8 key bases for scoring and reconstruction, rank-16 PCA value bases, and a per-head $\eta_h$ from the rule of \cref{sec:etarule}.
A rank-$r$ basis with its coefficients costs $(dr+mr+d)\cdot2$\,B per page and KV head ($d=128$). The exact keys or values of a page take 8192\,B, so the two bases together take half the memory of the exact page.

The rule is always calibrated on held-out data:
\begin{itemize}[leftmargin=*,itemsep=1pt,topsep=2pt]
\item \emph{long reasoning}: once per model, on six MATH500 problems outside the evaluation subset (four to fit, two to validate; separate sampling seed), shared by the three tasks;
\item \emph{long context}: on held-out documents and prompts;
\item \emph{video}: once per model, on six LLaVA-Video-178K clips, and reused for the robot task.
\end{itemize}
Full details are in \cref{supp:impl,supp:protocols,supp:calib}.

\begin{figure*}[t]
\centering
\includegraphics[width=\textwidth]{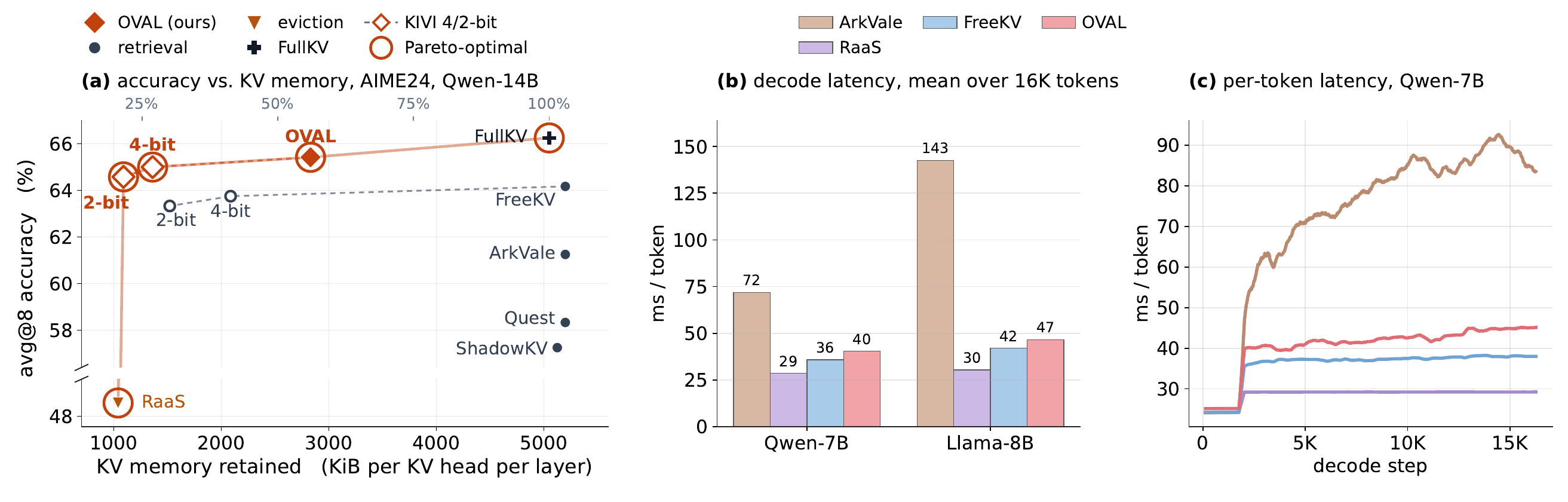}
\caption{Memory and decode latency on long reasoning. \textbf{(a)}~Accuracy versus retained KV
memory per KV head and layer on AIME24 (R1-Distill Qwen-14B; broken $y$-axis; top axis: share of FullKV's memory);
FullKV and five baselines of \cref{tab:reasoning}; hollow markers on dashed lines: FreeKV and \method with
4- and 2-bit KIVI-style quantization (\cref{tab:abl-quant}); large rings: Pareto-optimal. \textbf{(b)}~Mean decode latency
(Qwen-7B, Llama-8B) over a 16K-token AIME24 generation in FreeKV's engine (decode time only). \textbf{(c)}~Per-token
latency along that generation (Qwen-7B). Details: \cref{supp:exp-pareto,supp:latency}.}
\label{fig:pareto}
\end{figure*}
\subsection{Long reasoning}
\label{sec:reasoning}

\begin{table*}[t]
\centering
\caption{Long reasoning (\%, $k=8$). Average: mean of the three tasks. Subscripts: SE over 8 seeds. Full: FullKV, not ranked; Razor: RazorAttention. Gold and silver: best and second-best distinct score among the
baselines and \method. Full and baseline numbers are as reported by
FreeKV~\citep{liu2026freekvboostingkvcache}, including its rounding; \method runs in its exact setting and harness (\cref{sec:setup}, \cref{supp:same-setting}).}
\label{tab:reasoning}
\setlength{\tabcolsep}{4pt}
\small
\begin{tabular*}{\textwidth}{@{\extracolsep{\fill}}ll|c|ccccccc|c@{}}
\toprule
\textbf{Benchmark} & \textbf{Metric} & Full & Razor & RaaS & Quest & ArkVale & ShadowKV & InfiniGen & FreeKV & \textbf{\method} \\
\midrule
\multicolumn{11}{c}{\emph{R1-Distill Llama-8B}} \\
\midrule
\multirow{2}{*}{MATH500}
 & pass@$k$ & 78.00 & 72.00 & 74.00 & 72.00 & 72.00 & \second{76.00} & 74.00 & \best{78.00} & \best{78.00}\\
 & avg@$k$ & 67.25 & 60.50 & 62.50 & 62.00 & 62.50 & 60.25 & 66.50 & \second{66.75} & \best{67.25$_{\pm0.8}$} \\
\midrule
\multirow{2}{*}{AIME24}
 & pass@$k$ & 80.00 & 46.67 & 66.67 & 73.33 & \best{80.00} & 63.33 & 70.00 & \second{76.67} & \best{80.00} \\
 & avg@$k$ & 47.08 & 30.00 & 36.25 & 44.17 & \second{46.67} & 36.50 & 45.83 & \best{47.50} & \best{47.50$_{\pm1.4}$} \\
\midrule
\multirow{2}{*}{GPQA}
 & pass@$k$ & 82.00 & 60.00 & 64.00 & 76.00 & 72.00 & \second{78.00} & 46.00 & \best{86.00} & 76.00 \\
 & avg@$k$ & 39.75 & 34.25 & 33.50 & 37.25 & 39.75 & 36.25 & 27.50 & \second{41.25} & \best{44.00$_{\pm1.1}$} \\
\midrule
\multirow{2}{*}{\textbf{Average}}
 & pass@$k$ & 80.00 & 59.56 & 68.22 & 73.78 & 74.67 & 72.44 & 63.33 & \best{80.22} & \second{78.00} \\
 & avg@$k$ & 51.36 & 41.58 & 44.08 & 47.81 & 49.64 & 44.33 & 46.61 & \second{51.83} & \best{52.92$_{\pm0.5}$} \\
\midrule
\multicolumn{11}{c}{\emph{R1-Distill Qwen-7B}} \\
\midrule
\multirow{2}{*}{MATH500}
 & pass@$k$ & 78.00 & 72.00 & 74.00 & \second{76.00} & \second{76.00} & 74.00 & 74.00 & \best{78.00} & \best{78.00} \\
 & avg@$k$ & 71.75 & 66.75 & 67.00 & 68.00 & 68.25 & 64.75 & \second{70.00} & \second{70.00} & \best{71.75$_{\pm0.5}$} \\
\midrule
\multirow{2}{*}{AIME24}
 & pass@$k$ & 83.33 & 65.33 & 73.33 & 76.67 & 73.33 & 73.33 & 63.33 & \best{83.33} & \second{80.00} \\
 & avg@$k$ & 56.66 & 35.42 & 42.92 & 47.50 & \second{47.92} & 43.75 & 43.34 & \best{52.92} & \second{47.92$_{\pm1.3}$} \\
\midrule
\multirow{2}{*}{GPQA}
 & pass@$k$ & 72.00 & 60.00 & 58.00 & 72.00 & 72.00 & 70.00 & 58.00 & \second{74.00} & \best{76.00} \\
 & avg@$k$ & 35.75 & 32.50 & 33.25 & \second{38.75} & 34.25 & 33.50 & 35.50 & \best{39.50} & \best{39.50$_{\pm1.2}$} \\
\midrule
\multirow{2}{*}{\textbf{Average}}
 & pass@$k$ & 77.78 & 65.78 & 68.44 & 74.89 & 73.78 & 72.44 & 65.11 & \best{78.44} & \second{78.00} \\
 & avg@$k$ & 54.72 & 44.89 & 47.72 & 51.42 & 50.14 & 47.33 & 49.61 & \best{54.14} & \second{53.06$_{\pm0.6}$} \\
\midrule
\multicolumn{11}{c}{\emph{R1-Distill Qwen-14B}} \\
\midrule
\multirow{2}{*}{MATH500}
 & pass@$k$ & 74.00 & 70.00 & 68.00 & \second{76.00} & 72.00 & \second{76.00} & 72.00 & \best{78.00} & \best{78.00} \\
 & avg@$k$ & 70.25 & 59.75 & 64.75 & 67.25 & 66.25 & 65.00 & 67.00 & \second{67.50} & \best{69.25$_{\pm0.7}$} \\
\midrule
\multirow{2}{*}{AIME24}
 & pass@$k$ & 86.67 & 46.67 & 73.33 & \best{83.33} & \second{76.67} & \best{83.33} & \second{76.67} & \best{83.33} & \best{83.33} \\
 & avg@$k$ & 66.25 & 32.50 & 48.75 & 58.33 & 61.25 & 57.25 & 60.00 & \second{64.17} & \best{65.42$_{\pm1.1}$} \\
\midrule
\multirow{2}{*}{GPQA}
 & pass@$k$ & 82.00 & 68.00 & \second{80.00} & \second{80.00} & \best{86.00} & \best{86.00} & 58.00 & \best{86.00} & \best{86.00} \\
 & avg@$k$ & 53.25 & 38.50 & 44.25 & 51.25 & 53.75 & 51.75 & 38.00 & \second{56.00} & \best{57.50$_{\pm0.9}$} \\
\midrule
\multirow{2}{*}{\textbf{Average}}
 & pass@$k$ & 80.89 & 61.56 & 73.78 & 79.78 & 78.22 & \second{81.78} & 68.89 & \best{82.44} & \best{82.44} \\
 & avg@$k$ & 63.25 & 43.58 & 52.58 & 58.94 & 60.42 & 58.00 & 55.00 & \second{62.56} & \best{64.06$_{\pm0.5}$} \\

\bottomrule
\end{tabular*}
\end{table*}

\begin{table*}[t]
\centering
\caption{LongBench~v2 accuracy (\%, overall and by length; single run) and LongGenBench
completion rate CR and CR$\times$Acc (\%). Full, ranking, and
baseline sources as in \cref{tab:reasoning}. 
\label{tab:longbench}\label{tab:longgen} }
\setlength{\tabcolsep}{4pt}
\small
\begin{tabular*}{\textwidth}{@{\extracolsep{\fill}}ll|c|ccccccc|c@{}}
\toprule
\textbf{Benchmark} & \textbf{Metric} & Full & Razor & RaaS & Quest & ArkVale & ShadowKV & InfiniGen & FreeKV & \textbf{\method} \\
\midrule
\multicolumn{11}{c}{\emph{Llama-3.1-8B Instruct}} \\
\midrule
\multirow{4}{*}{LongBench v2}
 & Overall & 29.22 & 27.44 & 28.23 & 28.43 & 28.63 & 25.45 & 28.56 & \second{29.22} & \best{30.02} \\
 & Short & 34.44 & 33.89 & 33.89 & 33.33 & 33.89 & 32.78 & 32.28 & \second{35.00} & \best{36.11} \\
 & Medium & 27.91 & 25.12 & 26.51 & \best{27.44} & \second{26.98} & 22.79 & 26.05 & \best{27.44} & \best{27.44} \\
 & Long & 23.15 & 21.30 & 22.22 & 22.22 & 23.15 & 18.52 & \second{24.07} & 23.15 & \best{25.00} \\
\midrule
\multirow{2}{*}{LongGenBench}
 & CR & 80.03 & 35.90 & 76.63 & \second{78.03} & 39.36 & \best{79.28} & 76.68 & \second{78.03} & 77.87 \\
 & CR$\times$Acc & 26.82 & 12.20 & 26.00 & 27.71 & 10.36 & \best{30.66} & 26.21 & 27.62 & \second{28.64} \\
\midrule
\multicolumn{11}{c}{\emph{Qwen2.5-7B Instruct}} \\
\midrule
\multirow{4}{*}{LongBench v2}
 & Overall & 27.44 & 25.25 & 26.24 & \second{27.63} & 26.84 & 25.84 & 26.44 & 26.84 & \best{29.82} \\
 & Short & 36.11 & 32.78 & 35.56 & \best{36.67} & \second{36.11} & 32.22 & 32.22 & 34.44 & \best{36.67} \\
 & Medium & 23.72 & 21.86 & 21.86 & 22.79 & 22.33 & 20.00 & \second{23.26} & 22.33 & \best{27.44} \\
 & Long & 20.37 & 19.44 & 19.44 & 22.22 & 20.37 & \best{26.85} & \second{23.15} & \second{23.15} & \second{23.15} \\
\midrule
\multirow{2}{*}{LongGenBench}
 & CR & 79.56 & 42.13 & \best{77.65} & 62.89 & 75.91 & 35.49 & 72.96 & 76.93 & \second{77.23} \\
 & CR$\times$Acc & 31.09 & 21.48 & \second{34.40} & 25.96 & 31.79 & 11.43 & 27.67 & 32.81 & \best{35.00} \\
\midrule
\multicolumn{11}{c}{\emph{Qwen2.5-14B Instruct}} \\
\midrule
\multirow{4}{*}{LongBench v2}
 & Overall & 33.40 & 34.19 & 32.60 & 33.80 & 34.19 & \second{34.79} & 32.31 & 34.19 & \best{34.99} \\
 & Short & 41.11 & \best{43.33} & 40.56 & 40.00 & \second{41.11} & 40.56 & 40.56 & \second{41.11} & \second{41.11} \\
 & Medium & 31.16 & 30.70 & 32.09 & \second{33.49} & \second{33.49} & \best{34.88} & 29.89 & \second{33.49} & \best{34.88} \\
 & Long & 25.00 & \best{25.93} & 20.37 & 24.07 & 24.07 & \second{25.00} & 23.95 & 24.07 & \second{25.00} \\
\midrule
\multirow{2}{*}{LongGenBench}
 & CR & 65.84 & 26.48 & 62.29 & 45.49 & 45.31 & 21.25 & 63.02 & \best{65.46} & \second{64.06} \\
 & CR$\times$Acc & 29.35 & 13.89 & \best{29.79} & 19.76 & 19.65 & 8.25 & 26.08 & \second{29.39} & 28.83\\

\bottomrule
\end{tabular*}
\end{table*}

\begin{table*}[t!]
\centering
\footnotesize
\setlength{\tabcolsep}{1.5pt}
\renewcommand{\arraystretch}{0.93}
\resizebox{\textwidth}{!}{%
\begin{tabular}{@{}llllcccccccccccc@{}}
\toprule
 & & & & \multicolumn{8}{c}{\mbox{Video-MME}} & \multicolumn{2}{c}{LVBench} & \multicolumn{2}{c}{W\&D success} \\
\cmidrule(lr){5-12}\cmidrule(lr){13-14}\cmidrule(l){15-16}
 & & & & \multicolumn{4}{c}{Qwen2.5-VL} & \multicolumn{4}{c}{LLaVA-Video} & Qwen & LLaVA & Qwen & LLaVA \\
\cmidrule(lr){5-8}\cmidrule(lr){9-12}
\textbf{Method} & \textbf{Stored} (\%, Q/L) & \textbf{Read} (\%, Q/L) & \textbf{KV} (MiB, Q/L) & S & M & L & All & S & M & L & All & All & All & All & All \\
\midrule
\multicolumn{16}{@{}l}{\scriptsize\emph{Reference}} \\[-1.5pt]
FullKV ($N$) & 100\,/\,100 & 101.0\,/\,104.8 & 725\,/\,771 & 75.11 & 63.78 & 52.22 & 63.70 & 75.56 & 61.67 & 51.44 & 62.89 & 41.25 & 41.64 & 40.83 & 32.08 \\
FullKV ($N/2$) & 50.5\,/\,50.6 & 51.0\,/\,53.0 & 366\,/\,390 & 72.67 & 61.11 & 51.22 & 61.67 & 74.00 & 57.56 & 51.00 & 60.85 & 38.41 & 39.38 & -- & -- \\
\midrule
\multicolumn{16}{@{}l}{\scriptsize\emph{Fewer tokens cached (read $=$ stored)}} \\[-1.5pt]
VisionZip$^\ddagger$ 50\% & 50.5\,/\,52.3 & 51.0\,/\,54.8 & 366\,/\,403 & 74.56 & 62.33 & 51.89 & 62.93 & 76.00 & 62.33 & 52.00 & 63.44 & 39.90 & 41.51 & -- & -- \\
InfiniPot-V$^\ddagger$ 50\% & 50.5\,/\,52.3 & 51.0\,/\,54.8 & 366\,/\,403 & 73.22 & 63.44 & 51.67 & 62.78 & 74.22 & 61.44 & 50.56 & 62.07 & 40.41 & 41.64 & -- & -- \\
VisionZip$^\ddagger$ 25\% & 25.7\,/\,28.5 & 26.0\,/\,29.8 & 186\,/\,219 & 72.44 & 60.89 & 52.00 & 61.78 & 72.89 & 60.56 & 49.89 & 61.11 & 38.48 & 40.35 & -- & -- \\
\midrule
\multicolumn{16}{@{}l}{\scriptsize\emph{$\approx$2K tokens read per step}} \\[-1.5pt]
InfiniPot-V$^\ddagger$ & 16.4\,/\,19.5 & 16.6\,/\,20.5 & 119\,/\,150 & 68.00 & \topthree{58.11} & 49.78 & 58.63 & 70.00 & \toptwo{57.89} & 47.78 & 58.56 & 36.60 & 38.99 & 37.92 & 24.17 \\
ReKV$^\ddagger$ & 100.1\,/\,100.2$^{\S}$ & 14.8\,/\,15.0 & 726\,/\,772$^{\S}$ & 69.67 & 57.89 & \topthree{52.00} & 59.85 & \topthree{70.89} & \topthree{57.78} & \topone{50.44} & 59.70 & 37.38 & 39.70 & 39.58 & 24.17 \\
SnapKV$^\dagger$ & 18.5\,/\,17.6 & 15.6\,/\,15.2 & 134\,/\,136 & 64.33 & 59.00 & 49.44 & 57.59 & 64.33 & 56.00 & 46.67 & 55.67 & 36.09 & 36.22 & 37.92 & 30.00 \\
Quest & 103.0\,/\,103.0 & 15.6\,/\,15.2 & 747\,/\,794 & \topone{74.22} & \topone{62.22} & \topone{52.67} & \topone{63.04} & \topone{75.00} & \topone{60.56} & \toptwo{49.89} & \topone{61.81} & \toptwo{38.99} & \toptwo{41.32} & \topthree{42.08} & \topone{34.58} \\
FreeKV & 103.0\,/\,103.0 & 15.6\,/\,15.2 & 747\,/\,794 & \toptwo{73.78} & \toptwo{61.22} & \toptwo{52.56} & \toptwo{62.52} & \toptwo{74.44} & \topone{60.56} & \topthree{49.33} & \toptwo{61.44} & \topthree{38.48} & \topone{42.16} & \topone{42.92} & \toptwo{32.50} \\
\method{} & 53.1\,/\,53.0 & 15.6\,/\,15.2 & 385\,/\,409 & \topthree{72.56} & \topone{62.22} & 51.00 & \topthree{61.93} & \toptwo{74.44} & \topone{60.56} & 49.11 & \topthree{61.37} & \topone{39.70} & \topthree{40.87} & \toptwo{42.50} & \topthree{30.83} \\
\bottomrule
\end{tabular}}
\caption{Accuracy (\%) on \mbox{Video-MME} (S/M/L, All) and LVBench, and Watch-and-Do task success (\%); ``--'': not run.
Every row is our run on the same questions.
Stored: stored KV cache (\cref{supp:exp-pareto}) as a share of FullKV with $N{=}64$ frames (\mbox{Video-MME}).
Q/L: Qwen2.5-VL\,/\,LLaVA-Video. Read: tokens attended per decoding step (layers 1--27, text included) as a share of the visual tokens.
$^\dagger$: upper bound, not ranked. $^\ddagger$: our port. $^{\S}$: keeps the full cache.
Bold, dark, light: best three per column within the $\approx$2K-read block.
W\&D: OpenVLA success, 240 episodes (oracle command 75.83\%, random 21.25\%); KV as \mbox{Video-MME}
(\cref{supp:tab:watchdo-full}).}
\label{tab:videomme}\label{tab:watchdo}
\end{table*}

\begin{figure*}[t!]
\centering
\includegraphics[width=0.82\textwidth]{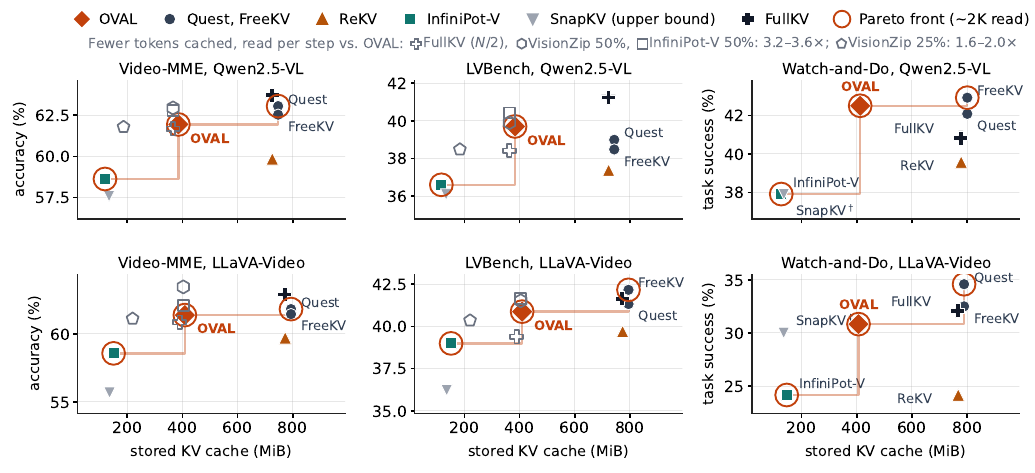}
\caption{\textbf{Accuracy versus stored KV cache} (Watch-and-Do: task
success; ReKV at its full stored cache). Rings: Pareto front among methods that read ${\approx}$2K tokens per step (\method, Quest, FreeKV, ReKV, InfiniPot-V; SnapKV$^\dagger$ excluded); FullKV is a reference.
Hollow grey: token pruning and fewer frames, which read every token they cache; legend: read per step relative to \method (layers 1--27).}
\label{fig:video-pareto}
\end{figure*}

\paragraph{Accuracy.}
\Cref{tab:reasoning} shows the results. Averaged over the three tasks, \method's avg@8 is 1.6 points above FullKV on Llama-8B and 0.8 points above on Qwen-14B, and 1.7 points below on Qwen-7B. It has the highest mean of all methods on Llama-8B and Qwen-14B, and it is best or tied-best in all but one avg@8 cell. Interestingly, it exceeds FullKV on GPQA for every model, by 3.75--4.25 points.

On MATH500 and AIME24, \method is within 1 point of FullKV in five of the six cells. The one exception is Qwen-7B on AIME24 (47.92 vs.\ 56.66), and this gap mostly closes with rank-16 keys, which reach 55.42 (\cref{tab:abl-keyrank}). In pass@8, \method is best or tied-best in seven of the nine cells.

\paragraph{Accuracy versus stored memory.}
Accuracy alone does not tell the full story, since the methods store very different amounts. \Cref{fig:pareto}a plots accuracy against stored memory on AIME24 with Qwen-14B. The retrieval baselines all store 101--103\% of FullKV's memory (FullKV: 5050\,KiB). \method stores 2830\,KiB, which is 56\% of FullKV's memory and 54\% of FreeKV's. At this footprint, it is 1.25 points above FreeKV and 0.83 below FullKV (65.42 vs.\ 64.17 and 66.25\%).

Adding KIVI quantization (\cref{tab:abl-quant}), the \method variants form the whole Pareto front between eviction (RaaS) and FullKV. In a similar MATH500 study with Qwen-7B, \method matches FullKV (71.75\%) at 66\% of its memory, as does its exact-page variant at 115\% (\cref{supp:pareto}).
\subsection{Long-context understanding and generation}
\label{sec:longcontext}
We test long inputs on LongBench~v2 (503 questions, 64K context) and long outputs on LongGenBench
(16K tokens).

\paragraph{Results.}
On LongBench~v2, \method is on par with FullKV (\cref{tab:longbench}), and it has the highest overall accuracy on each model. Broken down by length, it is best or tied-best on all three splits for Llama-3.1-8B, and on the Medium split for all three models.

On LongGenBench, \method's completion rate (CR) is within 2.33 points of FullKV. Its CR$\times$Acc is at most 0.52 points below FullKV ($+$1.82, $+$3.91, and $-$0.52 points across the three models), and it is the highest of all methods on Qwen2.5-7B.
\subsection{Video understanding}

\label{sec:exp-video}
Video is where the cache is largest, so it is a natural test for \method.

\paragraph{Setup.}
\mbox{Video-MME} (2700 questions) and LVBench (1549 questions on hour-long videos) use 64 frames. That is about 13K visual tokens, nearly all of the cache, and they are cached before the question is asked.
We report two quantities: \emph{Stored}, the GPU + host KV cache relative to FullKV, and \emph{Read}, the tokens attended per step relative to the visual tokens.
All rows of \cref{tab:videomme} use the same questions, so comparisons between them are paired (intervals in \cref{supp:video}).

\paragraph{Comparison with exact-page retrieval.}
Quest, FreeKV, and \method all read 2048 tokens per step, with the same page size and the same exact sink and recent windows. The difference is what they attend to: Quest and FreeKV attend to exact pages chosen by digests, while \method attends to rebuilt pages chosen by its bases.
At 52\% of their storage, \method is within 0.32 and 0.18 points of Quest and FreeKV on average, and both paired intervals contain zero ([$-$0.96, $+$0.30] and [$-$0.82, $+$0.44]).
All three trail FullKV, by 1.1, 1.2, and 1.4 points on average. So most of \method's gap to FullKV comes from reading only 15\% of the tokens, not from the bases themselves.

\paragraph{Other methods at the same read.}
FullKV leads \method by 0.8--1.8 points. Among methods with a similar read, \method leads ReKV$^\ddagger$, which keeps the full cache, by 1.9--3.9 points, and InfiniPot-V$^\ddagger$ (16.6--20.5\% read) and SnapKV$^\dagger$ by 2.6--7.2 points.
In other words, \method is the only method below full storage that stays within 1.8 points of FullKV, and it is more accurate than ReKV and InfiniPot-V in all four settings. At this read, it also has the best LVBench accuracy with Qwen2.5-VL (39.70) and ties for the best \mbox{Video-MME} Medium accuracy on both models (62.22 and 60.56).

\paragraph{Token pruning.}
Token pruning and using fewer frames are a different trade-off: these methods read everything they store. At \method's storage, pruning (VisionZip-50\%, InfiniPot-V-50\%) is 0.8--1.0 points more accurate on average, but reads 3.4$\times$ more per step. Simply caching half the frames, FullKV ($N/2$), is 0.89 points below \method.

A useful property of \method is that its attended set stays at 2048 tokens as the video grows. With 128 frames, it stores 104\% of FullKV-64 and scores 2.1 points higher on long \mbox{Video-MME} videos (Qwen2.5-VL). VisionZip-50\% is within 0.11 points there, but reads 6.5$\times$ more (\cref{supp:tab:iso128}).
The two approaches also compose: VisionZip-50\% + \method is within 0.6 points of VisionZip-25\% at 27--28\% storage (\cref{supp:tab:compose}).

\paragraph{Robot-episode recall with a VLA policy.}
In LIBERO Watch-and-Do, the video LLM watches cached robot demonstrations and picks one of four commands, which OpenVLA then executes. \method picks the right command in 54.17\% of episodes with Qwen2.5-VL (tied best; FullKV: 53.33) and 41.25\% with LLaVA-Video (\cref{tab:watchdo,supp:tab:watchdo-full}).
We also ask multiple-choice questions about such episodes (LIBERO-Recall; $n{=}100$ each). On first-task questions, \method matches FullKV (55 vs.\ 54\%, and 61 vs.\ 62\%), while InfiniPot-V ($\approx$2K) drops to 31 and 41\% (chance is 25\%; \cref{supp:tab:liberorecall}). Since the answers here are short, we time decoding on long text generation instead (\cref{sec:efficiency}).
\subsection{Decode latency}
\label{sec:efficiency}

Storing less is only useful if decoding stays fast. We time \method on a 16K-token AIME24 generation in FreeKV's engine~\citep{liu2026freekvboostingkvcache}, storing only the bases of the paged keys and values (\cref{fig:pareto}b,\,c).
\method takes 40 and 47\,ms per token on Qwen-7B and Llama-8B. This is 11\% and 12\% above FreeKV (36 and 42\,ms), and 1.8$\times$ and 3.0$\times$ faster than ArkVale (72 and 143\,ms; \cref{supp:latency}). Building the page bases adds a one-time prefill cost, which shrinks relative to the prompt as the prompt grows (\cref{supp:prefill}).
\subsection{Ablations}
\label{sec:ablations}
\begin{table}[t]
\centering
\begin{minipage}[t]{0.44\textwidth}
\centering
\footnotesize
\caption{Ablations on R1-Distill Qwen-7B (avg@8, \%). Each block changes one
component of \method (shaded): the key reconstruction rank $r$ (pages are still scored at rank 8), the value rank $r_v$, or one
$\eta$ shared by all heads instead of the per-head rule.  }
\label{tab:abl-keyrank}\label{tab:abl-valrank}\label{tab:abl-eta}
\begin{tabular*}{\linewidth}{@{\extracolsep{\fill}}c@{\hspace{4pt}}lccc@{}}
\toprule
 & & MATH500 & AIME24 & GPQA \\
\midrule
\rowcolor{gray!15}
\multicolumn{2}{@{}l}{\method (main)} & 71.75 & 47.92 & 39.50 \\
\midrule
\multirow{3}{*}{\rotatebox[origin=c]{90}{\scriptsize Key}}
 & $r=1$  & 64.25 & 24.17 & 23.75 \\
 & $r=4$  & 69.25 & 37.92 & 33.00 \\
 & $r=16$ & 71.75 & 55.42 & 37.50 \\
\midrule
\multirow{4}{*}{\rotatebox[origin=c]{90}{\scriptsize Value}}
 & $r_v=1$  & 60.75 & 21.25 & 21.75 \\
 & $r_v=4$  & 69.50 & 40.00 & 33.50 \\
 & $r_v=8$  & 70.25 & 46.67 & 33.75 \\
 & $r_v=16$ (main) & 71.75 & 47.92 & 39.50 \\
\midrule
\multirow{5}{*}{\rotatebox[origin=c]{90}{\scriptsize Shared $\eta$}}
 & $\eta=0$ (key PCA) & 71.75 & 43.33 & 38.50 \\
 & $\eta=0.25$ & 71.75 & 47.08 & 38.50 \\
 & $\eta=0.5$  & 70.50 & 42.92 & 36.50 \\
 & $\eta=0.75$ & 70.25 & 45.00 & 33.75 \\
 & $\eta=1$ (output) & 70.50 & 43.75 & 34.25 \\
\bottomrule
\end{tabular*}
\end{minipage}\hfill
\begin{minipage}[t]{0.53\textwidth}
\centering
\footnotesize
\setlength{\tabcolsep}{3pt}
\caption{FreeKV and \method with KIVI-style quantization (avg@8, \%; 8 seeds) on MATH500 (Qwen-7B) and AIME24 (Qwen-14B, \cref{fig:pareto}a). Memory: stored KiB per KV head and layer; FullKV stores 1873 and 5050.}
\label{tab:abl-quant}
\begin{tabular*}{\linewidth}{@{\extracolsep{\fill}}lcccccc@{}}
\toprule
 & \multicolumn{3}{c}{FreeKV} & \multicolumn{3}{c}{\method} \\
\cmidrule(lr){2-4}\cmidrule(lr){5-7}
 & FP16 & 4-bit & 2-bit & FP16 & 4-bit & 2-bit \\
\midrule
\multicolumn{7}{@{}l}{\emph{MATH500, Qwen-7B}} \\
Memory & 1923 & 988 & 818 & 1237 & 776 & 691 \\
avg@8  & 70.00 & 69.25 & 69.25 & 71.75 & 71.50 & 70.75 \\
\midrule
\multicolumn{7}{@{}l}{\emph{AIME24, Qwen-14B}} \\
Memory & 5200 & 2087 & 1521 & 2830 & 1361 & 1090 \\
avg@8  & 64.17 & 63.75 & 63.33 & 65.42 & 65.00 & 64.58 \\
\bottomrule
\end{tabular*}
\end{minipage}
\end{table}

\paragraph{Ranks.}
\Cref{tab:abl-keyrank} explains our choice of rank-8 keys and rank-16 values. Lower ranks lose accuracy on every task (up to 10.00 points for rank-4 key reconstruction and 5.75 for $r_v{=}8$), while rank-16 keys nearly double the key memory.
We stop the value sweep at $r_v{=}16$ for a simple reason. The centered values of a 32-token page span at most 31 directions, so a rank-32 value basis is exact, but it would store even more than the exact values (10496 vs.\ 8192\,B). Rank 16 keeps half of these directions, is nearly lossless, and stores 66\% of the exact values.

\paragraph{Is output-awareness needed?}
Yes. Replacing the per-head rule with key PCA ($\eta{=}0$), with everything else fixed and both arms averaged over the same 8 seeds, lowers avg@8 in eight of the nine settings and ties in the ninth (mean 2.91 points, range 0--8.5; exact sign test $p{=}0.008$; \cref{supp:tab:eta-paired}).
Going to the other extreme does not help either: on R1-Distill Qwen-7B, the output term alone ($\eta{=}1$) is on average no better than key PCA, as \cref{sec:basis} suggests. The per-head rule also matches or exceeds the best \emph{shared} $\eta$, even though that one was selected on the evaluation set (\cref{tab:abl-eta}).

\paragraph{Composing with quantization.}
Quantization lowers the bits per element, while \method lowers the number of elements per page, so the two savings multiply. We apply KIVI-style~\citep{liu2024kivi} quantization (keys per channel, values per token) to \method's bases and coefficients (\cref{tab:abl-quant}).
\method loses 0.25--0.42 points at 4 bits (at 41\% and 27\% of FullKV's memory) and 0.84--1.00 points at 2 bits (37\% and 22\%). This is about as much as FreeKV loses when its exact pages are quantized (0.42--0.84). At each bit width, \method stays 1.25--2.25 points above quantized FreeKV while using 65--84\% of its memory.

\FloatBarrier
\section{Conclusion}
\label{sec:conclusion}
\method lets page retrieval for sparse attention store \emph{less} than the full cache instead of more. Each page is kept only as output-aware low-rank bases, which both index the page and replace its exact entries. The key basis scores the page and rebuilds its keys. Because the rebuilt keys feed the attention output, its objective mixes key geometry with the output error, using a per-head weight that is set in closed form, without a sweep.

At a matched attended budget, \method matches FullKV on MATH500 (Qwen-7B) with 66\% of its memory, is within 0.83 points on AIME24 (Qwen-14B) with 56\%, and is on par with FullKV on LongBench~v2, all while decoding within 12\% of FreeKV's latency. On video, it stores 53\% of the full cache instead of 103\% at the same read, and is only 0.25 points below exact-page retrieval on average. It also composes with quantization and visual-token pruning: with 4-bit bases, it stays within 0.25 points of FullKV at 41\% of its memory on MATH500, and within 1.25 points at 27\% on AIME24.

\paragraph{Limitations.} Our analysis assumes isotropic queries and near-uniform attention within a page. We derive a whitened basis for general queries (\cref{supp:ext-whitened}), but have not yet evaluated it.

\section*{LLM Usage Statement}
We used Anthropic's Claude (through Claude Code) to edit and polish the
writing, and as a coding assistant for experiment drivers, analysis scripts and
figure generation. All methods, experiments, analyses and results were designed,
run and verified by the authors, who take full responsibility for the content
of this paper.

{
  \small
  \bibliographystyle{ieeenat_fullname}
  \bibliography{references}
}

\clearpage
\appendix
\begin{center}
{\Large\bfseries Appendix}
\end{center}
\vspace{0.5em}

\DeclareUrlCommand\fl{\urlstyle{tt}%
  \def\UrlBreaks{\do\/\do\.\do\_\do\-\do\=\do\*}}
\setlist[itemize]{leftmargin=*,itemsep=1pt,topsep=2pt,parsep=0pt}
\newcommand{\idea}[1]{\par\noindent\emph{Idea.}~#1\par\smallskip}

\noindent
The appendix collects everything that did not fit in the main text. We first
explain how the compared numbers differ in source, protocol and accounting
(\cref{supp:read}). Then we follow the paper's order: we
prove why a page basis should see the values, extend the studies of $\eta$
and of accuracy versus memory, document how every number was produced,
describe the two robot tasks (Watch-and-Do and LIBERO-Recall), and give the paired video statistics, \method on
the cache that token pruning keeps and with more frames at matched memory.

\begin{center}
\small
\begin{tabularx}{\linewidth}{@{}l>{\raggedright\arraybackslash}X@{}}
\toprule
Section & Question it answers \\
\midrule
\ref{supp:read}~How to read the comparisons
  & Which methods are ranked, which are reference points or composition axes,
    which numbers come from prior work, which are selected on the evaluation
    set, and how are memory and budgets counted? \\
\ref{supp:proofs}~Why a page basis should see the values
  & Why is the output error first order and the page-mass error second order in expectation,
    which basis minimizes the leading output error, and why does the same basis
    rebuild the keys? \\
\ref{supp:ext}~Extended results
  & How does the per-head rule compare with key PCA ($\eta=0$), where does every arm of the
    accuracy--memory study land, and how well does the per-head rule fit? \\
\ref{supp:impl}~Experimental reference
  & Which models, data, seeds, budgets and switches the experiments use,
    how our runs match the FreeKV setting, how the video experiments and
    baseline ports are set up, how $\eta_h$ is calibrated, and how
    memory, latency and prefill time are measured. \\
\ref{supp:watchdo}~LIBERO Watch-and-Do
  & How is the robot-episode recall task built, and how are commands executed and scored? \\
\ref{supp:liberorecall}~LIBERO-Recall
  & How does each method do on a question-answering variant of the robot task, by
    question type? \\
\ref{supp:video}~Paired statistics, composition with pruning, more frames
  & What are the paired video intervals? How does \method do on the cache that VisionZip keeps, and at what
    storage? What do twice as many frames buy at the memory of FullKV? \\
\bottomrule
\end{tabularx}
\end{center}

\section{How to read the comparisons}
\label{supp:read}

\noindent
\method is a page-retrieval method for sparse attention, like Quest, ArkVale and FreeKV; it differs in
what it stores. The paper's numbers fall into three classes (below) and differ in source, protocol and
accounting. This section lists each such difference
once; the tables and figures that depend on it repeat it locally. We call a
comparison \emph{paired} when both numbers are our own runs on the same
problems or questions with the same seeds.

\noindent\textbf{Comparison class.}
The ranked methods attend the same number of tokens per decoding step (\emph{Budgets} below): the
retrieval methods Quest, ArkVale, ShadowKV, InfiniGen and FreeKV and the eviction methods RazorAttention
and RaaS in long reasoning and long context, and Quest, FreeKV, ReKV and InfiniPot-V on video. At about equal
attended tokens they differ in accuracy and stored memory.

\noindent\textbf{Reported versus run.}
\begin{itemize}
\item \emph{Long reasoning and long context}
  (\cref{tab:reasoning,tab:longbench}).
  FullKV and the seven baselines are the values reported by
  FreeKV~\citep{liu2026freekvboostingkvcache}. \method is run in FreeKV's
  accuracy harness in its exact setting: the same problem files, prompts,
  decoding settings, budgets and seeds (\cref{supp:same-setting}), so every
  number in these tables is on the same problems and seeds. The long-context
  results are single runs (seed 42).
\item \emph{Accuracy--memory studies.} On AIME24 with Qwen-14B
  (\cref{fig:pareto}a), FullKV and the baselines at 16 bits are the values
  FreeKV reports and \method at 16 bits is the run of \cref{tab:reasoning};
  the 4- and 2-bit arms are our runs (\cref{tab:abl-quant}). On MATH500
  with Qwen-7B (\cref{supp:fig:pareto-math500,supp:tab:pareto}), every arm,
  \method included, is our paired run (\cref{supp:exp-pareto}).
  For the six arms that also appear in \cref{tab:reasoning} (FullKV,
  FreeKV, Quest, ArkVale, ShadowKV and RaaS), our reruns recover the
  accuracies FreeKV reports.
\item \emph{Video and robot tasks} (\cref{tab:videomme}, \cref{supp:tab:watchdo-full,supp:tab:lrmain}).
  Every number is our run for the two video LLMs, and every
  comparison is paired on the same questions. VisionZip, InfiniPot-V and ReKV are our
  ports of each method's core step, not the official code
  (\cref{supp:video-ports}).
\end{itemize}

\noindent\textbf{Reference points.}
H2O, TOVA and SnapKV are implemented without eviction: they keep the full
cache and stop attending to the pages they would drop. Their accuracy is
therefore an upper bound on that of the evicting method, while their memory is
counted as if they evicted. They are marked $\dagger$, are not ranked in
the tables, and are excluded from the Pareto front. The two Oracle arms need the exact keys of
every page and are ceilings, not methods. FullKV, also with fewer frames, is a reference and is not ranked.

\noindent\textbf{Composition axes.}
Quantization lowers the bits per stored element and token pruning decides which tokens enter the cache;
both are applied to \method, not ranked against it. KIVI-style quantization is applied to both FreeKV and
\method (\cref{tab:abl-quant}): at 16 bits, FreeKV is our rerun on MATH500, which recovers its reported value, and its reported value on AIME24 (\cref{tab:reasoning}),
\method at 16 bits is the run of \cref{tab:reasoning}, and the 4- and 2-bit arms are our paired runs
(\cref{supp:exp-pareto}). VisionZip at 50\% and 25\% is reported at similar storage in
\cref{tab:videomme}, and \method runs on the VisionZip-50\% cache in \cref{supp:tab:compose}.

\noindent\textbf{Selection on the evaluation set.}
Every \method configuration, the retrieval-only arm of \cref{supp:fig:pareto-math500}
included, uses the per-head $\eta_h$ from the rule (except the shared-$\eta$ ablations), calibrated on held-out
data (\cref{supp:calib}); $\eta$ is never tuned on accuracy labels. The
shared-$\eta$ ablation (\cref{tab:abl-eta}) is a sweep on the evaluation
set, so its best value is a best-of-five reference that favors the shared
$\eta$. The comparison with key PCA, the fixed endpoint $\eta=0$
(\cref{supp:tab:eta-paired}), involves no selection.

\noindent\textbf{Stored memory versus GPU memory.}
Every memory number is the size of the stored KV state under the analytic
memory model of \cref{supp:exp-pareto}. It is not measured GPU memory. The
long-reasoning studies count one KV head in one layer on a fixed reference
sequence (3745 tokens for MATH500, 10100 for AIME24), in KiB. The video and robot tables apply the
same per-page accounting to every layer and KV head of the model, at each
question's own final length, in MiB. The two scales differ and should be
compared only through their ratios to FullKV. The model ignores where the state resides (GPU or
host), e.g. the host page pools of the latency engine (\cref{supp:latency}) or ReKV, which keeps
14--15\% of its cache on the GPU while decoding.

\noindent\textbf{Budgets.}
Every budget is the total number of tokens attended per decoding step,
including the exact sink and recent tokens; a long-reasoning step attends
2048 tokens, as in FreeKV, in both the accuracy runs and the latency engine
(\cref{supp:tab:protocols,supp:latency}).
The video text separates two resources:
\emph{Stored} is the stored KV cache as a share of FullKV with $N{=}64$ frames, and \emph{Read} is the number of
tokens attended per decoding step as a share of the visual tokens. Quest, FreeKV, SnapKV$^\dagger$ and \method read
2048 tokens in layers 1--27 (on \mbox{Video-MME} and LVBench 15.6\% of the visual tokens for Qwen2.5-VL and 15.2\% for LLaVA-Video, 14.5\% for Qwen2.5-VL on LIBERO; layer 0 is
dense), ReKV reads its retrieved blocks and InfiniPot-V at about 2K tokens its compressed cache, each plus the question, and token pruning
and FullKV with fewer frames read every token they store (\cref{supp:protocols}). Read ratios are per layer in
layers 1--27. \cref{tab:videomme} reports Stored and Read directly; in
\cref{supp:tab:lrmain,supp:tab:liberorecall}, the Budget column of a KV-cache method is its Read, and that of
FullKV (all or fewer frames) and of the token reducers is the share of visual tokens kept, which they also read. Bracketed intervals in the video text are paired 95\% intervals from a bootstrap over videos
(10{,}000 resamples, seed 0); for a mean over the four video settings, each resample is averaged over them.

\noindent\textbf{Ranks.}
The deployed ranks (rank-8 key bases, rank-16 value bases) are chosen
empirically (\cref{tab:abl-keyrank,tab:abl-valrank}, \cref{supp:recon}).

\section{Why a page basis should see the values: proofs}
\label{supp:proofs}

\noindent
This section proves the statements of \cref{sec:method} in four steps. A
logit error inside a page moves the attention output at first order, in
proportion to how much the page's values differ (\cref{supp:sensitivity}).
The basis that minimizes this leading output error solves one eigenproblem,
and a basis built from the keys alone can discard exactly the direction the
output depends on (\cref{supp:basis}). The same error changes the page's
attention mass, which retrieval ranks, only at second order in expectation over isotropic
queries (first order per query), so key energy
and output sensitivity are mixed into a single eigenproblem
(\cref{supp:hybrid}). The basis that scores a page also rebuilds its keys,
with the same error term (\cref{supp:recon}).

\subsection{Setting, notation and two tools}
\label{supp:prelim}

We work with one KV head and one completed page $P$ of $m\ge 1$ tokens, and
drop the subscript $P$. Keys $K\in\R^{m\times d}$ and values
$V\in\R^{m\times d_v}$ have rows $k_i^\top$ and $v_i^\top$ (in the models used
here $d=d_v=128$); keys and queries include RoPE. With $\onev\in\R^m$ the
all-ones vector,
\begin{equation*}
\mu=\tfrac1mK^\top\onev,\quad X=K-\onev\mu^\top,\qquad
\nu=\tfrac1mV^\top\onev,\quad Y=V-\onev\nu^\top ,
\end{equation*}
so $\onev^\top X=0$ and $\onev^\top Y=0$. For a query $q\in\R^d$ let
$z(q)=Xq/\sqrt d$ and $\operatorname{lse}(z)=\log\sum_i e^{z_i}$. The page's
log attention mass is
\begin{equation*}
\ell(q)=\log\sum_{i=1}^m e^{k_i^\top q/\sqrt d}
=\frac{\mu^\top q}{\sqrt d}+\operatorname{lse}(z(q)).
\end{equation*}
For a fixed query the global softmax denominator is shared by all pages, so
ranking pages by $\ell(q)$ ranks their attention mass. Let
$\St(d,r)=\{B\in\R^{d\times r}:B^\top B=I_r\}$ and $P_B=BB^\top$. The key
basis keeps $\mu$ exactly and approximates the centered keys by $XP_B$:
\begin{equation}
\widehat z_B(q)=\frac{XP_Bq}{\sqrt d},\qquad
e_B(q)=z(q)-\widehat z_B(q)=\frac{X(I_d-P_B)q}{\sqrt d},\qquad
\widehat\ell_B(q)=\frac{\mu^\top q}{\sqrt d}+\operatorname{lse}(\widehat z_B(q)).
\label{supp:eq:approx-logits}
\end{equation}
The key basis is stored as $(\mu,B,C)$ with $C=XB\in\R^{m\times r}$, so
$XP_Bq=C(B^\top q)$. Within the page, the attention weights and the output are
$\beta(z)=\softmax(z)$ and $f(z)=V^\top\beta(z)=\nu+Y^\top\beta(z)$.
Output errors are measured in the seminorm $\norm{a}_G=(a^\top Ga)^{1/2}$,
$G\succeq0$; \method uses
$G=G_h=\sum_{i\in\mathcal G(h)}W_{O,i}^\top W_{O,i}$, the output projections
of the query heads that share KV head $h$. The two risks, the output
sensitivity matrix and the mixed matrix are
\begin{align}
\cR_{\mathrm{key}}(B)&=\E_q\norm{e_B(q)}_2^2,\qquad
\cR_{\mathrm{out}}(B)=\E_q\Bigl\|\tfrac1mY^\top e_B(q)\Bigr\|_G^2,
\qquad M_{\mathrm{out}}=X^\top YGY^\top X, \label{supp:eq:risks}\\
M_\eta&=(1-\eta)\,c_{\mathrm{key}}X^\top X+\eta\, c_{\mathrm{out}}M_{\mathrm{out}},
\qquad
c_{\mathrm{key}}=\bigl(\trace(X^\top X)+\varepsilon\bigr)^{-1},\quad
c_{\mathrm{out}}=\bigl(\trace(M_{\mathrm{out}})+\varepsilon\bigr)^{-1},
\label{supp:eq:Meta}
\end{align}
with $\eta\in[0,1]$ and a stabilizer $\varepsilon>0$. The hybrid risk is
$\cR_\eta(B)=(1-\eta)c_{\mathrm{key}}\cR_{\mathrm{key}}(B)
+\eta m^2c_{\mathrm{out}}\cR_{\mathrm{out}}(B)$.

\begin{assumption}[Query second moment]
\label{supp:ass:query}
The queries satisfy $\E_q[qq^\top]=I_d$.
\end{assumption}

\noindent
Under \cref{supp:ass:query},
\begin{equation}
\cR_{\mathrm{key}}(B)=\tfrac1d\trace\bigl((I_d-P_B)X^\top X(I_d-P_B)\bigr),\qquad
\cR_{\mathrm{out}}(B)=\tfrac{1}{m^2d}\trace\bigl((I_d-P_B)M_{\mathrm{out}}(I_d-P_B)\bigr).
\label{supp:eq:risk-traces}
\end{equation}
The first identity is proved in \cref{supp:hybrid}, the second in
\cref{supp:basis}; \cref{supp:ext-whitened} removes the assumption by
whitening. We also write $\mathcal S(b)=\Diag(b)-bb^\top$ for a probability
vector $b\in\R^m$, $H=I_m-\tfrac1m\onev\onev^\top$ (so $X=HK$, $Y=HV$), and
$Df(z)[h]=\frac{d}{dt}f(z+th)\big|_{t=0}$. Two standard tools are used
throughout: the derivatives of the softmax, and the Ky Fan principle~\citep{fan1949theorem}, which
turns every risk below into an eigenproblem.

\begin{lemma}[Softmax calculus]
\label{supp:lem:softmax}
For every $z,h\in\R^m$: (i)~$D\beta(z)[h]=\mathcal S(\beta(z))h$;
(ii)~$\mathcal S(\beta(z))\onev=0$ and $\mathcal S(\beta(z))\succeq0$;
(iii)~$D\operatorname{lse}(z)=\beta(z)^\top$ and
$D^2\operatorname{lse}(z)=\mathcal S(\beta(z))$. In particular
$D\operatorname{lse}(0)=\onev^\top/m$ and $D^2\operatorname{lse}(0)=H/m$.
\end{lemma}

\begin{proof}
For $\beta_i(z)=e^{z_i}/\sum_je^{z_j}$,
$\partial\beta_i/\partial z_j=\beta_i(\delta_{ij}-\beta_j)$, which is (i).
Next, $\mathcal S(\beta)\onev=\beta-\beta(\onev^\top\beta)=0$, and for every
$x\in\R^m$,
$x^\top \mathcal S(\beta)x=\sum_i\beta_ix_i^2-(\sum_i\beta_ix_i)^2
=\operatorname{Var}_{i\sim\beta}(x_i)\ge0$. Finally
$\nabla\operatorname{lse}(z)=\beta(z)$, so
$D^2\operatorname{lse}(z)=\mathcal S(\beta(z))$; at $z=0$, $\beta(0)=\onev/m$ and
$\mathcal S(\onev/m)=H/m$.
\end{proof}

\begin{lemma}[Ky Fan: optimal subspace for a quadratic risk]
\label{supp:lem:kyfan}\label{supp:kyfan}
Let $M\in\R^{d\times d}$ be symmetric positive semidefinite. The minimizers
of $\trace((I_d-P_B)M(I_d-P_B))$ over $B\in\St(d,r)$ are exactly the
orthonormal bases of leading $r$-dimensional invariant subspaces of $M$, and
the minimum is the sum of the $d-r$ smallest eigenvalues.
\end{lemma}

\begin{proof}
Let $\lambda_1\ge\cdots\ge\lambda_d\ge0$ be the eigenvalues of $M$ with
orthonormal eigenvectors $u_j$. Since $I_d-P_B$ is an orthogonal projector,
$\trace((I_d-P_B)M(I_d-P_B))=\trace(M)-\trace(B^\top MB)$, so it suffices
to maximize $\trace(B^\top MB)$. With $c_j=\norm{B^\top u_j}_2^2$ we have
$0\le c_j\le1$ and $\sum_jc_j=r$, so
$\trace(B^\top MB)=\sum_j\lambda_jc_j\le\sum_{j\le r}\lambda_j$, with
equality for $B=[u_1,\ldots,u_r]$. Equality holds exactly when $c_j=1$ for
$\lambda_j>\lambda_r$ and $c_j=0$ for $\lambda_j<\lambda_r$, \ie when the
retained subspace contains every eigendirection above $\lambda_r$ and its
remaining dimensions lie in the $\lambda_r$ eigenspace. These are the
leading invariant subspaces; the subspace is unique when
$\lambda_r>\lambda_{r+1}$. If $M=Z^\top Z$, its positive eigenvectors are
the right singular vectors of $Z$.
\end{proof}

\subsection{A logit error moves the output at first order, weighted by the values}
\label{supp:sensitivity}

\idea{Perturbing the logits only reshuffles attention among the page's
tokens. If those tokens carry the same value, nothing changes; the output
moves in proportion to the value differences. At uniform attention, a
perturbation that keeps the page's total logit fixed leaves the page's
share of the softmax unchanged but still moves the output.}

\begin{lemma}[Conditional output sensitivity]
\label{supp:lem:jacobian}
For any $z,h\in\R^m$,
\begin{equation}
Df(z)[h]=V^\top \mathcal S(\beta(z))h=\sum_{i=1}^m\beta_i(z)\bigl(v_i-f(z)\bigr)h_i,
\label{supp:eq:jacobian}
\end{equation}
and $V^\top \mathcal S(\beta)=Y^\top \mathcal S(\beta)$.
\end{lemma}

\begin{proof}
By \cref{supp:lem:softmax}, $Df(z)[h]=V^\top \mathcal S(\beta(z))h$. Writing
$\beta=\beta(z)$,
$V^\top \mathcal S(\beta)h=V^\top\Diag(\beta)h-V^\top\beta\,\beta^\top h
=\sum_i\beta_iv_ih_i-f(z)\sum_i\beta_ih_i$, which is the second form.
Since $\mathcal S(\beta)$ is symmetric and $\mathcal S(\beta)\onev=0$,
$\onev^\top \mathcal S(\beta)=0$; with $V=\onev\nu^\top+Y$ this gives
$V^\top \mathcal S(\beta)=Y^\top \mathcal S(\beta)$.
\end{proof}

\begin{proposition}[Exact path identity]
\label{supp:prop:path}
Let $z,\widehat z\in\R^m$, $e=z-\widehat z$, $z_t=\widehat z+te$ and
$\beta_t=\softmax(z_t)$. Then
\begin{equation*}
f(z)-f(\widehat z)=\Bigl(\int_0^1V^\top \mathcal S(\beta_t)\,dt\Bigr)e,
\qquad
\norm{f(z)-f(\widehat z)}_G^2\le\int_0^1e^\top \mathcal S(\beta_t)VGV^\top \mathcal S(\beta_t)e\,dt.
\end{equation*}
\end{proposition}

\begin{proof}
With $g(t)=f(z_t)$, \cref{supp:lem:jacobian} gives
$g'(t)=V^\top \mathcal S(\beta_t)e$, and the fundamental theorem of calculus
gives the identity. For the bound, let $a_t=G^{1/2}V^\top \mathcal S(\beta_t)e$;
Jensen's inequality gives
$\norm{\int_0^1a_t\,dt}_2^2\le\int_0^1\norm{a_t}_2^2\,dt$.
\end{proof}

\begin{proposition}[Value diameter bound]
\label{supp:prop:diameter}
With $\operatorname{diam}_G(V)=\max_{i,j}\norm{v_i-v_j}_G$, for all
$z,\widehat z\in\R^m$,
$\norm{f(z)-f(\widehat z)}_G\le\operatorname{diam}_G(V)\,\norm{z-\widehat z}_\infty$.
\end{proposition}

\begin{proof}
For a probability vector $b$, $V^\top b$ lies in the convex hull of the
$v_i$, so $\norm{v_i-V^\top b}_G\le\operatorname{diam}_G(V)$. By
\cref{supp:lem:jacobian},
$\norm{Df(z)[h]}_G\le\sum_i\beta_i(z)\norm{v_i-f(z)}_G|h_i|
\le\operatorname{diam}_G(V)\norm{h}_\infty$. Integrate along the segment
from $\widehat z$ to $z$.
\end{proof}

\noindent
The next result places the page inside the full cache: the rest of the
cache contributes a softmax denominator and a value numerator, which the
page's basis does not see.

\begin{proposition}[Global sensitivity of centered page perturbations]
\label{supp:prop:global}
Fix the query, the page values and the rest of the cache, which contributes
a softmax denominator $Z_0>0$ and a value numerator $n_0\in\R^{d_v}$. Hold
the mean key factor $c=\exp(\mu^\top q/\sqrt d)>0$ fixed, and let
\begin{equation*}
o_{\mathrm{glob}}(z)=\frac{n_0+cV^\top\exp(z)}{Z_0+c\onev^\top\exp(z)},\qquad
p(z)=\frac{c\onev^\top\exp(z)}{Z_0+c\onev^\top\exp(z)}
\end{equation*}
be the full attention output and the page's attention probability. Then
\begin{equation*}
Do_{\mathrm{glob}}(z)[h]=p(z)Df(z)[h]+p(z)\bigl(f(z)-o_{\mathrm{glob}}(z)\bigr)\beta(z)^\top h,
\qquad Dp(z)[h]=p(z)(1-p(z))\beta(z)^\top h.
\end{equation*}
With $T_0=Z_0+cm$ and $p_0=p(0)=cm/T_0$, every centered perturbation
($\onev^\top h=0$) satisfies $Dp(0)[h]=0$ and
\begin{equation}
Do_{\mathrm{glob}}(0)[h]=p_0Df(0)[h]=\frac{p_0}{m}Y^\top h=\frac{c}{T_0}Y^\top h.
\label{supp:eq:global-centered}
\end{equation}
\end{proposition}

\begin{proof}
Write $o_{\mathrm{glob}}=N/T$ with $N(z)=n_0+cV^\top\exp(z)$ and
$T(z)=Z_0+c\onev^\top\exp(z)$, so
$DN(z)[h]=cV^\top\Diag(\exp(z))h$ and $DT(z)[h]=c\exp(z)^\top h$. Since
$c\exp(z)/T(z)=p(z)\beta(z)$, the quotient rule gives
$Do_{\mathrm{glob}}(z)[h]=p(z)\bigl(V^\top\Diag(\beta(z))h
-o_{\mathrm{glob}}(z)\beta(z)^\top h\bigr)$, which equals the stated form by
\cref{supp:lem:jacobian} and $f(z)=V^\top\beta(z)$. Applying the quotient
rule to $p=(T-Z_0)/T$ gives
$Dp(z)[h]=\frac{Z_0}{T(z)}\frac{DT(z)[h]}{T(z)}=(1-p(z))p(z)\beta(z)^\top h$.
At $z=0$, $\beta(0)=\onev/m$, so a centered $h$ has $\beta(0)^\top h=0$, and
both the mass derivative and the extra output term vanish. Finally
$\mathcal S(\beta(0))=H/m$ and $V^\top H=Y^\top$ give $Df(0)[h]=Y^\top h/m$.
\end{proof}

\noindent
The basis's logit error $e_B(q)$ is centered ($\onev^\top X=0$), so the proposition applies directly: at uniform within-page attention, to first order it leaves the page's mass
unchanged and moves the output in proportion to the value spread $Y$.

\subsection{The output-aware basis minimizes the leading output error}
\label{supp:basis}

\idea{Linearizing at uniform within-page attention turns the output error
into a linear map of the discarded query directions, $Y^\top X(I_d-P_B)q$.
Its expected size is a quadratic form in $I_d-P_B$, so Ky Fan gives the best
basis in closed form. The matrix that matters couples keys and values token
by token; neither the key covariance nor the value covariance alone
determines it.}

\begin{proposition}[First-order conditional output error]
\label{supp:prop:first-order}\label{supp:first-order}
For any $q\in\R^d$ and $B\in\St(d,r)$,
\begin{equation}
f(z(q))-f(\widehat z_B(q))=\frac{1}{m\sqrt d}Y^\top X(I_d-P_B)q+r_B(q).
\label{supp:eq:leading}
\end{equation}
Let $\rho=\max\{\norm{z(q)}_2,\norm{\widehat z_B(q)}_2\}$ and
\begin{equation*}
L_{V,G}(\rho)=\sup\bigl\{\norm{D^2f(u)[a,b]}_G:\ \norm{u}_2\le\rho,\
\norm{a}_2=\norm{b}_2=1\bigr\}.
\end{equation*}
This quantity is finite, and
$\norm{r_B(q)}_G\le L_{V,G}(\rho)\,\rho\,\norm{e_B(q)}_2$.
\end{proposition}

\begin{proof}
Write $z=z(q)$, $\widehat z=\widehat z_B(q)$, $e=e_B(q)$. $L_{V,G}(\rho)$ is
finite because $D^2f$ is continuous on a compact set. At uniform attention,
$Df(0)[e]=\frac1mV^\top He=\frac1mY^\top e
=\frac{1}{m\sqrt d}Y^\top X(I_d-P_B)q$. With $z_t=\widehat z+te$, the path
identity gives $f(z)-f(\widehat z)=Df(0)[e]+r_B(q)$ with
$r_B(q)=\int_0^1(Df(z_t)-Df(0))[e]\,dt$, which establishes
\eqref{supp:eq:leading} whether or not $G$ is invertible. Since
$\norm{z_t}_2\le\rho$,
$(Df(z_t)-Df(0))[e]=\int_0^1D^2f(sz_t)[z_t,e]\,ds$; taking the
$G$-seminorm ($\norm{a}_G=\norm{G^{1/2}a}_2$ obeys the triangle inequality)
gives $\norm{r_B(q)}_G\le L_{V,G}(\rho)\rho\norm{e}_2$. Taylor's theorem at
zero also gives
$\norm{r_B(q)}_G\le\frac12L_{V,G}(\rho)\bigl(\norm{z}_2^2+\norm{\widehat z}_2^2\bigr)$.
\end{proof}

\paragraph{The output surrogate.}
Under \cref{supp:ass:query}, the expected squared leading term is
\begin{equation*}
\E_q\Bigl\|\frac{Y^\top X(I_d-P_B)q}{m\sqrt d}\Bigr\|_G^2
=\frac{1}{m^2d}\E_q\bigl[q^\top(I_d-P_B)M_{\mathrm{out}}(I_d-P_B)q\bigr]
=\frac{1}{m^2d}\trace\bigl((I_d-P_B)M_{\mathrm{out}}(I_d-P_B)\bigr),
\end{equation*}
which proves the second identity of \eqref{supp:eq:risk-traces}. Since
$M_{\mathrm{out}}=(G^{1/2}Y^\top X)^\top(G^{1/2}Y^\top X)\succeq0$,
\cref{supp:lem:kyfan} gives its minimizers: the leading right singular
subspaces of $G^{1/2}Y^\top X$. If $\operatorname{rank}(M_{\mathrm{out}})<r$,
all positive eigendirections are kept and the rest of the basis is free in
the null space.

The surrogate is the leading coefficient of the true expected squared
conditional error in the small-logit limit. Scale queries by $t>0$ with
$X,V,G,B$ fixed. Differentiability at zero gives
$\bigl(f(z(tq))-f(\widehat z_B(tq))\bigr)/t\to\frac1mY^\top e_B(q)$, and
\cref{supp:prop:diameter} with $e_B(tq)=te_B(q)$ bounds the squared
$G$-norm of the ratio by $\operatorname{diam}_G(V)^2\norm{e_B(q)}_2^2$,
which is integrable. Dominated convergence gives
\begin{equation*}
\E_q\norm{f(z(tq))-f(\widehat z_B(tq))}_G^2=t^2\cR_{\mathrm{out}}(B)+o(t^2)
\qquad(t\to0).
\end{equation*}
No higher query moment is needed, and the statement does not claim
optimality for finite-logit errors. For the full-cache output,
\cref{supp:prop:global} adds a per-query weight $p_0(q)^2$, which depends on
the rest of the cache; we omit it, so that the basis is a function of
$(K,V)$ alone.

\begin{proposition}[Value covariance does not determine output sensitivity]
\label{supp:prop:cov}\label{supp:covariance}
For every token permutation $\Pi$ and $Y'=\Pi Y$, $(Y')^\top Y'=Y^\top Y$,
but there are $X,Y,\Pi,G$ with
$X^\top Y'G(Y')^\top X\ne X^\top YGY^\top X$.
\end{proposition}

\begin{proof}
$\Pi^\top\Pi=I$ gives the first claim. Take $m=3$, $d=d_v=1$, $G=1$,
$X=Y=(1,0,-1)^\top$, and let $\Pi$ swap the first two tokens, so
$Y'=(0,1,-1)^\top$. Then $(X^\top Y)^2=4$ while $(X^\top Y')^2=1$, although
$X^\top X$ and $Y^\top Y$ are identical.
\end{proof}

\noindent
What matters is the token-wise alignment of key and value deviations,
$Y^\top X$. The next example shows that key PCA can discard exactly the
direction the output depends on.

\begin{proposition}[Key PCA vs.\ the output-aware basis]
\label{supp:prop:counter}
Let $m\ge3$, $d=2$, $d_v=1$, $G=1$, let $a,b\in\onev^\perp$ be orthonormal,
and set $X=[\sigma_1a\ \ \sigma_2b]$, $Y=b$ with $\sigma_1>\sigma_2>0$. At
rank $r=1$, key PCA uniquely selects the first coordinate direction and the
output-aware basis uniquely selects the second. The latter has zero leading
output risk; key PCA has $\sigma_2^2/(m^2d)$.
\end{proposition}

\begin{proof}
$X^\top X=\Diag(\sigma_1^2,\sigma_2^2)$ and $Y^\top X=(0,\ \sigma_2)$, so
$M_{\mathrm{out}}=\Diag(0,\sigma_2^2)$. Key PCA keeps the first direction
since $\sigma_1^2>\sigma_2^2$; \cref{supp:lem:kyfan} keeps the second for
$M_{\mathrm{out}}$. By \eqref{supp:eq:risk-traces}, the second direction
discards only the zero eigenvalue, and the first discards $\sigma_2^2$.
\end{proof}

\subsection{Page mass is second order in expectation; mixing the two risks}
\label{supp:hybrid}

\idea{Because the logit error is centered, the log-sum-exp has zero slope in
its direction at uniform attention, so in expectation over isotropic queries
the page mass moves only at second order (per query, a first-order cross term remains). The key risk still matters away from that point and for the logits
of rebuilt keys, so the basis mixes the two risks. Each is normalized by
its total energy, which makes $\eta$ a trade between fractions of energy and
keeps a single eigenproblem.}

\begin{lemma}[LSE expansion with explicit remainder]
\label{supp:lem:lse}
For every $h\in\onev^\perp$,
$\bigl|\operatorname{lse}(h)-\log m-\tfrac{1}{2m}\norm{h}_2^2\bigr|
\le\tfrac13\norm{h}_\infty\norm{h}_2^2$.
\end{lemma}

\begin{proof}
Let $g(t)=\operatorname{lse}(th)$ and $\beta^{(t)}=\softmax(th)$. By
\cref{supp:lem:softmax}, $g'(t)=\beta^{(t)\top}h$ and
$g''(t)=h^\top \mathcal S(\beta^{(t)})h$; with $m_t=\beta^{(t)\top}h$,
$g'''(t)=\sum_i\beta_i^{(t)}(h_i-m_t)^3$, so
$|g'''(t)|\le\max_i|h_i-m_t|\sum_i\beta^{(t)}_i(h_i-m_t)^2
\le2\norm{h}_\infty\norm{h}_2^2$. Since $h\in\onev^\perp$, $g(0)=\log m$,
$g'(0)=0$ and $g''(0)=\norm{h}_2^2/m$; Taylor's theorem with remainder
$g'''(\tau)/6$ concludes.
\end{proof}

\begin{proposition}[Page mass is second order in expectation]
\label{supp:mass}
For every $q$ and $B\in\St(d,r)$,
\begin{equation}
\ell(q)-\widehat\ell_B(q)=\frac{\norm{e_B}_2^2+2\,\widehat z_B^\top e_B}{2m}+\Delta_B(q),
\qquad
|\Delta_B(q)|\le\tfrac13\bigl(\norm{z}_\infty\norm{z}_2^2
+\norm{\widehat z_B}_\infty\norm{\widehat z_B}_2^2\bigr),
\label{supp:eq:mass-expansion}
\end{equation}
and globally $|\ell(q)-\widehat\ell_B(q)|\le\norm{e_B(q)}_\infty$. Under
\cref{supp:ass:query},
\begin{equation}
\E_q[\widehat z_B^\top e_B]=0,\qquad
\cR_{\mathrm{key}}(B)=\tfrac1d\trace\bigl((I_d-P_B)X^\top X(I_d-P_B)\bigr),\qquad
\E_q\bigl[(\ell(q)-\widehat\ell_B(q))^2\bigr]\le\cR_{\mathrm{key}}(B),
\label{supp:eq:mass-bound}
\end{equation}
and if moreover $\E_q\norm{q}_2^3<\infty$, scaling the queries by $t\ge0$
gives $\E_q[\ell(tq)-\widehat\ell_B(tq)]=\frac{t^2}{2m}\cR_{\mathrm{key}}(B)+O(t^3)$.
\end{proposition}

\begin{proof}
Both $z(q)$ and $\widehat z_B(q)$ lie in $\onev^\perp$. Applying
\cref{supp:lem:lse} to each and cancelling the exact mean term gives
\eqref{supp:eq:mass-expansion}, with
$\norm{z}_2^2-\norm{\widehat z_B}_2^2=\norm{e_B}_2^2+2\widehat z_B^\top e_B$.
Although $P_Bq\perp(I_d-P_B)q$, multiplication by $X$ need not preserve
orthogonality; in expectation, however,
$\E_q[\widehat z_B^\top e_B]=\frac1d\trace(P_BX^\top X(I_d-P_B))
=\frac1d\trace((I_d-P_B)P_BX^\top X)=0$. The same assumption gives
$\E_q\norm{e_B}_2^2=\frac1d\trace((I_d-P_B)X^\top X(I_d-P_B))$ with no
small-logit assumption. For the global bound, let $a=\norm{e_B(q)}_\infty$;
componentwise $\widehat z_B-a\onev\le z\le\widehat z_B+a\onev$, and
monotonicity of $\operatorname{lse}$ with
$\operatorname{lse}(u+c\onev)=\operatorname{lse}(u)+c$ gives
$|\ell-\widehat\ell_B|\le a\le\norm{e_B}_2$; squaring and taking
expectations gives the last inequality of \eqref{supp:eq:mass-bound}.
Under scaling, the first term of \eqref{supp:eq:mass-expansion} has
expectation $\frac{t^2}{2m}\cR_{\mathrm{key}}(B)$, and $\Delta_B$ is
$O(t^3\norm{q}_2^3)$.
\end{proof}

\noindent
The page-mass error is thus \emph{second} order in the logit residual in expectation over isotropic queries (first order per query),
whereas the output error of \cref{supp:prop:first-order} is \emph{first}
order. By \cref{supp:lem:kyfan} with $M=X^\top X$, the minimizers of
$\cR_{\mathrm{key}}$ are the leading eigenspaces of $X^\top X$ (page-local
key PCA~\citep{hwang2026lockspagelocalcompactkey}).

\paragraph{Trace scaling.}
For $M\in\{X^\top X,M_{\mathrm{out}}\}$ and $J=I_d-P_B$,
$0\le\trace(JMJ)/(\trace(M)+\varepsilon)\le\trace(M)/(\trace(M)+\varepsilon)<1$.
Without the stabilizer this is the fraction of the component's energy that
the basis discards; with $\varepsilon>0$ a zero component contributes zero
and all denominators stay positive. The hybrid risk is a weighted sum of
these two fractions, so $\eta$ trades fractions of key energy for fractions
of output sensitivity, independently of their scales.

\begin{corollary}[Output-aware page basis]
\label{supp:cor:hybrid}
Under \cref{supp:ass:query}, $\cR_\eta(B)=\frac1d\trace\bigl((I_d-P_B)M_\eta(I_d-P_B)\bigr)$,
and its minimizers over $\St(d,r)$ are exactly the orthonormal bases of
leading $r$-dimensional invariant subspaces of $M_\eta$, with minimum
$\frac1d\sum_{j>r}\lambda_j(M_\eta)$.
\end{corollary}

\begin{proof}
With $J=I_d-P_B$, \eqref{supp:eq:risk-traces} gives
$\cR_\eta(B)=\frac{1-\eta}{d}c_{\mathrm{key}}\trace(JX^\top XJ)
+\frac{\eta}{d}c_{\mathrm{out}}\trace(JM_{\mathrm{out}}J)=\frac1d\trace(JM_\eta J)$;
the factors $m^2$ and $1/m^2$ cancel. Both components are positive
semidefinite with nonnegative weights, so $M_\eta\succeq0$, and
\cref{supp:lem:kyfan} applies. Equivalently, with
$A_\eta=(1-\eta)c_{\mathrm{key}}I_m+\eta c_{\mathrm{out}}YGY^\top\succeq0$,
$M_\eta=X^\top A_\eta X=(A_\eta^{1/2}X)^\top(A_\eta^{1/2}X)$, so the minimizers
are the leading right singular subspaces of $A_\eta^{1/2}X$.
\end{proof}

\paragraph{Computation in the token domain.}\label{supp:small-gram}
For $m<d$ the eigenproblem reduces to $m\times m$. With $\Gamma=XX^\top$ and
$\Omega=YGY^\top$, cyclicity gives $\trace(X^\top X)=\trace(\Gamma)$ and
$\trace(M_{\mathrm{out}})=\trace(\Omega\Gamma)$. Let $S=A_\eta^{1/2}X$, so
$M_\eta=S^\top S$, let $s=\min\{r,\operatorname{rank}(S)\}$, and let
$U_s$ and $\Lambda_s$ hold the $s$ leading positive eigenpairs of $SS^\top$.
Then $B_+=S^\top U_s\Lambda_s^{-1/2}$ satisfies
$B_+^\top B_+=\Lambda_s^{-1/2}U_s^\top SS^\top U_s\Lambda_s^{-1/2}=I_s$ and
$M_\eta B_+=S^\top(SS^\top U_s)\Lambda_s^{-1/2}=B_+\Lambda_s$; since the
positive spectra of $S^\top S$ and $SS^\top$ coincide, $B_+$ holds leading
eigenvectors of $M_\eta$. If $s<r$, $B_+$ spans $\operatorname{range}(M_\eta)$,
so any orthonormal completion $N$ ($N^\top N=I_{r-s}$, $B_+^\top N=0$) lies in
$\ker M_\eta$ and keeps the hybrid risk minimal. Among completions, the
retained key energy is
$\norm{XB_+}_F^2+\norm{XN}_F^2=\norm{XB_+}_F^2+\norm{X_\perp N}_F^2$ with
$X_\perp=X-XB_+B_+^\top$, which is maximized by the leading right singular
directions of $X_\perp$ (all orthogonal to $B_+$, since $X_\perp B_+=0$).
This construction therefore returns a basis that is optimal for
$\cR_\eta$ and, among those, best for $\cR_{\mathrm{key}}$. Page scoring
evaluates $XP_Bq$ as $C(B^\top q)$ at $O(d+dr+mr+m)$ cost per page; values
enter only the basis construction, never scoring.

\emph{Implementation.} $A_\eta$ is rescaled to unit mean diagonal before the
eigendecomposition, which leaves $B$ unchanged. Directions with zero
eigenvalue are set to zero instead of being completed. For $\eta<1$,
$A_\eta\succ0$, so $\ker M_\eta=\ker X$ and these directions carry no key
energy: the rebuilt keys $\onev\mu^\top+CB^\top$ are the same as with the
completion. At $\eta=1$ the two can differ, but only when
$\operatorname{rank}(G^{1/2}Y^\top X)<r$.

\subsection{The same basis rebuilds the keys; a PCA basis rebuilds the values}
\label{supp:recon}

\idea{The key basis and its coefficients form a rank-$r$ factorization of the centered keys. Attending
with the rebuilt keys perturbs each logit by exactly the error that scoring
already makes, so the analysis above covers key reconstruction with no new
term.}

\begin{proposition}[The key risk is the reconstruction error]
\label{supp:prop:key-recon}
Let $\widehat K=\onev\mu^\top+CB^\top$ be the keys rebuilt from the key basis.
Then $K-\widehat K=X(I_d-P_B)$ and
\begin{equation*}
\norm{K-\widehat K}_F^2=\trace\bigl((I_d-P_B)X^\top X(I_d-P_B)\bigr)=d\,\cR_{\mathrm{key}}(B).
\end{equation*}
Attending with $\widehat K$ instead of $K$ perturbs every logit of the page
by exactly $e_B(q)$, so \cref{supp:prop:first-order,supp:mass} bound the
output and mass errors of key reconstruction as well as those of scoring.
\end{proposition}

\begin{proof}
$K-\widehat K=(\onev\mu^\top+X)-(\onev\mu^\top+XBB^\top)=X(I_d-P_B)$, and
$\norm{XJ}_F^2=\trace(JX^\top XJ)$ for the symmetric idempotent $J=I_d-P_B$.
The logits of $\widehat K$ are $\widehat Kq/\sqrt d
=\onev\mu^\top q/\sqrt d+\widehat z_B(q)$.
\end{proof}

\noindent
At $\eta=0$ the basis gives the Eckart--Young optimal rank-$r$ key
reconstruction; for $\eta>0$ it trades key energy for output sensitivity in the
proportion set by $\eta$ (\cref{supp:cor:hybrid}). The deployed rank
is 8, and the stored bases are bfloat16.

\noindent
The value basis (\cref{sec:valcomp}) is the rank-$r_v$ PCA basis of the
centered page values with the page mean $\nu$ kept exactly; by
Eckart--Young it minimizes $\norm{V-\widehat V}_F$ over all reconstructions
$\onev\nu^\top+(\text{rank}\le r_v)$. With attention weights $\beta$ the
output error is $(V-\widehat V)^\top\beta$, so only the discarded value
directions contribute. Since $\onev^\top Y=0$, $\operatorname{rank}(Y)\le m-1$,
so for $m=32$ every rank $r_v\ge31$ rebuilds the values exactly. Such a basis
is lossless but larger than the page: at $r_v=32$ it stores
$(dr_v+mr_v+d)\cdot2=10496$\,B per page, more than the 8192\,B of the exact
values, so we do not use it. The deployed rank, $r_v=16$, keeps half of these
directions at 5376\,B (66\% of the exact values) and is nearly lossless in
practice; lower ranks lose accuracy on every task (\cref{tab:abl-valrank}).

\subsection{Beyond isotropic queries and uniform attention}
\label{supp:ext-whitened}

\paragraph{Non-isotropic queries.}
Let $\Sigma_q=\E_q[qq^\top]\succ0$, $\xi=\Sigma_q^{-1/2}q$ and
$\widetilde X=X\Sigma_q^{1/2}$, so that $\E[\xi\xi^\top]=I_d$ and
$z(q)=Xq/\sqrt d=\widetilde X\xi/\sqrt d$.

\begin{corollary}[Query-whitened basis]
\label{supp:cor:whitened}
With $\widetilde M_{\mathrm{out}}=\widetilde X^\top YGY^\top\widetilde X$,
$\widetilde c_{\mathrm{key}}=(\trace(\widetilde X^\top\widetilde X)+\varepsilon)^{-1}$
and $\widetilde c_{\mathrm{out}}=(\trace(\widetilde M_{\mathrm{out}})+\varepsilon)^{-1}$,
the leading rank-$r$ eigenspace of
$\widetilde M_\eta=(1-\eta)\widetilde c_{\mathrm{key}}\widetilde X^\top\widetilde X
+\eta\widetilde c_{\mathrm{out}}\widetilde M_{\mathrm{out}}$
minimizes the hybrid risk in the whitened query coordinates.
\end{corollary}

\begin{proof}
$\E[\xi\xi^\top]=\Sigma_q^{-1/2}\Sigma_q\Sigma_q^{-1/2}=I_d$, so the
derivations of \eqref{supp:eq:risk-traces} apply with $\widetilde X$ in place
of $X$, and \cref{supp:cor:hybrid} concludes.
\end{proof}

\noindent
At decode time the query code becomes $\widetilde B^\top\Sigma_q^{-1/2}q$ with
stored coefficients $\widetilde X\widetilde B$. A singular $\Sigma_q$ can be
whitened on its support or regularized as $\Sigma_q+\varepsilon I_d$.

\paragraph{Non-uniform within-page attention.}
The method linearizes at uniform attention. For a general reference $z$, let
$A_{\mathrm{out}}(z)=\mathcal S(\beta(z))VGV^\top \mathcal S(\beta(z))\succeq0$. By
\cref{supp:lem:jacobian},
$\norm{f(z)-f(z-e)}_G^2=e^\top A_{\mathrm{out}}(z)e+O(\norm{e}_2^3)$, which
gives the Gauss--Newton risk
\begin{equation*}
\cR_{\mathrm{GN}}(B)=\tfrac1d\E_q\bigl[q^\top(I_d-P_B)X^\top A_{\mathrm{out}}(z(q))X(I_d-P_B)q\bigr].
\end{equation*}
This yields a hierarchy: uniform attention gives one query-free
eigenproblem (the method); a fixed non-uniform reference gives a frozen
metric; and the risk above can be estimated from sampled queries, at the
cost of losing the single closed-form eigenproblem.

\section{Extended results: choosing \texorpdfstring{$\eta$}{eta} and accuracy versus memory}
\label{supp:ext}

\noindent
This section extends three studies of the main text. The first compares
the per-head rule with key PCA ($\eta=0$) on all three models
(\cref{supp:eta-sweeps}). The second is the MATH500 accuracy--memory study that extends \cref{fig:pareto}a, with
every arm and its memory (\cref{supp:pareto}). The third checks the
fit of the endpoint rule (\cref{supp:eta-rule}). The first two report accuracy under the long-reasoning protocol (\cref{supp:tab:protocols}): $k=8$
samples per problem (seeds 42--49), rank-8 key bases and FreeKV's problem
files (\cref{supp:same-setting}), on the evaluation set of
\cref{tab:reasoning}. The third measures output error, not accuracy, on the
held-out calibration traces (\cref{supp:calib}).

\subsection{The per-head rule against key PCA}
\label{supp:eta-sweeps}

\Cref{supp:tab:eta-paired} sets the per-head rule (\method in
\cref{tab:reasoning}) against $\eta=0$, each decoded in the same
configuration: rank-8 keys, $r_v=16$, the same problems and the same eight seeds (42--49), so every cell is an 8-seed mean for both arms, including the largest gaps ($+6.67$ and $+8.50$ on Llama-8B).
The rule values are the \method runs of \cref{tab:reasoning} and the
$\eta=0$ values are those of the shared-$\eta$ ablation (\cref{tab:abl-eta}) for
Qwen-7B; Llama-8B and Qwen-14B were run at $\eta=0$ in the same configuration. The rule is at or above $\eta=0$ in all nine
(model, task) cells, by $+2.91$ avg@8 points on average. An exact two-sided
sign test over the nine cells (ties excluded; eight positive, one tie) gives
$p=0.008$, significant at the 5\% level. The test treats the cells as
independent, although cells of one model share its calibration table.

\begin{table}[!htbp]
\centering
\footnotesize
\caption{\textbf{Per-head rule versus key PCA ($\eta=0$)}
(avg@8, \%), rank-8 keys, values at PCA rank $r_v=16$, on the evaluation
set of \cref{tab:reasoning}; both arms are means over the same 8 seeds (42--49). Rule: \method in \cref{tab:reasoning}; $\eta=0$:
one $\eta$ shared by all heads. The last column is the difference in points.
Sign test: exact two-sided sign test over the nine cells, ties excluded;
$^\ast$: significant at $p<0.05$.}
\label{supp:tab:eta-paired}
\begin{tabular*}{\linewidth}{@{\extracolsep{\fill}}llccc@{}}
\toprule
Model & Task & $\eta=0$ & Rule & Rule $-$ $\eta{=}0$ \\
\midrule
\multirow{3}{*}{Qwen-7B}
& MATH500 & 71.75 & 71.75 & $0.00$ \\
& AIME24  & 43.33 & 47.92 & $+4.58$ \\
& GPQA    & 38.50 & 39.50 & $+1.00$ \\
\midrule
\multirow{3}{*}{Llama-8B}
& MATH500 & 66.75 & 67.25 & $+0.50$ \\
& AIME24  & 40.83 & 47.50 & $+6.67$ \\
& GPQA    & 35.50 & 44.00 & $+8.50$ \\
\midrule
\multirow{3}{*}{Qwen-14B}
& MATH500 & 68.25 & 69.25 & $+1.00$ \\
& AIME24  & 65.00 & 65.42 & $+0.42$ \\
& GPQA    & 54.00 & 57.50 & $+3.50$ \\
\midrule
\multicolumn{4}{@{}l}{Mean over the nine cells} & $+2.91$ \\
\multicolumn{4}{@{}l}{Sign test: positive\,/\,negative\,/\,tie} & 8\,/\,0\,/\,1 \\
\multicolumn{4}{@{}l}{Sign test: $p$ (two-sided)} & $\mathbf{0.008}^\ast$ \\
\bottomrule
\end{tabular*}
\end{table}

\subsection{Accuracy--memory study: every arm}
\label{supp:pareto}\label{supp:pareto-list}

\Cref{supp:fig:pareto-math500} and \cref{supp:tab:pareto} show every arm of the
matched-budget study on MATH500 with R1-Distill Qwen-7B, which extends the AIME24 study of
\cref{fig:pareto}a; the arms and the memory model are defined in
\cref{supp:exp-pareto}. \method reaches the accuracy of FullKV (71.75 avg@8) at 1236.5\,KiB, or 66.0\% of FullKV's memory. With its bases and coefficients quantized to 4 bits (\cref{tab:abl-quant}), it reaches 71.50 avg@8 at 776.0\,KiB, or 41.4\% of FullKV's memory.

\begin{figure}[!htbp]
\centering
\includegraphics[width=0.7\linewidth]{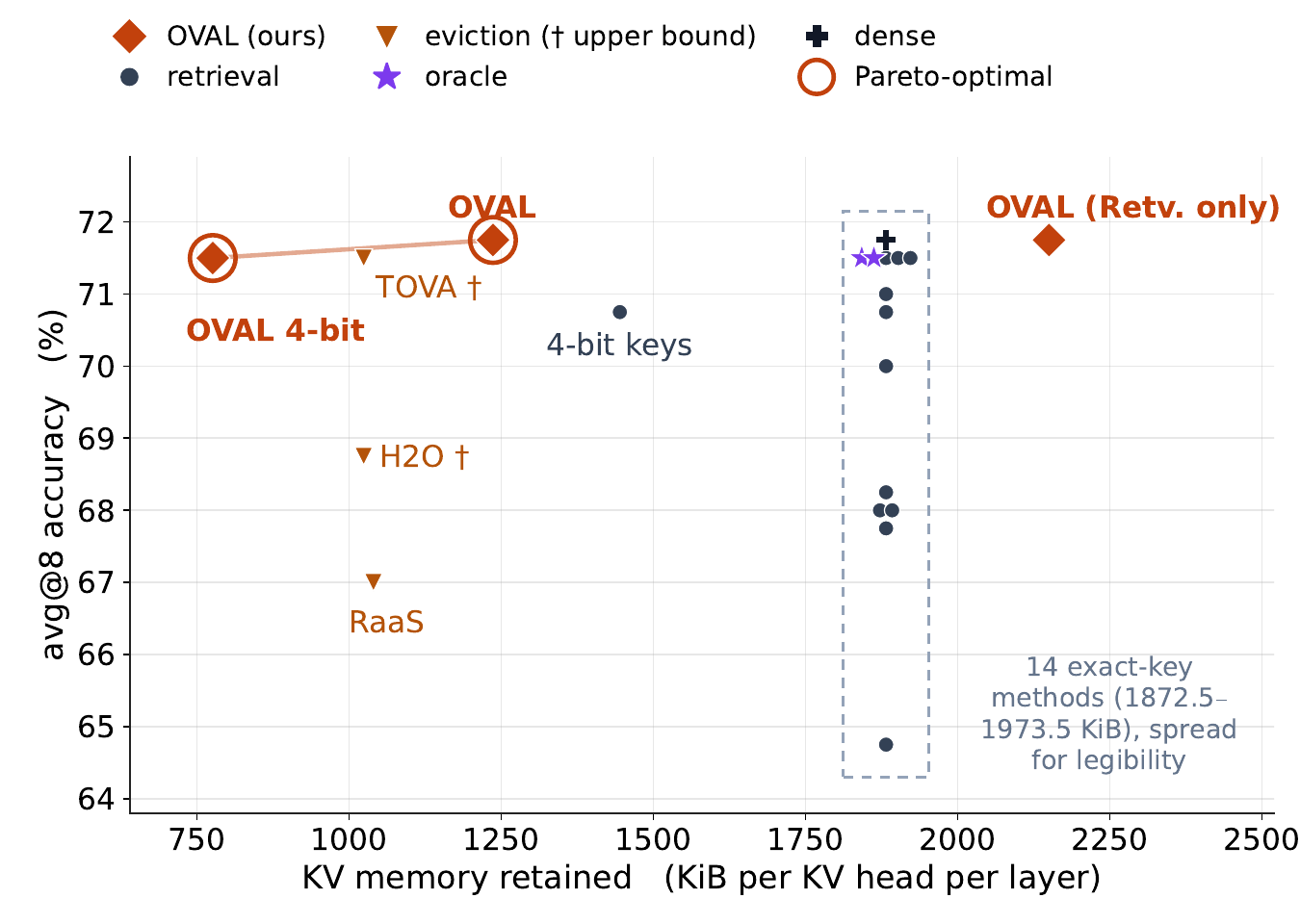}
\caption{\textbf{Accuracy versus retained KV memory} per KV head and layer on MATH500
(R1-Distill Qwen-7B): the 20 arms of the upper block of \cref{supp:tab:pareto} and \method (4-bit). Circles:
Pareto-optimal among entries that are not upper bounds or oracles; $\dagger$: upper bound.}
\label{supp:fig:pareto-math500}
\end{figure}

\begin{table}[!htbp]
\centering
\footnotesize
\setlength{\tabcolsep}{4pt}
\caption{Every arm of the matched-budget study (MATH500, R1-Distill
Qwen-7B, avg@8 over seeds 42--49, the 50 problems of
\cref{tab:reasoning}): the 20 arms drawn in \cref{supp:fig:pareto-math500} and \method (4-bit), with
the plotted accuracies and memory (upper block), and three arms that were
run but not drawn (lower block). \method (4-bit): \method with its bases and
coefficients quantized KIVI-style at 4 bits (\cref{tab:abl-quant}), also drawn in
\cref{supp:fig:pareto-math500}. \method: the main configuration of
\cref{tab:reasoning} (rank-8 keys, rank-16 PCA values, per-head
$\eta_h$), its MATH500/Qwen-7B cell. \method (Retv.\ only): pages scored with the rank-8 basis, exact keys and values attended. All arms, \method included, are our paired runs
(\cref{supp:exp-pareto,supp:same-setting});
for the six arms marked \cref{tab:reasoning}, our reruns recover the values FreeKV reports. Digest: bytes stored
per page to choose pages. Memory: KiB per KV head per layer on the reference
sequence (\cref{supp:exp-pareto}). $\dagger$: upper bound, kept off the Pareto front (\cref{supp:read}). Arms in the lower block are not drawn
in \cref{supp:fig:pareto-math500}: fixed windows make no selection, and PyramidKV is
a per-layer budget schedule on Quest.}
\label{supp:tab:pareto}
\begin{tabular*}{\linewidth}{@{\extracolsep{\fill}}llrrcl@{}}
\toprule
Arm & Family & Digest (B) & Memory (KiB) & avg@8 & Note \\
\midrule
\method & low-rank basis & 2816 & 1236.5 & 71.75 & \\
4-bit keys (per channel) & quantized keys & 2560 & 1445.0 & 70.75 & \\
\method (Retv.\ only) & low-rank basis & 2816 & 2150.3 & 71.75 & \\
\method (4-bit) & low-rank basis, 4-bit & 1056 & 776.0 & 71.50 & \cref{tab:abl-quant} \\
Dense (FullKV) & dense & 0 & 1872.5 & 71.75 & \cref{tab:reasoning} \\
Oracle $q(t{-}1)$ & ceiling & 0 & 1872.5 & 71.50 & \\
Oracle & ceiling & 0 & 1872.5 & 71.50 & \\
TOVA$^\dagger$ & attention history & 4 & 1024.1 & 71.50 & upper bound \\
SparQ & query sparsity & 0 & 1872.5 & 71.50 & \\
InfLLM & representative tokens & 8 & 1873.3 & 71.50 & \\
RocketKV & two-stage & 256 & 1897.8 & 71.50 & \\
Centroid (rank 0) & low-rank & 256 & 1897.8 & 71.00 & \\
Loki ($r=16$) & low-rank (shared basis) & 1024 & 1973.5 & 70.75 & \\
FreeKV & bound digest & 512 & 1923.0 & 70.00 & \cref{tab:reasoning} \\
H2O$^\dagger$ & attention history & 4 & 1024.1 & 68.75 & upper bound \\
ArkVale & bound digest & 512 & 1923.0 & 68.25 & \cref{tab:reasoning} \\
Quest & bound digest & 512 & 1923.0 & 68.00 & \cref{tab:reasoning} \\
Expected Attention & query-agnostic & 4 & 1872.9 & 68.00 & \\
K-norm & query-agnostic & 2 & 1872.7 & 67.75 & \\
RaaS & eviction & 516 & 1040.1 & 67.00 & \cref{tab:reasoning} \\
ShadowKV & representative tokens & 256 & 1897.8 & 64.75 & \cref{tab:reasoning} \\
\midrule
StreamingLLM & fixed window & 0 & -- & 58.50 & not plotted \\
SnapKV$^\dagger$ & attention history & 0 & 1024.0 & 68.00 & not plotted \\
PyramidKV (on Quest) & budget schedule & 512 & 1923.0 & 71.50 & not plotted \\
\bottomrule
\end{tabular*}
\end{table}

\subsection{Per-head \texorpdfstring{$\eta_h$}{eta}: fit of the rule}
\label{supp:eta-rule}

The endpoint rule (\cref{eq:etarule}) sets $\eta_h$ from the output
errors $(a_h,b_h,c_h)$ measured at $\eta=0$ and $\eta=1$ on calibration
traces 1--4 (\cref{supp:calib}). This section checks the quadratic model $f_h$ against the output error measured
over the full $\eta$ range. Everything is in the \method setting: rank-8 key
bases and rank-16 PCA value bases.

\noindent\textbf{Fit of the endpoint model.}
On the two held-out MATH500 calibration traces (5 and 6), which are not used to set $\eta_h$, we measure each head's
output error $\sum_q\norm{\Delta_h^{\eta}(q)}^2$ on a 21-point grid of
$\eta$ and compare it with $f_h$ fitted on traces 1--4
(\cref{supp:fig:eta-fit}). For most heads $f_h$ follows the measured curve
and the rule lands at or next to its minimum.
\Cref{supp:fig:eta-regret} summarizes all heads by the
held-out error at the rule's $\eta_h$ relative to the best grid value. On
Qwen-7B the median excess is 0.08\%, 44\% of heads sit exactly at the grid
optimum, and 90\% are within 4.8\%. On Llama-8B the median is 0.76\% and
90\% are within 10.2\%. The remaining tail comes from heads whose measured
error is not convex in $\eta$ (\eg Llama-8B layer 5, head 2 in
\cref{supp:fig:eta-fit}); the two-point model cannot resolve such heads
and places them at the wrong endpoint. Resolving them would require
measuring interior values of $\eta$, at the cost of more calibration
passes.

\begin{figure*}[!htbp]
\centering
\includegraphics[width=0.8\textwidth]{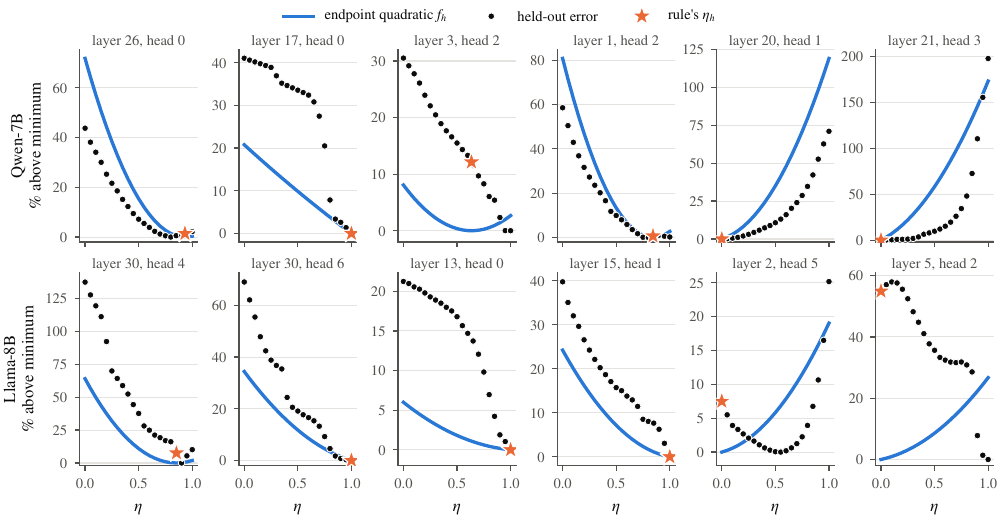}
\caption{\textbf{Endpoint model versus measured output error} on MATH500,
six heads per model chosen at quantiles of $b_h-a_h$. Dots: output error on
held-out calibration traces over 21 values of $\eta$ (\% above its
minimum). Line: $f_h$ fitted on traces 1--4. Star: the rule's $\eta_h$.}
\label{supp:fig:eta-fit}
\end{figure*}

\begin{figure}[!htbp]
\centering
\includegraphics[width=0.45\linewidth]{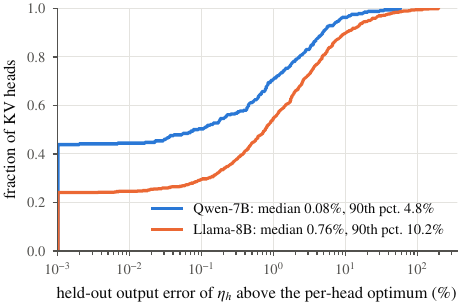}
\caption{\textbf{How far the rule lands from each head's best $\eta$}:
cumulative distribution over KV heads (layers $\ge1$, held-out MATH500
calibration traces) of the held-out output error at $\eta_h$ above the best of 21
grid values. Values below $10^{-3}\%$, including exact hits, are drawn at
$10^{-3}\%$.}
\label{supp:fig:eta-regret}
\end{figure}

\section{Experimental reference}
\label{supp:impl}

\noindent
This section states each setting once. \Cref{supp:tab:flags} gives the switches that turn each piece of \method on and
\cref{supp:tab:protocols} the benchmark protocols. The video experiments have their own
protocol and baseline ports (\cref{supp:protocols}). The subsections that
follow give only what differs from these tables.

\subsection{Harness, models and the switches of \method}
\label{supp:impl-harness}\label{supp:impl-flags}

\method is implemented inside the accuracy harness of
FreeKV~\citep{liu2026freekvboostingkvcache} (pure PyTorch) and, for latency,
inside FreeKV's efficient engine (CUDA extension). The method (bases,
scoring, value bases, per-head $\eta$ tables, the endpoint rule) is a
separate module imported by one patch to the harness
(\cref{supp:repro}). Every run in this harness uses batch size 1. The video
experiments run in a separate harness for the two video LLMs
(\cref{supp:protocols}); the switches below refer to the text harness.

\noindent\textbf{Models.}
R1-Distill Llama-8B, Qwen-7B and Qwen-14B
(\fl{deepseek-ai/DeepSeek-R1-Distill-*} on the Hugging Face hub) for
reasoning; Llama-3.1-8B-Instruct (\fl{NousResearch/Meta-Llama-3.1-8B-Instruct}),
Qwen2.5-7B-Instruct and Qwen2.5-14B-Instruct (\texttt{Qwen/\ldots}) for long
context. All are loaded in bfloat16 with eager attention,
\texttt{device\_map=auto} and the slow tokenizer, with a 65{,}536-token
context limit; stop ids come from each checkpoint's
\texttt{GenerationConfig}. The video LLMs are LLaVA-Video-7B-Qwen2 and
Qwen2.5-VL-7B-Instruct (\cref{supp:protocols}).

\noindent\textbf{Bases.}
Pages hold $m=32$ tokens and $d=128$. Bases are built in float32 when a
page completes and stored in bfloat16 (\texttt{OVAL\_QUANT=0}). $A_\eta$ is formed in the $m\times m$ token
domain (\cref{supp:small-gram}). Page scores
are combined over the GQA group by \texttt{avgSM}, and layer 0 stays dense.
Sink and recent tokens are never paged and stay exact under every switch.

\begin{table}[!htbp]
\centering
\footnotesize
\caption{Switches of the patched harness. Each ablation differs from
\method only in the switches it varies. Baselines are selected by
\texttt{--page\_rep} (\texttt{quest}, \texttt{arkv}) or \texttt{--method}
(\texttt{raas}, \texttt{full}).}
\label{supp:tab:flags}
\begin{tabularx}{\linewidth}{@{}l>{\raggedright\arraybackslash}X@{}}
\toprule
Piece & Switches \\
\midrule
Retrieval & \texttt{--page\_rep oval --oval\_rank 8 --GQA\_policy avgSM} \\
Key reconstruction & \texttt{--oval\_recon}: selected keys rebuilt as $\onev\mu_P^\top+C_PB_P^\top$ from the scoring basis; \texttt{--oval\_recon\_rank R} rebuilds at rank $R$ while scoring stays at rank 8 (key-rank ablation only; all other runs rebuild at rank 8) \\
Value basis & \texttt{--oval\_vrank 16 --vbasis pca}: selected values rebuilt from a rank-16 PCA basis; unset: exact values \\
Shared $\eta$ & \texttt{--eta E}, one value for every head ($E=0$ is page-local key PCA) \\
Per-head $\eta_h$ & \texttt{--eta\_table T.json}, a \texttt{[n\_layers][n\_kv\_heads]} table that overrides \texttt{--eta} \\
Fixed in every run & \texttt{--page\_size 32 --skip\_layer 1 --sparsity 1 --last\_layer\_budget 0 --batch\_size 1}; \texttt{OVAL\_QUANT=0}; \texttt{KVC\_MAX\_TOKENS} 32{,}768 (reasoning) or 65{,}536 (LongBench~v2) \\
\bottomrule
\end{tabularx}
\end{table}

\subsection{Benchmark protocols and statistics}
\label{supp:protocols}

\begin{table}[!htbp]
\centering
\footnotesize
\caption{Benchmark protocols. The attention budget is the total number of
tokens attended per decoding step, including the exact sink and recent
tokens, in every setting and in the latency engine (\cref{supp:latency}).}
\label{supp:tab:protocols}
\begin{tabular*}{\linewidth}{@{\extracolsep{\fill}}llll@{}}
\toprule
& Long reasoning & LongBench~v2 & LongGenBench \\
\midrule
Data & MATH500 50, AIME24 30, GPQA 50 & 503 (S/M/L 180/215/108) & 400 (100 per type) \\
Selected tokens & 1024 (32 pages) & 1792 (56 pages) & 1024 (32 pages) \\
Exact sink / recent (each) & 512 & 128 & 512 \\
Tokens attended & $1024+512+512=2048$ & $1792+128+128=2048$ & $1024+512+512=2048$ \\
Decoding & $T=0.6$, top-$p$ 0.95 & greedy & $T=0.95$, top-$p$ 0.95 \\
Max.\ generation & 16{,}384 & 128 & 16{,}000 \\
Input & unchanged & middle-truncated to 65{,}536 & unchanged \\
Stop & EOS & EOS & \texttt{*** finished} \\
Seeds & 42--49 ($k=8$) & 42 & 42 \\
Metric & avg@8, pass@8 & Overall, Short, Med., Long & CR, CR$\times$Acc \\
\bottomrule
\end{tabular*}
\end{table}

\paragraph{Long reasoning: same setting as FreeKV.}\label{supp:same-setting}
All long-reasoning runs use the data files shipped with the FreeKV harness,
unmodified: the 50-problem MATH500 subset (\texttt{math50.jsonl}, SHA-256
prefix \texttt{6614a226}), the 30 problems of AIME 2024
(\texttt{aime24.jsonl}, \texttt{68866ba7}) and the 50-problem GPQA subset
(\texttt{gpqa50c.jsonl}, \texttt{164808b8}). Decoding follows FreeKV:
$T=0.6$, top-$p$ 0.95, up to 16{,}384 generated tokens, seeds 42--49
($k=8$), a 2048-token budget that includes 512 exact sink and 512 exact recent
tokens (1024 selected tokens), 32-token pages, layer 0 dense, and bfloat16. The prompt is the problem, an instruction to
reason step by step and to put the final answer in \verb|\boxed{}|, then
\verb|<think>|, with no chat template. Each seed gives one sampled
generation per problem. avg@$k$ averages accuracy over the $k$ generations
and the problems; pass@$k$ is the fraction of problems solved by at least
one generation.

\noindent\textbf{LongBench~v2.}
The prompt is the benchmark's \texttt{0shot.txt}, without chain of thought or
retrieval; longer inputs keep their first and last 32{,}768 tokens.
Short, Medium and Long use the benchmark's own length labels, and scoring
follows its \texttt{result.py}.

\noindent\textbf{LongGenBench.}
CR is the completion rate (reference \texttt{parse\_blocks} and
\texttt{calculate\_completion\_rate}); Acc is judged by Qwen3-32B served with
vLLM, the judge FreeKV uses.

\noindent\textbf{Video understanding.}
Two video LLMs, LLaVA-Video-7B-Qwen2 and Qwen2.5-VL-7B-Instruct, are
evaluated on two benchmarks. \mbox{Video-MME} is run without subtitles on all
2700 questions of its 900 videos; the metric is accuracy per duration
split (Short, Medium, Long) and overall. LVBench is run on all 1549
questions of its 103 videos, frames only; the metric is overall accuracy.
Each video is represented by $N=64$ uniformly sampled frames. LLaVA-Video
encodes each frame into 210 tokens (13,440 visual tokens; about 14,043
prefill tokens on \mbox{Video-MME}). Qwen2.5-VL encodes pairs of adjacent frames
into 32 temporal groups; its visual-token count depends on the video's resolution (mean 13,133 on
\mbox{Video-MME}, range 12,480--14,336), and it keeps its
own 3-D M-RoPE positions. The video is cached before the question is seen
(query-agnostic protocol, \cref{sec:exp-video}). Decoding is greedy in
bfloat16. As a sanity check, streamed dense inference matches
question-in-prompt dense inference on 97 (LLaVA-Video) and 98 (Qwen2.5-VL)
of the first 100 Long \mbox{Video-MME} questions; the LLaVA-Video differences are
bfloat16 near-ties with logit margins below 0.25.

\noindent\textbf{Video budgets.}
Quest, FreeKV and SnapKV$^\dagger$ keep the full cache, and
\method keeps the page bases in its place; all four attend to 2048 tokens
per decoding step: 1792 selected tokens
plus 128 exact sink and 128 exact recent tokens, with 32-token pages and a
dense first layer. This is 15.2\% of the visual tokens for LLaVA-Video and
15.6\% (mean) for Qwen2.5-VL. \method uses rank-8 key bases for scoring and key
reconstruction, rank-16 PCA value bases, and one $\eta_h$ table per model
(\cref{supp:calib}). SnapKV cannot see the question, so its observation
window is the last 32 tokens of the video prefill; the selected pages are
then kept fixed ($\dagger$, \cref{supp:read}). The video-specific methods and FullKV with fewer frames reduce what
is cached: FullKV with $N/2$ frames, VisionZip at 50\% or 25\% of the visual
tokens, InfiniPot-V memory-matched (50\%) and at about 2K tokens (about 2.1K for Qwen2.5-VL, 2.7K for LLaVA-Video, text prefix included),
and ReKV, which keeps the full cache, answers from 8 frames (LLaVA-Video) or 4 temporal groups (Qwen2.5-VL),
1.6K--2.0K tokens with sink and recent (recorded cache lengths).

\noindent\textbf{Video-specific baselines (our ports).}
\label{supp:video-ports}
No official implementation of VisionZip, InfiniPot-V, or ReKV supports both models on our stack, so we port each
method's core step.
\emph{VisionZip}: the score is the attention each visual token receives at the last vision-encoder layer
(head mean, summed over queries); the dominant tokens are a global top-$k$, and 5\% of the tokens are contextual
tokens formed by merging the remaining tokens into their most similar target by key (cosine). Retention 50\% uses
45\% dominant tokens (the official setting) and 25\% uses 20\%. For LLaVA-Video, which uses a SigLIP encoder
without a CLS token, the same score is taken at the encoder layer the model reads and pooled to its 14$\times$14 grid;
the row-separator tokens are always kept. For Qwen2.5-VL, the kept tokens keep their M-RoPE positions.
\emph{InfiniPot-V}: combines the two official scores, temporal redundancy (negative cosine similarity of keys at the same
spatial position between the newest 25\% of frames and older frames) and value norm, 50/50,
per layer and KV head during block-wise prefill.
\emph{ReKV}: one block per frame (per temporal group for Qwen2.5-VL), represented by its mean key; the mean
query of the question tokens scores the blocks once per question, and the top blocks (8 frames for LLaVA-Video,
4 groups for Qwen2.5-VL) plus the 128 sink and 128 recent tokens form the cache used for answering. Unlike the
official code, keys keep their original positions, since contiguous re-indexing has no M-RoPE equivalent.
InfiniPot-V and ReKV run on the Hugging Face attention path; dense attention on that path matches our harness
(LLaVA-Video: $-$0.04 points [$-$0.33, $+$0.26] on \mbox{Video-MME}; Qwen2.5-VL: identical outputs on 100 of 100 questions).

\noindent\textbf{Per-frame token selection.}
\Cref{supp:fig:videomme-frames} shows the share of each frame's visual tokens that \method selects on the Long
split of \mbox{Video-MME}. Selection concentrates at the start and end of the video.
\Cref{supp:fig:filmstrips} shows the same quantity on one video per benchmark for both models. Each tile is one Qwen2.5-VL temporal
group (two adjacent frames; the first is shown), and its value is the share of the tile's visual tokens that
\method selects, averaged over decoding steps, layers, heads, and all questions of the video; for LLaVA-Video
the two frames are pooled. Both models sample the same 64 frames. The videos are not hand-picked: in each
benchmark we show the video whose per-tile profiles are closest to the benchmark-wide mean profiles
(\mbox{Video-MME}: among the 300 Long videos; LVBench: among all 103).

\begin{figure}[!htbp]
\centering
\includegraphics[width=0.6\linewidth]{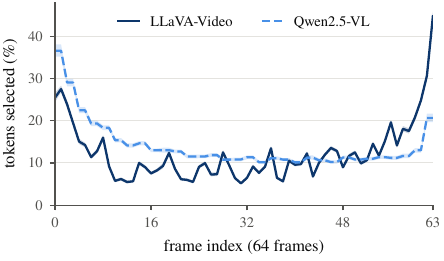}
\caption{\textbf{Per-frame token selection.} Share of each frame's visual tokens selected by \method, averaged
over decoding steps, layers, heads, and all 900 Long \mbox{Video-MME} questions.
The first and last four frames receive 2.66$\times$ (LLaVA-Video) and 1.97$\times$ (Qwen2.5-VL) the share of the
middle frames. Single-video examples are in \cref{supp:fig:filmstrips}.}
\label{supp:fig:videomme-frames}
\end{figure}

\begin{figure*}[t]
\centering
\includegraphics[width=0.85\textwidth]{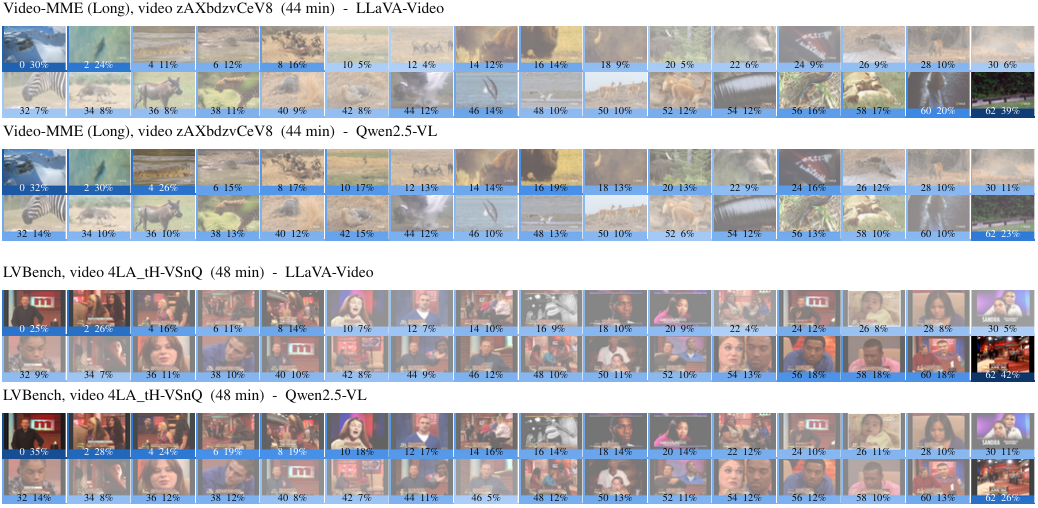}
\caption{\textbf{Per-frame selection of \method on one video per benchmark, both models.} Tiles are washed out
and shaded in proportion to the share of their visual tokens that \method selects; the strip under each tile
gives the first frame index and that share. Selection concentrates on the first and last tiles for both models
and both benchmarks; on these videos \method selects 12.4\% (LLaVA-Video) and 14.3\% (Qwen2.5-VL) of all visual
tokens on \mbox{Video-MME}, and 12.5\% and 14.3\% on LVBench.}
\label{supp:fig:filmstrips}
\end{figure*}

\noindent\textbf{Statistics.}
Each configuration is run once (one generation per seed), and a table that reuses a run says so:
the \method entry of \cref{supp:tab:pareto} and the Rule column of
\cref{supp:tab:eta-paired} are the \method cells of \cref{tab:reasoning},
and the FullKV and baseline values of \cref{tab:reasoning,tab:longbench} are those
reported by FreeKV (\cref{supp:read}). Equal values across
different runs (\eg 71.75 on MATH500 for several arms) are expected
and do not indicate shared results: avg@8 on 50 problems moves in steps of
0.25 points (30 problems: 0.42 points), and under the long-reasoning
protocol every arm attends the full cache until a trace exceeds 2048
tokens.

\subsection{Calibration of \texorpdfstring{$\eta_h$}{eta\_h}}
\label{supp:calib}

Each $\eta_h$ table is calibrated once per model, offline, on data disjoint
from every evaluation set.
\begin{itemize}
\item \emph{Capture.} The model generates six calibration traces (long
reasoning: six MATH500 problems outside the evaluation subset, 6144 tokens
each, seed 7); traces 1--4 set $\eta_h$ and traces 5--6 are held out to
evaluate the rule (\cref{supp:eta-rule}). Keys and values are read from the cache, and windows of 96
queries are taken near the end of each trace.
\item \emph{Endpoint errors.} For each window, \method at $\eta=0$ and
$\eta=1$, in the deployed configuration (rank-8 keys, rank-$r_v$ values, the
same budget, sink and recent tokens), is compared with dense attention
through $\Delta_h^\eta(q)$ (\cref{eq:delta}).
\item \emph{Rule.} The $(a_h,b_h,c_h)$ accumulated over traces 1--4 give $\eta_h$
(\cref{eq:etarule}), stored as one table per model.
\item \emph{Other settings.} The same procedure is applied per model and
benchmark for long context (LongBench-Pro documents for LongBench~v2;
non-evaluated LongGenBench items) and per video model under the
query-agnostic protocol (six LLaVA-Video-178K clips not in \mbox{Video-MME} or
LVBench, four to set $\eta_h$ and two held out, capturing the prefill). On the held-out clips, the median head's
output error at $\eta_h$ is within 2\% of its grid-search best.
\end{itemize}

\subsection{Accuracy--memory study: arms and memory model}
\label{supp:exp-pareto}

\noindent\textbf{Arms.}
All arms attend the same number of tokens (reasoning protocol, MATH500). MATH500
prompts fit inside the 512-token sink, so every page holds generated tokens.
\Cref{supp:tab:arms} defines each arm. H2O, TOVA and SnapKV are $\dagger$ arms
(\cref{supp:read}); since the prompt lies in the sink, H2O and SnapKV anchor their selection at the
first selection step.

\begin{table}[!htbp]
\centering
\footnotesize
\caption{What each arm of the accuracy--memory study does. All arms attend
the same number of tokens; they differ in what they store per page and how
they choose pages.}
\label{supp:tab:arms}
\begin{tabular*}{\linewidth}{@{\extracolsep{\fill}}lp{0.72\linewidth}@{}}
\toprule
Arm & Definition \\
\midrule
\method & Pages scored with rank-8 output-aware key bases; keys and values stored as rank-8 key bases and rank-16 PCA value bases; selected pages rebuilt from the bases. \\
\method (4-bit) & As \method, with the basis columns, key coefficients and value coefficients quantized KIVI-style at 4 bits (\cref{supp:exp-pareto}, memory model). \\
\method (Retv.\ only) & Pages scored with rank-8 output-aware key bases with the per-head $\eta_h$; selected pages attended with their exact keys and values (the basis is only an index). \\
Oracle & Chooses pages by the true attention mass of the current query; needs all exact keys, so it is a ceiling, not a method. \\
Oracle $q(t{-}1)$ & As Oracle, with the previous step's query. \\
Quest & Per-page channel-wise min and max of the keys; upper-bound score. \\
ArkVale & Per-page bounding box (center $\pm$ spread) of the keys. \\
FreeKV & Quest-style page digest, as in FreeKV~\citep{liu2026freekvboostingkvcache}. \\
Loki ($r=16$) & Keys projected on one shared low-rank basis calibrated offline; pages scored in that basis. \\
Centroid & Page mean key (rank 0); score $\mu_P^\top q$. \\
4-bit keys & Keys stored at 4 bits per channel and used for both scoring and attention. \\
SparQ & Reads the keys only on the top $d/8$ query dimensions to score. \\
InfLLM & Four representative tokens per page (largest key norm). \\
ShadowKV & Page centroid plus outlier tokens kept exact. \\
RocketKV & Two stages: coarse centroid filter, then exact scoring of the survivors. \\
H2O$^\dagger$ & Accumulated attention each page received so far. \\
TOVA$^\dagger$ & Attention each page received at the previous step. \\
SnapKV$^\dagger$ & Attention from an observation window, frozen once selection starts. \\
Expected Attention & Attention expected under a query distribution estimated from the prompt (per-head mean and covariance). \\
K-norm & Key magnitude only, no query. \\
PyramidKV & Quest pages with a budget that falls linearly over the layers ($1.5\times$ to $0.5\times$). \\
RaaS & Reasoning-aware eviction: keeps pages that were attended recently. \\
StreamingLLM & Sink and recent tokens only, no selection. \\
FullKV & Full attention over the exact cache. \\
\bottomrule
\end{tabular*}
\end{table}

\noindent\textbf{Memory model.}
Memory is the resident KV state of one KV head in one layer, in bytes, on
one fixed reference sequence for all arms: $L=3745$ tokens (the FullKV mean trace
length), $d=128$, 2 bytes per element, so one token of keys or values takes
256\,B. The length is fixed on purpose: weaker methods generate longer
traces, and charging each arm its own length would count accuracy twice.
The 1024 sink and recent tokens stay exact for every arm. The paged region
holds $\lfloor(3745-1024)/32\rfloor=85$ selectable pages; the one leftover token
($3745-1024-85\cdot32=1$) stays exact. Digests are written when a page completes,
even inside the recent window, so $\lfloor(3745-512)/32\rfloor=101$ pages carry one.
FullKV stores all $3745$ tokens. Retrieval arms store the exact cache plus
one digest per page; eviction arms store the 2048 survivors plus their
digests. \method stores each selectable page as a rank-8 key basis (which is
also its digest) and a rank-16 PCA value basis, $8192$\,B per page instead
of $16384$\,B, plus a $2816$\,B key basis for each of the 16 digested pages still in the recent window, giving $1236.5$\,KiB, $66.0\%$ of FullKV ($1872.5$\,KiB). The
quantized arms (\cref{tab:abl-quant}) apply KIVI-style~\citep{liu2024kivi}
quantization at 4 or 2 bits to the paged entries (FreeKV) or to the bases and
coefficients (\method), with per-group bfloat16 scales; this gives 988 and
818\,KiB for FreeKV and 776 and 691\,KiB for \method.

\noindent\textbf{AIME24, R1-Distill Qwen-14B} (\cref{fig:pareto}a). The same memory model, on a reference
sequence of $L=10100$ tokens (the FullKV mean trace length, rounded) gives 283 selectable and 299 digested pages.
FullKV stores 5050.0\,KiB; FreeKV, Quest and ArkVale (512\,B digests) 5199.5\,KiB (103.0\%); ShadowKV
(256\,B) 5124.8\,KiB (101.5\%); RaaS keeps 2048 tokens, 1040.1\,KiB; and \method 2830.0\,KiB (56.0\%). At 4
and 2 bits, FreeKV stores 2086.5 and 1520.5\,KiB (41.3\% and 30.1\%) and \method 1361.0 and 1090.2\,KiB
(27.0\% and 21.6\%).

\noindent\textbf{Scope of the memory model.}
The model counts stored bytes, computed from the configuration (\cref{supp:read}). For the video and robot tables
(\cref{tab:videomme}, \cref{supp:tab:lrmain}), the same per-page accounting is
applied to every layer and KV head of the model (28 layers, 4 KV heads,
bfloat16) at each question's own final length, and reported in MiB per
question. There, \method keeps the 128 sink and 128 recent tokens and the
dense first layer exact (scales: \cref{supp:read}).

\subsection{Latency}
\label{supp:latency}

Latency is measured in FreeKV's efficient engine on one RTX A6000, batch 1,
greedy decoding, with the same 2048-token budget as the accuracy runs
(512 sink, 512 recent, 1024 selected tokens) and the KV pages offloaded to
host memory. The metric is the mean time per generated token over
16{,}384 tokens of AIME24 problem 0, after one warm-up generation
(\cref{fig:pareto}b,\,c). \method stores and scores only the page bases
and rebuilds the selected pages. $\eta_h$ is computed once per model,
offline, and is not part of the decode time. \method takes 40 and 47\,ms per token on R1-Distill Qwen-7B and
Llama-8B, FreeKV 36 and 42\,ms, and ArkVale 72 and 143\,ms.

\subsection{Prefill time}
\label{supp:prefill}

\method builds key-value bases for every page when the page is completed, so a
long prompt pays for all of its pages at prefill. We measure the prefill time
(time to first token) in the engine and configuration of the latency study
(\cref{supp:latency}), with 32-token pages and rank-8 key bases
with the model's $\eta_h$ table (calibrated on MATH500). The prompt is the engine's gov\_report document,
truncated from the middle to $n\in\{4,8,16,32\}$K tokens; each configuration runs
one warm-up generation and one timed generation, three times, and we report the
median. \Cref{supp:fig:prefill} shows the result. Building the bases adds 13\% to 9\% to
FreeKV's prefill time on R1-Distill Qwen-7B and 32\% to 22\% on Llama-8B as the prompt grows from 4K to 32K tokens;
Llama-8B has twice as many KV heads per layer, hence twice as many bases per page. The
overhead falls with $n$ because the bases cost $O(n)$, two fixed-size eigenproblems per page, while
prefill attention grows faster than linearly in $n$. The cost is paid once per
prompt; decoding is unaffected (\cref{supp:latency}).

\begin{figure*}[!htbp]
\centering
\includegraphics[width=\textwidth]{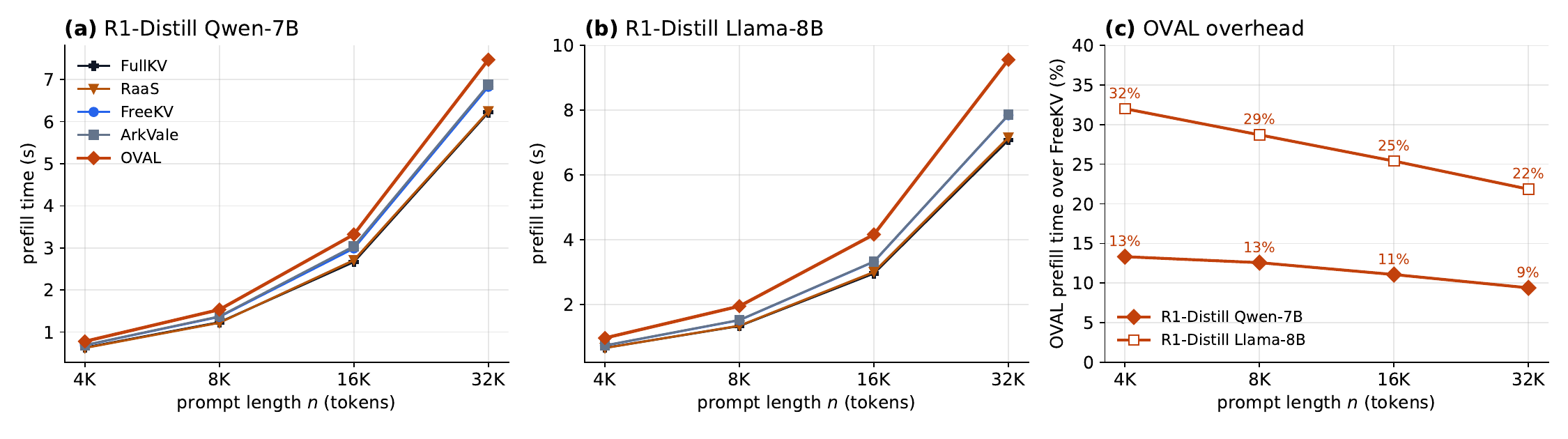}
\caption{\textbf{Prefill time} versus prompt length $n$ in FreeKV's engine (one
RTX A6000, batch 1; median of three runs). \textbf{(a,\,b)}~Time to first token.
\textbf{(c)}~\method's prefill time relative to FreeKV; the overhead of building the
bases falls as $n$ grows.}
\label{supp:fig:prefill}
\end{figure*}

\subsection{Reproducing the results}
\label{supp:repro}\label{supp:repro-code}

\method is applied as a patch to a pinned commit of the FreeKV harness.
The code is available at \url{https://github.com/Ashkan13776/oval-kv}, together with the setup, drivers and
scorers of every experiment.

\section{LIBERO Watch-and-Do: robot-episode recall, executed by a VLA policy}
\label{supp:watchdo}

\noindent\textbf{Construction.}
Watch-and-Do uses the ten LIBERO-Goal tasks~\citep{liu2023libero}, which share one kitchen scene, and their
demonstrations in the LeRobot release~\citep{cadene2024lerobot} (third-person camera, 256$\times$256 pixels,
10\,fps). A session draws three distinct tasks and one random demonstration of each and plays them back to
back as one video (about 30\,s). Each of the 80 sessions has three instructions: repeat the task performed
\emph{first}, the one performed \emph{last}, or one task that was \emph{present} (240 episodes). The
planner chooses among four commands: the target and three other LIBERO-Goal tasks (the two other session tasks and one task not in the session for
\emph{first} and \emph{last}, tasks not in the session for \emph{present}). We use three tasks per session;
with six tasks per session, even FullKV chose the right command near chance. Construction is seeded (seed 0),
and the session list is released with the code.

\noindent\textbf{Planner and execution.}
The planner is each video LLM under each KV method, run as in \cref{tab:videomme}: 64 frames, the
query-agnostic protocol, the same budgets and $\eta_h$ tables. The chosen command is executed by
OpenVLA-7B fine-tuned on LIBERO-Goal~\citep{kim2024openvla}, with the official OpenVLA LIBERO evaluation
loop unchanged (center crop, 10 no-op steps, at most 300 policy steps, seed 7, reset before every
rollout), from the LIBERO initial state of the target task assigned to the episode. An episode succeeds
if the simulator's goal check of the target task passes; none of the 452 distinct wrong-command
rollouts succeeded. Rollouts are keyed by (command, target, initial state) and
shared across arms, so two arms that choose the same command in an episode get the same outcome.
The oracle row executes the target command, and the random row a uniformly drawn option.

\noindent\textbf{Statistics and memory.}
Paired differences in success of \method against the other arms:
Qwen2.5-VL: FullKV $+$1.67, Quest $+$0.42, FreeKV $-$0.42, ReKV $+$2.92, InfiniPot-V $+$4.58, SnapKV
$+$4.58; LLaVA-Video: FullKV $-$1.25, Quest $-$3.75, FreeKV $-$1.67, ReKV $+$6.67, InfiniPot-V $+$6.67,
SnapKV $+$0.83. Stored KV follows the memory model of
\cref{supp:exp-pareto} at each episode's final length (about 14K tokens), as for the video tables;
token pruning and fewer frames were not run on this task.

\begin{table}[!htbp]
\centering
\footnotesize
\setlength{\tabcolsep}{2.6pt}
\begin{tabular}{@{}lrccrcc@{}}
\toprule
 & \multicolumn{3}{c}{Qwen2.5-VL} & \multicolumn{3}{c}{LLaVA-Video} \\
\cmidrule(lr){2-4}\cmidrule(l){5-7}
\textbf{Method} & KV & Cmd. & Succ. & KV & Cmd. & Succ. \\
\midrule
Oracle command & -- & 100 & 75.83 & -- & 100 & 75.83 \\
Random command & -- & 24.58 & 21.25 & -- & 24.58 & 21.25 \\
\midrule
FullKV ($N$) & 778 & 53.33 & 40.83 & 766 & 43.33 & 32.08 \\
\midrule
InfiniPot-V$^\ddagger$ & 126 & 47.08 & 37.92 & 146 & 30.00 & 24.17 \\
ReKV$^\ddagger$ & 778$^{\S}$ & \topthree{49.17} & 39.58 & 768$^{\S}$ & 31.25 & 24.17 \\
SnapKV$^\dagger$ & 136 & 47.92 & 37.92 & 135 & 39.58 & 30.00 \\
Quest & 801 & \toptwo{52.92} & \topthree{42.08} & 789 & \topone{46.25} & \topone{34.58} \\
FreeKV & 801 & \topone{54.17} & \topone{42.92} & 789 & \toptwo{42.50} & \toptwo{32.50} \\
\method{} & 413 & \topone{54.17} & \toptwo{42.50} & 407 & \topthree{41.25} & \topthree{30.83} \\
\bottomrule
\end{tabular}
\caption{\textbf{LIBERO Watch-and-Do}: correct command chosen by the planner (Cmd., \%),
OpenVLA task success (Succ., \%), and stored KV cache per episode (MiB); 240 episodes per arm, all
methods except FullKV read about 2K tokens per step (InfiniPot-V with LLaVA-Video: about 2.7K). Marks as in \cref{tab:videomme}.}
\label{supp:tab:watchdo-full}
\end{table}

\section{LIBERO-Recall: robot-episode memory}
\label{supp:liberorecall}

\noindent\textbf{Construction.}
LIBERO-Recall is built from the LIBERO demonstrations~\citep{liu2023libero} in the LeRobot
release~\citep{cadene2024lerobot}: 40 tasks in four suites (LIBERO-10, -Goal, -Object, -Spatial), 1693
episodes, third-person camera, 256$\times$256 pixels at 10\,fps. A session draws eight distinct tasks uniformly from
the 40 and one random demonstration of each, and plays them back to back as one video (1.6--3.2 minutes, mean 2.2).
There are 100 sessions and six four-way questions per session, one of each type:
\emph{present} (which task did the robot perform; one session task and three tasks not in the session),
\emph{absent} (which task did it not perform; one task not in the session and three session tasks),
\emph{first} and \emph{last} (four session tasks, one of them the first or last), and
\emph{after} and \emph{before} (which task came right after or before a named session task $X$; the true
successor or predecessor and three other session tasks, never $X$).
Options not in the session are drawn from the same suite as the correct option where possible, so they share the
scene and objects and cannot be ruled out from the scene alone. Option order is a random permutation (answer
letters A/B/C/D: 161/148/153/138). Every answer follows from the dataset's task labels and the episode order;
nothing is annotated by hand. Construction is seeded (seed 0); the session list with episode indices and frame
spans is released with the code. The evaluated models are the video LLMs
that a robot would query about its own history; the version in which a VLA policy executes the recalled command is described in \cref{supp:watchdo}.

\noindent\textbf{Protocol and statistics.}
Sessions are stored in the LVBench format and run with the LVBench runners unchanged: 64 uniformly sampled frames,
the query-agnostic protocol of \cref{sec:exp-video}, the same budgets, baselines, and $\eta_h$ tables (no
calibration on LIBERO). KV memory is the stored KV cache per
question at the final length (memory model of \cref{supp:exp-pareto}).

\noindent\textbf{Results.}
\Cref{supp:tab:lrmain} gives overall accuracy and stored KV cache; FullKV reaches 40.50\%. \method stores about half the memory of Quest and FreeKV (407--412 vs.\ 791--801\,MiB) at similar
accuracy: 0.83 and 1.17 points above them on Qwen2.5-VL; on LLaVA-Video it ties FreeKV and trails Quest
by 2.00. It trails FullKV by 2.17 and 3.33 points. At the same budget it is 2.50 points above ReKV on both models, 5.50 and 3.83
above InfiniPot-V at about 2K stored tokens, and 3.33 above SnapKV on Qwen2.5-VL (within
0.2 on LLaVA-Video). Token pruning and fewer frames reach similar or higher accuracy at similar or lower storage
(\cref{supp:tab:compose} combines pruning with \method).

\begin{table}[!htbp]
\centering
\footnotesize
\setlength{\tabcolsep}{3pt}
\caption{\textbf{LIBERO-Recall} accuracy (\%, 600 questions) and stored KV cache per
question (MiB). Budget: Read for KV-cache methods (InfiniPot-V at $\approx$2K: 16.4\,/\,20.1\%), share of visual tokens kept (and read) otherwise
(\cref{supp:read}); marks as in \cref{tab:videomme}.}
\label{supp:tab:lrmain}
\begin{tabular}{@{}llrcrc@{}}
\toprule
 & & \multicolumn{2}{c}{Qwen2.5-VL} & \multicolumn{2}{c}{LLaVA-Video} \\
\cmidrule(lr){3-4}\cmidrule(l){5-6}
\textbf{Method} & \textbf{Budget} & KV (MiB) & Acc. & KV (MiB) & Acc. \\
\midrule
FullKV ($N$) & 100\% & 778 & 40.50 & 768 & 40.50 \\
FullKV ($N/2$) & 50\% & 392 & 37.67 & 388 & 38.50 \\
\midrule
VisionZip$^\ddagger$ & 50\% & 392 & 39.83 & 401 & 37.50 \\
InfiniPot-V$^\ddagger$ & 50\% & 392 & 38.50 & 401 & 37.83 \\
VisionZip$^\ddagger$ & 25\% & 199 & 38.00 & 217 & 37.33 \\
\midrule
InfiniPot-V$^\ddagger$ & $\approx$15\% & 126 & 32.83 & 148 & 33.33 \\
ReKV$^\ddagger$ & $\approx$15\% & 778$^{\S}$ & 35.83 & 770$^{\S}$ & \topthree{34.67} \\
SnapKV$^\dagger$ & $\approx$15\% & 136 & 35.00 & 135 & 37.33 \\
Quest & $\approx$15\% & 801 & \toptwo{37.50} & 791 & \topone{39.17} \\
FreeKV & $\approx$15\% & 801 & \topthree{37.17} & 791 & \toptwo{37.17} \\
\method{} & $\approx$15\% & 412 & \topone{38.33} & 407 & \toptwo{37.17} \\
\bottomrule
\end{tabular}
\end{table}

\noindent\textbf{Results by question type.}
\Cref{supp:tab:liberorecall} gives accuracy per type (100 questions per type). FullKV ($N$) reaches
40.50\% on both models; \emph{after} and \emph{before} are the hardest types (26--29\% for
FullKV). The spread between methods is largest on
\emph{first} (21--24 points between the best and worst budgeted arm).

\begin{table}[!htbp]
\centering
\scriptsize
\setlength{\tabcolsep}{2.6pt}
\caption{LIBERO-Recall accuracy (\%) by question type (100 questions each) and overall (All, 600), $N{=}64$
frames. Budget and marks as in \cref{supp:tab:lrmain}; ranking within the matched $\approx$15\% budget, per
column, SnapKV excluded.}
\label{supp:tab:liberorecall}
\begin{tabular}{@{}llcccccccccccccc@{}}
\toprule
 & & \multicolumn{7}{c}{Qwen2.5-VL} & \multicolumn{7}{c}{LLaVA-Video} \\
\cmidrule(lr){3-9}\cmidrule(l){10-16}
\textbf{Method} & \textbf{Budget} & Pres. & Abs. & First & Last & After & Before & All & Pres. & Abs. & First & Last & After & Before & All \\
\midrule
FullKV ($N$) & 100\% & 36.00 & 40.00 & 54.00 & 58.00 & 29.00 & 26.00 & 40.50 & 48.00 & 28.00 & 62.00 & 50.00 & 28.00 & 27.00 & 40.50 \\
FullKV ($N/2$) & 50\% & 34.00 & 35.00 & 50.00 & 48.00 & 33.00 & 26.00 & 37.67 & 45.00 & 29.00 & 63.00 & 42.00 & 28.00 & 24.00 & 38.50 \\
\midrule
VisionZip$^\ddagger$ & 50\% & 36.00 & 36.00 & 55.00 & 60.00 & 24.00 & 28.00 & 39.83 & 41.00 & 29.00 & 62.00 & 43.00 & 27.00 & 23.00 & 37.50 \\
InfiniPot-V$^\ddagger$ & 50\% & 38.00 & 40.00 & 42.00 & 59.00 & 30.00 & 22.00 & 38.50 & 39.00 & 28.00 & 56.00 & 53.00 & 26.00 & 25.00 & 37.83 \\
VisionZip$^\ddagger$ & 25\% & 36.00 & 39.00 & 54.00 & 49.00 & 23.00 & 27.00 & 38.00 & 46.00 & 27.00 & 56.00 & 42.00 & 29.00 & 24.00 & 37.33 \\
\midrule
InfiniPot-V$^\ddagger$ & $\approx$15\% & \toptwo{33.00} & \topone{42.00} & 31.00 & 51.00 & 23.00 & 17.00 & 32.83 & 34.00 & \toptwo{27.00} & 41.00 & \topthree{46.00} & \topone{27.00} & 25.00 & 33.33 \\
ReKV$^\ddagger$ & $\approx$15\% & \topthree{32.00} & 31.00 & \topthree{49.00} & \toptwo{55.00} & \topone{26.00} & 22.00 & 35.83 & \topthree{38.00} & \toptwo{27.00} & \topthree{59.00} & 38.00 & 20.00 & \topthree{26.00} & \topthree{34.67} \\
SnapKV$^\dagger$ & $\approx$15\% & 27.00 & 33.00 & 37.00 & 66.00 & 30.00 & 17.00 & 35.00 & 34.00 & 34.00 & 56.00 & 55.00 & 19.00 & 26.00 & 37.33 \\
Quest & $\approx$15\% & 29.00 & \toptwo{39.00} & \toptwo{53.00} & \topone{56.00} & \topthree{24.00} & \topthree{24.00} & \toptwo{37.50} & \topone{42.00} & \topthree{24.00} & \topone{62.00} & \topone{53.00} & \toptwo{26.00} & \topone{28.00} & \topone{39.17} \\
FreeKV & $\approx$15\% & \topone{37.00} & \topthree{34.00} & 48.00 & \topthree{54.00} & 23.00 & \toptwo{27.00} & \topthree{37.17} & \toptwo{40.00} & \topthree{24.00} & \topthree{59.00} & \toptwo{49.00} & \topthree{25.00} & \topthree{26.00} & \toptwo{37.17} \\
\method{} & $\approx$15\% & \toptwo{33.00} & \topthree{34.00} & \topone{55.00} & 53.00 & \toptwo{25.00} & \topone{30.00} & \topone{38.33} & \toptwo{40.00} & \topone{29.00} & \toptwo{61.00} & 45.00 & 21.00 & \toptwo{27.00} & \toptwo{37.17} \\
\bottomrule
\end{tabular}
\end{table}

\section{Video: paired statistics, composition with token pruning, and more frames}
\label{supp:video}

\noindent\textbf{Paired statistics for \cref{sec:exp-video}.}
Intervals are paired 95\% intervals as defined in \cref{supp:read} (bootstrap over sessions for the robot tasks). Against Quest and FreeKV, \method's per-setting differences on \mbox{Video-MME} and
LVBench (both models) range from $-$1.11 to $+$0.71 and from $-$1.29 to $+$1.23; one of the eight intervals
excludes zero, \mbox{Video-MME} with Qwen2.5-VL against Quest ($-$1.11 [$-$2.11, $-$0.07]). Averaged over the
four settings, VisionZip-50\% and InfiniPot-V-50\% are 0.98 [0.34, 1.64] and 0.76 [0.06, 1.46] points above
\method, and \method is 0.89 [0.09, 1.71] points above FullKV with $N/2$ frames. At 128 frames, \method is
0.11 points above VisionZip-50\% [$-$1.78, $+$2.00] (\cref{supp:tab:iso128}). In Watch-and-Do, of \method's
six command-accuracy differences to FullKV, Quest, and FreeKV, only LLaVA-Video against Quest excludes zero
($-$5.00 [$-$9.58, $-$0.42]); no task-success interval excludes zero (LLaVA-Video: $-$3.75 to $-$1.25).

\noindent\textbf{Composition with token pruning.}
Token pruning and \method act on different axes. Pruning decides which
visual tokens enter the cache; \method changes how the cached pages are
stored and retrieved, whichever tokens they hold, so the two compose. We run
\method, in the configuration of \cref{tab:videomme}, on the cache that
VisionZip keeps at 50\% retention (\cref{supp:tab:compose}). The combination
stores 27--28\% of FullKV, about half of VisionZip-50\% alone. In all six (model,
benchmark) settings it is within 0.6 points of VisionZip-25\%, which stores about the same. Against \method alone, the combination
is 1.81 points lower on LVBench with Qwen2.5-VL (paired 95\% interval [0.26, 3.36]); in the other five
settings the paired intervals contain zero. VisionZip-25\% reads every token it keeps; the combination reads
2048 tokens per step.

\begin{table}[!htbp]
\centering
\footnotesize
\caption{\textbf{Composition with token pruning}, accuracy (\%) under the
query-agnostic protocol with $N{=}64$ frames, on all questions of \mbox{Video-MME}
(2{,}700), LVBench (1{,}549) and LIBERO-Recall (600). VisionZip-50\%,
VisionZip-25\% and \method alone are the rows of
\cref{tab:videomme} and \cref{supp:tab:lrmain}. Stored: stored KV cache as a
share of FullKV ($N$) on \mbox{Video-MME}, Qwen2.5-VL\,/\,LLaVA-Video (memory model
of \cref{supp:exp-pareto}); on LIBERO-Recall the shares differ by less than
one point. $^\ddagger$Our port
(\cref{supp:video-ports}).}
\label{supp:tab:compose}
\setlength{\tabcolsep}{4pt}
\setlength{\tabcolsep}{3pt}
\begin{tabular*}{\linewidth}{@{\extracolsep{\fill}}lccccccc@{}}
\toprule
 & & \multicolumn{2}{c}{\mbox{Video-MME}} & \multicolumn{2}{c}{LVBench} & \multicolumn{2}{c}{LIBERO-Recall} \\
\cmidrule(lr){3-4}\cmidrule(lr){5-6}\cmidrule(l){7-8}
Method & Stored (\%) & Qwen2.5-VL & LLaVA-Video & Qwen2.5-VL & LLaVA-Video & Qwen2.5-VL & LLaVA-Video \\
\midrule
FullKV ($N$) & 100\,/\,100 & 63.70 & 62.89 & 41.25 & 41.64 & 40.50 & 40.50 \\
\method & 53\,/\,53 & 61.93 & 61.37 & 39.70 & 40.87 & 38.33 & 37.17 \\
VisionZip$^\ddagger$ 50\% & 50\,/\,52 & 62.93 & 63.44 & 39.90 & 41.51 & 39.83 & 37.50 \\
\midrule
VisionZip$^\ddagger$ 25\% & 26\,/\,28 & 61.78 & 61.11 & 38.48 & 40.35 & 38.00 & 37.33 \\
VisionZip$^\ddagger$ 50\% + \method & $\approx$27\,/\,$\approx$28 & 61.70 & 61.56 & 37.90 & 40.35 & 38.50 & 37.83 \\
\bottomrule
\end{tabular*}
\end{table}

\noindent\textbf{Spending the saved memory on more frames.}
\method stores about half the cache, so at about the same memory it can watch twice as
many frames. We run Qwen2.5-VL with $N{=}128$ frames on the 900 questions of
the \mbox{Video-MME} Long split, in the configuration of \cref{tab:videomme}, and
compare at the storage of FullKV with $N{=}64$ (\cref{supp:tab:iso128}).
\method with 128 frames stores 104\% of that cache and scores 2.1 points
higher, within 0.5 points
of FullKV with 128 frames, which stores twice as much. VisionZip-50\% with 128
frames, at 100\%, gains a similar 2.0 points, 0.1 below \method.
Neither gain over FullKV with $N{=}64$ has a paired 95\% interval that excludes zero (\method: $+$2.11
[$-$0.22, $+$4.44]; VisionZip-50\%: $+$2.00 [$-$0.11, $+$4.11]). The two differ in what is read per decoding
step: \method still attends 2048 tokens, while VisionZip-50\% with 128 frames attends every token it keeps,
13{,}240 on average (6.5$\times$). At $N{=}64$, VisionZip-50\% reads 3.4$\times$ as many tokens per step as
\method on average over the four video settings (3.2--3.6$\times$; 2.7--3.0$\times$ if \method's dense
layer 0 is counted).

\begin{table}[!htbp]
\centering
\footnotesize
\caption{\textbf{More frames at matched memory}, Qwen2.5-VL on the \mbox{Video-MME}
Long split (900 questions), query-agnostic protocol. Stored: mean stored
KV cache per question as a share of FullKV with $N{=}64$ (memory model of
\cref{supp:exp-pareto}). $\Delta$: paired difference to FullKV ($N{=}64$) in
points. Read: tokens attended per decoding step (layers 1--27, text included) as a share of the
visual tokens of $N{=}64$ frames. $^\ddagger$Our port (\cref{supp:video-ports}).}
\label{supp:tab:iso128}
\setlength{\tabcolsep}{4pt}
\begin{tabular*}{\linewidth}{@{\extracolsep{\fill}}lccccc@{}}
\toprule
Method & $N$ & Stored (\%) & Read (\%) & Acc. (\%) & $\Delta$ vs FullKV ($N{=}64$) \\
\midrule
FullKV & 32 & 50.5 & 51.0 & 51.22 & -- \\
FullKV & 64 & 100.0 & 101.0 & 52.22 & -- \\
\method & 64 & 53.1 & 15.6 & 51.00 & -- \\
VisionZip$^\ddagger$ 50\% & 64 & 50.5 & 51.0 & 51.89 & -- \\
\midrule
FullKV & 128 & 199.0 & 201.0 & 54.78 & $+2.56$ \\
VisionZip$^\ddagger$ 50\% & 128 & 100.0 & 101.0 & 54.22 & $+2.00$ \\
\method & 128 & 104.4 & 15.6 & 54.33 & $+2.11$ \\
\bottomrule
\end{tabular*}
\end{table}

\end{document}